\documentclass{article}

\usepackage{arxiv}

\usepackage[utf8]{inputenc} % allow utf-8 input
\usepackage[T1]{fontenc}    % use 8-bit T1 fonts
\usepackage{hyperref}       % hyperlinks
\usepackage{url}            % simple URL typesetting
\usepackage{booktabs}       % professional-quality tables
\usepackage{amsfonts}       % blackboard math symbols
\usepackage{nicefrac}       % compact symbols for 1/2, etc.
\usepackage{microtype}      % microtypography
\usepackage{lipsum}		% Can be removed after putting your text content
\usepackage{graphicx}
\usepackage{natbib}
\usepackage{mathtools}
\usepackage{bm}
\usepackage{doi}
\usepackage{subcaption}
\usepackage{multirow}

\usepackage{amsmath}
\usepackage{amsfonts}
\usepackage{amssymb}
\usepackage{amsthm}  % 这个宏包是管理定理环境的核心

\theoremstyle{definition} % 定理样式：标题加粗，正文使用常规字体
\newtheorem{assumption}{Assumption} % 定义一个名为 assumption 的环境，显示为 "Assumption"

\newtheorem{theorem}{Theorem}
\newtheorem{lemma}{Lemma}
\newtheorem{remark}{Remark}
\newtheorem{proposition}{Proposition}
\newtheorem{corollary}{Corollary}
\newtheorem{definition}{Definition}
\usepackage{hyperref} % 用于点击跳转
\usepackage{cleveref} % 必须放在 hyperref 后面！
\crefname{assumption}{Assumption}{Assumptions}

\title{Wasserstein Convergence of ODE-Based Samplers in Decentralized Diffusion Model via Velocity Field Decomposition}

\author{
    \makebox[0.28\textwidth][c]{\textbf{Chencheng Tang}\thanks{Equal contribution.}} \\
    Peking University \\
    \texttt{\small tcc@stu.pku.edu.cn}
    \And
    \makebox[0.28\textwidth][c]{\textbf{Xuanyu Xue}\footnotemark[1]} \\
    Shanghai Jiao Tong University \\
    \texttt{\small sjtuxxy0398@sjtu.edu.cn}
    \And
    \makebox[0.28\textwidth][c]{\textbf{Fangyikang Wang}} \\
    MBZUAI \\
    \texttt{\small fangyikang.wang@mbzuai.ac.ae}
    \AND
    \makebox[0.28\textwidth][c]{\textbf{Chao Zhang}} \\
    Zhejiang University \\
    \texttt{\small zczju@zju.edu.cn}
    \And
    \makebox[0.28\textwidth][c]{\textbf{Hubery Yin}\thanks{Corresponding author.}} \\
    Tencent \\
    \texttt{\small hubery@tencent.com}
}
\renewcommand{\shorttitle}{Preprint}

\hypersetup{
pdftitle={Wasserstein Convergence of ODE-Based Samplers in Decentralized Diffusion Model via Velocity Field Decomposition},
pdfsubject={cs.LG, cs.AI},
pdfauthor={David S.~Hippocampus, Elias D.~Striatum},
pdfkeywords={Convergence analysis, diffusion models, ODE, Wasserstein Distance},
}

\begin{document}
\maketitle

\begin{abstract}
  Diffusion models have achieved impressive empirical success in generative tasks, and their convergence theory is now relatively well understood. Motivated by privacy and scalability, recent decentralized diffusion architectures replace a single global velocity field with multiple local experts and a routing mechanism, yielding a sampling dynamics with stochastic expert switching that falls outside standard diffusion convergence analyses. In this work, We study a decentralized diffusion framework with stochastic velocity fields and ODE-based sampling. We establish a convergence guarantee in Wasserstein-2 distance, showing that the distribution of the $N$-step discretization converges to the analytical solution at rate $\mathcal{O}(N^{-1/2}+\varepsilon)$ in $W_2$, where $\varepsilon$ captures the neural approximation errors. To our knowledge, this is the first $W_2$ convergence result for decentralized diffusion models with an ODE-based sampling scheme.
  %Furthermore, we show that a straightforward design that directly predicts the mixing coefficient of local velocity does not, in general, yield a convergent sampler. Motivated by this, we instead propose to learn the \emph{rate of change} of the coefficients, and we prove that this alternative parameterization leads to a provably convergent decentralized diffusion sampler.
\end{abstract}

% keywords can be removed

\section{Introduction}
\label{sec:introduction}

Diffusion probabilistic models \cite{ho2020denoising,song2020score} and their recent generalization, Flow Matching models \cite{lipman2023flow,albergo2023stochastic}, have achieved unprecedented success in generative tasks, covering a wide range of modalities such as text-to-image generation \cite{rombach2022high}, video synthesis \cite{ho2022video}, 3D content creation \cite{poole2023dreamfusion}, and image restoration \cite{saharia2022image}.
Among various sampling strategies, the Probability Flow ODE (PF-ODE) \cite{song2020score,karras2022elucidating} formulation has become the de facto standard for efficient deployment. 
Unlike stochastic samplers, the PF-ODE defines a deterministic transport between the noise and data distributions while preserving marginal distributions, enabling the use of specialized ODE solvers to significantly reduce sampling steps without sacrificing quality \cite{lu2022dpm,lu2025dpm}.
Given its centrality, the theoretical properties of the PF-ODE have been extensively studied, with a rich body of work establishing convergence guarantees in terms of Wasserstein-2 distance ($W_2$), KL divergence, and Total Variation (TV) distance \cite{lee2022convergence,kwon2022score,chen2023sampling,chen2023improved,gao2024convergence,huang2025convergence}.

However, existing theoretical frameworks predominantly rely on an assumption: they presume access to a global score function (or velocity field) learned from the entire dataset. 
This assumption is increasingly challenged by the shift towards decentralized and modular learning paradigms. 
Driven by data privacy concerns, computational scaling bottlenecks, and the mixture nature of real-world distributions, recent works have explored \textit{Decentralized Diffusion Model} architectures \cite{mcallister2025decentralized,chen2025gaussian}  . 
In these settings, one does not learn a single global velocity field. Instead, multiple ``expert'' fields are trained on local data partitions or clusters, and the global generation process is realized by dynamically routing or switching between these local experts.
While such heuristic routing strategies have demonstrated empirical success, they fundamentally violate the global Lipschitz continuity and smoothness assumptions required by classical ODE convergence theory. Moreover, the routing decisions introduce stochastic fluctuations that act as an extrinsic noise term, complicating the analysis of the sampling process.
Consequently, a critical theoretical gap emerges: \textit{Does the ODE sampling process driven by a mixture of local experts rigorously converge to the target distribution? If so, how do the routing mechanism and local approximation errors impact the final generation quality?}

In this paper, we address this problem by establishing the first rigorous convergence theory for decentralized sampling via the Probability Flow ODE, in contrast to previous works that focus on SDE-based samplers \cite{anonymous2026convergence}. Crucially, our ODE formulation naturally generalizes to Flow Matching models, providing a more versatile theoretical foundation that encompasses a broader class of deterministic generative trajectories.
We characterize the global velocity field as a probability-weighted superposition of component fields. 
Our main contributions are summarized as follows:

\begin{itemize}
    \item \textbf{Theoretical Framework for Mixture Dynamics:} We formulate the Probability Flow ODE for mixture distributions and derive a velocity field decomposition that justifies the use of decentralized experts. We identify the specific conditions under which component-wise learning recovers the global score.
    
    \item \textbf{Wasserstein-2 Convergence Guarantee:} 
    We provide a rigorous convergence analysis for the proposed sampling algorithm. Under standard regularity assumptions on the data support and noise schedule, we establish the \textbf{first} non-asymptotic convergence bound for the proposed sampling algorithm. Specifically, we prove that the generated distribution converges to the target distribution at the truncation time $t=\delta$ with a Wasserstein-2 distance scaling of $\tilde{O}(\epsilon + N^{-1/2})$. Crucially, our analysis explicitly characterizes the error coefficient's dependence on the data manifold radius $R$, dimension $d$, and the number of mixture components $L$, as well as the truncation time $\delta$, thereby providing a rigorous theoretical baseline for the decentralized diffusion framework.
\end{itemize}

\section{Related Work}
\label{sec:related_work}

\subsection{Convergence of SDE-based Samplers}
Score-based generative models (SGMs) \citep{song2020score} have achieved remarkable success, prompting extensive theoretical analysis of their sampling convergence\cite{lee2022convergence,kwon2022score,chen2023improved,benton2023nearly}. A significant body of work focuses on Stochastic Differential Equation (SDE) samplers. For instance, \citet{chen2023sampling} established polynomial-time convergence guarantees in Total Variation (TV) distance under minimal data assumptions. 
However, these analyses typically rely on SDE-specific tools, such as Girsanov's theorem and log-Sobolev inequalities, which crucially exploit the stochasticity of Brownian motion. The injected noise effectively acts as a regularizer, facilitating distribution convergence. 
Furthermore, to obtain meaningful TV bounds for general distributions, techniques like early stopping (targeting a distribution at $t=\delta$) are often required to avoid singularities near $t=0$.
In our decentralized setting, we utilize the Probability Flow ODE, which is deterministic. The absence of Brownian noise removes the contractive effects found in SDEs, making error accumulation more delicate and limiting the direct application of SDE-based proof techniques.

\subsection{Convergence of ODE-based Samplers}
The theoretical understanding of ODE-based samplers has Progressed rapidly. Recent works, such as \citet{kwon2022score,gao2024convergence}, have established rigorous Wasserstein-2 ($W_2$) convergence guarantees for ODE samplers in diffusion models and flow models. These analyses typically rely on Grönwall inequalities and assume the velocity field is deterministic and Lipschitz continuous.
The choice of $W_2$ distance in these works, as well as in the present work, is fundamentally motivated by the nature of the sampling process as numerical integration. Unlike TV or KL divergence, which require controlling pointwise density evolution, $W_2$ naturally quantifies the mean square error of particle trajectories.
However, existing ODE theories do not directly apply to our Mix-of-Velocity framework. Although our underlying probability flow is deterministic, the \textit{realized} velocity field in our decentralized setting involves stochastic routing among experts.  This introduces a "routing noise" (modeled as martingale noise in our analysis) that is that is neither Gaussian diffusion nor simple discretization error. A direct application of standard ODE bounds would treat this stochastic switching as a worst-case perturbation, leading to loose error estimates. Our work extends deterministic ODE theory by modeling this routing variability as a martingale difference sequence, allowing for tighter convergence rates.

\subsection{Decentralized Diffusion Models}
The paradigm of decentralized learning for diffusion models is an emerging area of research\cite{mcallister2025decentralized,chen2025gaussian}. Most relevant to our analysis is the Decentralized Diffusion Models (DDM) framework proposed by \citet{mcallister2025decentralized}, which utilizes a routing mechanism to partition the generation process among multiple experts. Although \citet{mcallister2025decentralized} demonstrate the empirical effectiveness and scalability of this approach, their study focuses primarily on system design and implementation, and lacks a rigorous theoretical foundation for convergence. To the best of our knowledge, our work provides the first rigorous Wasserstein-2 ($W_2$) convergence guarantee for such mixture-based decentralized diffusion frameworks, explicitly accounting for the stochasticity induced by the routing mechanism.

\section{Notations and Preliminaries}
\label{sec:preliminaries}

This section establish the mathematical framework for analyzing mixture-based diffusion models. Our formulation builds upon the mixture representation of data distribution and characterizes the forward and reverse dynamic through velocity fields.

\subsection{Mixture Representation of Data Distributions}

We assume that the initial data distribution $p_0(\mathbf{x})$ on $\mathbb{R}^d$ exhibits an underlying mixture structure. Specifically, we model $p_0(\mathbf{x})$ as a weighted sum of $L$ component distributions.

\begin{assumption}[Initial Mixture Distribution]
\label{assump:mixture_data}
The initial data distribution $p_0(\mathbf{x})$ is assumed to be a mixture of $L$ probability density functions $\{p_0^l(\mathbf{x})\}_{l=1}^L$:
\begin{equation}
    p_0(\mathbf{x}) = \sum_{l=1}^L \pi_l p_0^l(\mathbf{x}),
\end{equation}
where $\pi_l \in (0, 1)$ is the prior probability (mixing weight) of the $l$-th component, satisfying the normalization constraint $\sum_{l=1}^L \pi_l = 1$.
\end{assumption}

In the discrete dataset setting, \cref{assump:mixture_data} encompasses partitioning the dataset into $L$ subsets $\{D_l\}_{l=1}^L$, where each component $p_0^l(\mathbf{x})$ represents either the empirical distribution or a smoothed density estimate over subset $D_l$. Due to the linearity of the SDE in \cref{def:vp_sde}, the forward diffusion process preserves this mixture structure. For any time $t \in [0, T]$, the marginal distribution $p_t(\mathbf{x})$ remains a mixture:
\begin{equation}
\label{eq:marginal_mixture}
    p_t(\mathbf{x}) = \int p_{0t}(\mathbf{x} | \mathbf{x}_0) p_0(\mathbf{x}_0) d\mathbf{x}_0 = \sum_{l=1}^L \pi_l p_t^l(\mathbf{x}),
\end{equation}
where $p_t^l(\mathbf{x}) := \int p_{0t}(\mathbf{x} | \mathbf{y}) p_0^l(\mathbf{y}) d\mathbf{y}$ denotes the marginal density of the $l$-th component after diffusion.

\begin{definition}[Component Posterior Probability]
\label{def:posterior_coeff}
The coefficient $a^l(\mathbf{x}, t)$ is defined as the posterior probability that sample $\mathbf{x}$ at time $t$ belongs to the $l$-th component:
\begin{equation}
    a^l(\mathbf{x}, t) := P(\text{component}=l \mid \mathbf{x}_t = \mathbf{x}) = \frac{\pi_l p_t^l(\mathbf{x})}{p_t(\mathbf{x})}.
\end{equation}
\end{definition}

\begin{remark}
Using the Gaussian transition kernel from \cref{def:vp_sde}, namely $p_{0t}(\mathbf{x}|\mathbf{y}) = \mathcal{N}(\mathbf{x}; \alpha_t \mathbf{y}, \sigma_t^2 \mathbf{I})$ where $\sigma_t^2 = 1-\alpha_t^2$, the coefficient $a^l(\mathbf{x}, t)$ admits an explicit integral form:
\begin{equation}
\label{eq:CPP}
    a^l(\mathbf{x}, t) = \frac{\pi_l \int \exp\left( -\frac{\| \mathbf{x} - \alpha_t \mathbf{y} \|^2 }{2\sigma_t^2} \right) p_0^l(\mathbf{y}) d\mathbf{y}}{\sum_{k=1}^L \pi_k \int \exp\left( -\frac{\| \mathbf{x} - \alpha_t \mathbf{y}' \|^2 }{2\sigma_t^2} \right) p_0^k(\mathbf{y}') d\mathbf{y}'}.
\end{equation}
\end{remark}

\subsection{Forward and Reverse Diffusion Processes}

We formally define the variance-preserving stochastic differential equation (VP-SDE) governing the forward diffusion process.

\begin{definition}[Forward Process: VP-SDE]
\label{def:vp_sde}
The forward diffusion process is a stochastic process $\{\mathbf{x}_t\}_{t \in [0, T]}$ satisfying the It\^o SDE:
\begin{equation}
    d\mathbf{x}_t = f(t)\mathbf{x}_t \, dt + g(t) \, d\mathbf{w}_t, \quad \mathbf{x}_0 \sim p_0(\mathbf{x}),
\end{equation}
where $\mathbf{w}_t$ is a standard $d$-dimensional Wiener process and $f(t) < 0$ is a continuous scalar function. To satisfy the variance-preserving condition, the diffusion coefficient is $g(t) := \sqrt{-2f(t)}$. Given initial value $\mathbf{x}_0$,the conditional distribution of state $\mathbf{x}t$ is :
\begin{equation}
    p_{0t}(\mathbf{x}_t | \mathbf{x}_0) = \mathcal{N}(\mathbf{x}_t; \alpha_t \mathbf{x}_0, (1-\alpha_t^2)\mathbf{I}),
\end{equation}
where $\alpha_t := \exp\left(\int_0^t f(s)\,ds\right)$. Define $\beta(t) = -\frac{2}{\alpha_t}\frac{d\alpha_t}{dt}$ which corresponds to the drift coefficient $f(t) = -\frac{1}{2}\beta(t)$
\end{definition}

\begin{definition}[Score Function]
\label{def:score_function}
For any time $t \in [0, T]$, the score function of marginal density $p_t(\mathbf{x})$ is:
\begin{equation}
    \mathbf{s}(\mathbf{x}, t) := \nabla_{\mathbf{x}} \log p_t(\mathbf{x}).
\end{equation}
Similarly, for component $l$, the conditional score function is $\mathbf{s}^l(\mathbf{x}, t) := \nabla_{\mathbf{x}} \log p^l_t(\mathbf{x})$.
\end{definition}

\begin{definition}[Probability Flow ODE]
\label{def:velocity_field}
The probability flow ODE corresponding to \cref{def:vp_sde} is given by:
\begin{equation}
    \frac{d\mathbf{x}_t}{dt} = \mathbf{v}(\mathbf{x}_t, t),
\end{equation}
where the true velocity field is:
\begin{equation}
    \mathbf{v}(\mathbf{x}, t) := f(t)\mathbf{x} - \frac{1}{2}g(t)^2 \mathbf{s}(\mathbf{x}, t) = f(t)\left[ \mathbf{x} + \nabla_{\mathbf{x}} \log p_t(\mathbf{x}) \right].
\end{equation}
The component velocity field $\mathbf{v}^l(\mathbf{x}, t)$ is defined identically by replacing $\mathbf{s}(\mathbf{x}, t)$ with $\mathbf{s}^l(\mathbf{x}, t)$.
\end{definition}

Using Tweedie's formula, the component score function satisfies $\mathbf{s}^l(\mathbf{x}, t) = -(\mathbf{x} - \alpha_t \mathbf{y}^l(\mathbf{x}, t)) / \sigma_t^2$, where $\mathbf{y}^l(\mathbf{x}, t) := \mathbb{E}[\mathbf{x}_0 \mid \mathbf{x}_t = \mathbf{x}, \text{component}=l]$ is the component-conditional denoiser. Substituting this into the velocity definition yields:
\begin{equation} \label{eq:v_l_explicit}
    \mathbf{v}^l(\mathbf{x}, t) = f(t) \left( \mathbf{x} - \frac{\mathbf{x} - \alpha_t \mathbf{y}^l(\mathbf{x}, t)}{\sigma_t^2} \right).
\end{equation}

We define the \textbf{neural velocity field} induced by the regression network $\mathbf{y}^l_\theta(\mathbf{x}, t)$ with parameters $\theta$ as:
\begin{equation} \label{eq:v_l_theta_def}
    \mathbf{v}^l_\theta(\mathbf{x}, t) := f(t) \left( \mathbf{x} - \frac{\mathbf{x} - \alpha_t \mathbf{y}^l_\theta(\mathbf{x}, t)}{\sigma_t^2} \right).
\end{equation}
The mixture neural velocity field is then given by $\overline{\mathbf{v}}_\theta(\mathbf{x}, t) := \sum_{l=1}^L a^l_\phi(\mathbf{x}, t) \mathbf{v}^l_\theta(\mathbf{x}, t)$.
\begin{proposition}[Velocity Field Decomposition]
\label{prop:velocity_decomp}
Under \cref{assump:mixture_data}, the total velocity field decomposes as:
\begin{equation}
    \mathbf{v}(\mathbf{x}, t) = \sum_{l=1}^L a^l(\mathbf{x}, t) \mathbf{v}^l(\mathbf{x}, t),
\end{equation}
where $a^l(\mathbf{x}, t)$ are the posterior probabilities from \cref{def:posterior_coeff}.
\end{proposition}

\subsection{Sampling Dynamics and Algorithms}
To rigorously analyze the convergence, we define the reverse-time generation process and the corresponding numerical sampling algorithm. Let $\tau := T-t$ be the reverse time variable. The analysis focuses on the time interval $\tau \in [0, T_\delta]$, where $T_\delta = T - \delta$ and $\delta > 0$ is a cutoff to avoid the singularity at $t=0$.

\begin{itemize}
    \item \textbf{Reference Dynamics}: The ideal continuous-time reverse process driven by the exact velocity field $\tilde{\mathbf{v}}(\mathbf{x}, \tau) := -\mathbf{v}(\mathbf{x}, T-\tau)$. The dynamics are governed by:
    \begin{equation}
        d\mathbf{X}(\tau) = \tilde{\mathbf{v}}(\mathbf{X}(\tau), \tau) d\tau, \quad \tau \in [0, T_\delta].
    \end{equation}
    
    \item \textbf{Neural Numerical Algorithm}: The practical discrete-time sampler using the neural network velocity approximation $\tilde{\mathbf{v}}_\theta$. For time steps $0 = \tau_0 < \tau_1 < \dots < \tau_N = T_\delta$ with step size $\Delta \tau=\frac{T_\delta}{N}$, the update rule is:
    \begin{equation}
        \hat{\mathbf{x}}_{k+1} = \hat{\mathbf{x}}_k + \Delta \tau \tilde{\mathbf{v}}_\theta^{\hat{l}_k}(\hat{\mathbf{x}}_k, \tau_k),
    \end{equation}
    where $\hat{l}_k \sim \text{Categorical}(a_\phi(\hat{\mathbf{x}}_k, \tau_k))$ represents the stochastic component selection. This process incorporates discretization error, martingale sampling noise, and neural approximation error.
\end{itemize}

\subsection{Time Interval Truncation}

To analyze the convergence, we focus on a truncated time interval to handle the singularity at $t=0$. Let $\delta \in (0, T)$ be a small positive constant.

\begin{definition}[Analysis Domain]
\label{def:domain_partition}
We define the truncated spatiotemporal domain $\Omega$ as:
\begin{equation}
    \Omega  := \{ (\mathbf{x}, t) \in \mathbb{R}^d \times [0, T] : \delta \le t \le T \}. \label{eq:omega}
\end{equation}
\end{definition}

Our subsequent theoretical analysis focuses on bounding the sampling error within $\Omega$, effectively avoiding the singularities of the score function and velocity field as $t\to 0$.

\subsection{Choice of Metric}

\begin{definition}[Wasserstein-2 Distance]
\label{def:wasserstein}
For probability measures $\mu, \nu$ on $\mathbb{R}^d$, the Wasserstein-2 distance is defined as:
\begin{equation}
W_2(\mu, \nu) := \left( \inf_{\gamma \in \Pi(\mu, \nu)} \int_{\mathbb{R}^d \times \mathbb{R}^d} \|\mathbf{x} - \mathbf{y}\|^2 \, d\gamma(\mathbf{x}, \mathbf{y}) \right)^{1/2},
\end{equation}
where $\Pi(\mu, \nu)$ denotes the set of all joint distributions (couplings) with marginals $\mu$ and $\nu$.
\end{definition}

Our choice of the Wasserstein-2 ($W_2$) distance is fundamentally motivated by the analytical framework employed in this work, which treats the sampling process as the numerical integration of the Probability Flow ODE. Standard convergence analyses for numerical ODE solvers typically rely on bounding the \textit{global truncation error} by propagating \textit{local truncation errors} via discrete Grönwall inequalities. This technique directly yields a bound on the strong error of the trajectories, specifically the Mean Square Error (MSE): $\mathbb{E}[\|\mathbf{X}(T) - \hat{\mathbf{x}}_N\|^2]$. Since the $W_2$ distance is defined as the infimum of the quadratic transport cost over all couplings, the MSE derived from our synchronous coupling serves as a natural and tight upper bound for the squared $W_2$ distance, $W_2^2(\hat{\mu}_N, \mu_*)$. Given that our theoretical contribution focuses on the discretization and approximation errors of the velocity field decomposition, the $W_2$ metric offers the most direct and rigorous quantification of the generation quality consistent with our trajectory-wise error analysis.

\section{Convergence Theory}
\label{sec:convergence_theory}

\subsection{Assumptions}
We introduce the following assumptions regarding the data distribution, neural networks, and their approximation errors.

\begin{assumption}[Bounded Support of Initial Distribution]
\label{ass:p0_bounded}
The initial data distribution $p_0(\mathbf{x})$ is a probability density function on $\mathbb{R}^d$ with support satisfying:
\begin{equation}
    \text{supp}(p_0) \subseteq \{\mathbf{x} \in \mathbb{R}^d : \|\mathbf{x}\|_2 \le R\},
\end{equation}
where $R>0$ is a finite constant.
\end{assumption}

This is a standard technical assumption to simplify analysis and avoid complexities arising from unbounded data domains.

\begin{assumption}[Regularity of the Noise Schedule]
\label{ass:alpha_regularity}
Without loss of generality, we fix the time horizon $T=1$. The signal schedule $\alpha_t$ is assumed to be twice continuously differentiable ($C^2$) on the interval $[0,1]$, with boundary conditions $\alpha_0=1$ and $\alpha_1=0$. Furthermore, $\alpha_t$ satisfies the following regularity conditions:
\begin{enumerate}
    \item \textbf{Monotonicity}: $\frac{\mathrm{d}\alpha_t}{\mathrm{d}t} \le 0$ for all $t \in [0,1]$;
    \item \textbf{First-order regularity}: There exists a constant $C_{\alpha, 1} > 0$ such that $\left| \frac{\mathrm{d}\alpha_t}{\mathrm{d}t} \right| \le C_{\alpha, 1}$;
    \item \textbf{Second-order regularity}: There exists a constant $C_{\alpha, 2} > 0$ such that $\left| \frac{\mathrm{d}^2\alpha_t}{\mathrm{d}t^2} \right| \le C_{\alpha, 2}$.
\end{enumerate}
\end{assumption}

\begin{assumption}[Neural Network Approximation Error]
\label{ass:nn_approx}
We assume that the neural networks $a^i_\phi(\mathbf{x}, t)$ and $\mathbf{y}^i_\theta(\mathbf{x}, t)$ provide uniform approximations to the true posterior probabilities and component denoisers, respectively. Specifically, for all $i \in \{1, \dots, L\}$, the following $L^\infty$ bounds hold over the domain $(\mathbf{x}, t) \in \mathbb{R}^d \times [\delta, T]$:
\begin{align}
    \sup_{\mathbf{x}, t} \|a^i_\phi(\mathbf{x}, t) - a^i(\mathbf{x}, t)\| &\le \varepsilon_a \le \varepsilon \le 1, \\
    \sup_{\mathbf{x}, t} \|\mathbf{y}^i_\theta(\mathbf{x}, t) - \mathbf{y}^i(\mathbf{x}, t)\| &\le \varepsilon_y \le \varepsilon \le 1,
\end{align}
where $\mathbf{y}^i_\theta(\mathbf{x}, t)$ is the network parameterized via $\mathbf{x}_0$-prediction, representing the conditional expectation of the clean data.
\end{assumption}

This assumption imposes a uniform global constraint on the neural network approximation errors, ensuring that the cumulative deviation from the exact probability flow remains controlled throughout the sampling process.

\subsection{Main Convergence Results}

In this section, we present the global convergence results for the proposed Mix-of-Velocity sampling algorithm, encompassing both distribution-level and trajectory-level analysis. We first highlight the convergence in terms of the Wasserstein-2 distance, which serves as our primary result for distributional accuracy. This result is built upon a foundation of strong convergence in terms of the mean square error (MSE) for individual trajectories, which we provide subsequently.

\begin{theorem}[Wasserstein-2 Convergence under Neural Network Approximation Error]
\label{thm:w2_convergence_main}
Under Assumptions \ref{ass:p0_bounded}, \ref{ass:alpha_regularity}, \ref{ass:nn_approx}, and \ref{assump:mixture_data}, assume that both the sampling algorithm and the reference process are initialized with the same Gaussian prior $\mathbf{X}(0) \sim \pi = \mathcal{N}(\mathbf{0}, \mathbf{I})$. Let $\hat{\mu}_N$ denote the probability distribution of the numerical solution $\hat{\mathbf{x}}_N$ produced by the neural network-driven discrete sampler at reverse time $\tau = T_{\delta}$ (where $T_{\delta} = T - \delta$), and let $\mu_*$ denote the distribution of the exact probability flow ODE solution $\mathbf{X}(T_{\delta})$. The squared Wasserstein-2 distance satisfies:
\begin{equation}
    W_2^2(\hat{\mu}_N, \mu_*) \le \mathcal{S}(T_{\delta}) \left[ \mathcal{C}_{approx’} \varepsilon^2 + \frac{\mathcal{C}_{disc’}}{N} \right],
\end{equation}
where $\mathcal{S}(T_{\delta}) = (e^{\tilde{C}_1 T_{\delta}} - 1) / \tilde{C}_1$ is the stability factor, while $\mathcal{C}_{approx'}$ and $\mathcal{C}_{disc'}$ are constants representing approximation and discretization errors, respectively. The explicit expressions for these constants depend on the noise schedule regularity $(C_{\alpha,1}, C_{\alpha,2})$, dimension $d$, data manifold radius $R$, number of classes $L$, and truncation parameter $\sigma_{\delta}$, as detailed in \cref{rem:w2_constants} in the appendix.
\end{theorem}

\begin{theorem}[Global Mean Square Error Bound for Discrete Sampling]
\label{thm:strong_convergence_main}
Under the same assumptions as in \cref{thm:w2_convergence_main}, for any number of discretization steps $N \ge 1$, let the step size $\Delta\tau = T_\delta / N \le \min\{\Delta\tau_0, 1\}$. The numerical solution $\hat{\mathbf{x}}_k$ converges to the reference process $\mathbf{X}(\tau_k)$ in the $L^2$ sense. Specifically, for any $k \in \{0, \dots, N\}$, the global mean square error $\hat{E}_k := \mathbb{E}[\|\hat{\mathbf{x}}_k - \mathbf{X}(\tau_k)\|^2]$ satisfies:
\begin{equation}
\label{eq:mse_convergence_main}
\sup_{0 \le k \le N} \hat{E}_k \le \mathcal{S}(T_\delta) \left( \mathcal{C}_{\mathrm{approx}}(\hat{\mathbf{x}}_0) \varepsilon^2 + \mathcal{C}_{\mathrm{disc}}(\hat{\mathbf{x}}_0) \Delta \tau \right),
\end{equation}
where $\hat{\mathbf{x}}_0=\mathbf{X}(0)$ is the initial state, $\mathcal{C}_{\mathrm{approx}}(\hat{\mathbf{x}}_0) = 2C_{NN}^2(1 + \Delta \tau_0)^2 (1 + 2R_0^2)$ is the approximation error coefficient, and $\mathcal{C}_{\mathrm{disc}}(\hat{\mathbf{x}}_0)$ is the discretization error coefficient that explicitly depends on the initial state $\hat{\mathbf{x}}_0$. Here, $R_0 = \|\mathbf{X}(0)\| + R$ denotes the radius of the trajectory bound. The stability factor $\mathcal{S}(T_\delta)$ captures the cumulative effect of error over the simulation duration $T_\delta$. The detailed derivation and explicit expressions for all constants are provided in \cref{thm:strong_convergence_precise_v2} and \cref{rem:w2_constants} in the appendix.
\end{theorem}

\section{Proof of Main Theorems}
\label{sec:main_proofs}

In this section, we provide a structured and rigorous proof sketch for the global convergence results presented in \cref{sec:convergence_theory}. The proof follows a hierarchical approach: we first establish trajectory-level strong convergence in the mean square error (MSE) sense for any fixed initial point, and then integrate this result over the Gaussian prior to obtain the final distributional convergence in Wasserstein-2 distance. For brevity, detailed algebraic derivations are deferred to \cref{sec:convergence_analysis} in the appendix.

\subsection{Trajectory-Level Analysis: Global MSE Bound}
\label{subsec:mse_analysis_main}

The goal of trajectory-level analysis is to bound the conditional mean square error (MSE): $$\hat{E}_N(\mathbf{z}) := \mathbb{E}[\|\hat{\mathbf{x}}_N - \mathbf{X}(T_{\delta})\|^2 \mid \mathbf{X}(0) = \mathbf{z}]$$ This characterizes how discretization errors, random sampling noise, and neural network approximation errors jointly affect the pathwise accuracy of the algorithm.

\textbf{Step 1: Coupling and Error Evolution.}

We construct a coupling by initializing both the numerical trajectory $\{\hat{\mathbf{x}}_k\}_{k=0}^N$ and the reference probability flow ODE path $\{\mathbf{X}(\tau)\}$ from the same initial state $\hat{\mathbf{x}}_0 = \mathbf{X}(0) = \mathbf{z} \sim \mathcal{N}(\mathbf{0}, \mathbf{I})$. Let $\hat{\mathbf{e}}_k := \hat{\mathbf{x}}_k - \mathbf{X}(\tau_k)$ be the global error at step $k$. The numerical update follows $\hat{\mathbf{x}}_{k+1} = \hat{\mathbf{x}}_k + \Delta \tau \tilde{\mathbf{v}}_\theta^{\hat{l}_k}(\hat{\mathbf{x}}_k, \tau_k)$, while the reference process advances as $\mathbf{X}(\tau_{k+1}) = \mathbf{X}(\tau_k) + \int_{\tau_k}^{\tau_{k+1}} \tilde{\mathbf{v}}(\mathbf{X}(s), s) \, ds$.

\textbf{Step 2: Core Error Decomposition.}

To isolate the distinct error sources, we introduce the average neural network velocity field $\overline{\mathbf{v}}_\theta(\mathbf{x}, \tau) := \mathbb{E}_{\hat{l}}[\tilde{\mathbf{v}}_\theta^{\hat{l}}(\mathbf{x}, \tau)] = \sum_{i=1}^L \hat{a}_\phi^i(\mathbf{x}, \tau) \tilde{\mathbf{v}}_\theta^i(\mathbf{x}, \tau)$. The error increment $\hat{\mathbf{e}}_{k+1}$ can be decomposed into four components:

\begin{equation}
\label{eq:error_decomposition_main}
\begin{split}
\hat{\mathbf{e}}_{k+1} &= \hat{\mathbf{e}}_k + \underbrace{\Delta \tau \Delta \tilde{\mathbf{v}}_k}_{\text{Deterministic spatial error propagation}} \\
&\quad + \underbrace{\left( \Delta \tau \tilde{\mathbf{v}}(\mathbf{X}(\tau_k), \tau_k) - \int_{\tau_k}^{\tau_{k+1}} \tilde{\mathbf{v}}(\mathbf{X}(s), s) \, ds \right)}_{\text{Local temporal truncation error } \mathcal{T}_k} \\
&\quad + \underbrace{\Delta \tau \left( \tilde{\mathbf{v}}_\theta^{\hat{l}_k}(\hat{\mathbf{x}}_k, \tau_k) - \overline{\mathbf{v}}_\theta(\hat{\mathbf{x}}_k, \tau_k) \right)}_{\text{Stochastic martingale difference noise } \hat{\mathcal{M}}_k} \\
&\quad + \underbrace{\Delta \tau \left( \overline{\mathbf{v}}_\theta(\hat{\mathbf{x}}_k, \tau_k) - \tilde{\mathbf{v}}(\hat{\mathbf{x}}_k, \tau_k) \right)}_{\text{Neural network approximation error } \mathcal{E}_{NN, k}},
\end{split}
\end{equation}
where $\Delta \tilde{\mathbf{v}}_k := \tilde{\mathbf{v}}(\hat{\mathbf{x}}_k, \tau_k) - \tilde{\mathbf{v}}(\mathbf{X}(\tau_k), \tau_k)$ represents the spatial error propagation due to the difference in trajectories. This decomposition separates the dynamics into:
\begin{itemize}
    \item \textbf{Deterministic Drift}: Bounds error propagation via the spatial Lipschitz continuity of $\tilde{\mathbf{v}}$ (\cref{cor:global_bound_Lv_uniform}).
    \item \textbf{Truncation Error $\mathcal{T}_k$}: Captures the $O(\Delta \tau^2)$ discretization error of the Euler scheme (\cref{lem:local_truncation_bound}).
    \item \textbf{Martingale Noise $\hat{\mathcal{M}}_k$}: Unbiased noise from sampling component velocities. Crucially, $\mathbb{E}[\hat{\mathcal{M}}_k \mid \mathcal{F}_k] = \mathbf{0}$. Its second moment is controlled by \cref{lem:martingale_bound_nn}.
    \item \textbf{NN Approximation $\mathcal{E}_{NN, k}$}: The bias due to $L^\infty$ training error $\varepsilon$ (\cref{thm:nn_velocity_error_bound}).
\end{itemize}

\textbf{Step 3: Recursive Stability Analysis.}

Squaring both sides of \cref{eq:error_decomposition_main} and taking expectations, the martingale property ($\mathbb{E}[\hat{\mathcal{M}}_k \mid \mathcal{F}_k] = \mathbf{0}$) ensures that all cross-terms involving $\hat{\mathcal{M}}_k$ vanish. This leads to the fundamental decomposition:
\begin{equation}
\label{eq:mse_decomposition_step}
\begin{aligned}
\mathbb{E}[\|\hat{\mathbf{e}}_{k+1}\|^2] 
&= \mathbb{E}[\|\hat{\mathbf{e}}_k + \Delta \tau \Delta \tilde{\mathbf{v}}_k + \mathcal{T}_k + \mathcal{E}_{NN, k}\|^2] \\
&\quad + \mathbb{E}[\|\hat{\mathcal{M}}_k\|^2].
\end{aligned}
\end{equation}
To handle the remaining deterministic terms, we apply Young's inequality twice with weight $\eta = \Delta \tau$.

\textit{First Application (Separating $\mathcal{T}_k$):}
\begin{equation}
\label{eq:young1}
\begin{split}
\mathbb{E}[\|\hat{\mathbf{e}}_{k+1}\|^2] &\le (1+\Delta \tau) \mathbb{E}[\|\hat{\mathbf{e}}_k + \Delta \tau \Delta \tilde{\mathbf{v}}_k + \mathcal{E}_{NN, k}\|^2] \\
&\quad + \left(1+\frac{1}{\Delta \tau}\right) \mathbb{E}[\|\mathcal{T}_k\|^2] + \mathbb{E}[\|\hat{\mathcal{M}}_k\|^2].
\end{split}
\end{equation}

Since $\mathbb{E}[\|\mathcal{T}_k\|^2] = O(\Delta \tau^4)$, the second term contributes $O(\Delta \tau^3)$. Additionally, \cref{lem:martingale_bound_nn} ensures $\mathbb{E}[\|\hat{\mathcal{M}}_k\|^2] \le C_{\mathcal{M}} \Delta \tau^2 (1 + \|\hat{\mathbf{x}}_0\|^2)$, which is $O(\Delta \tau^2)$.

\textit{Second Application (Separating $\mathcal{E}_{NN, k}$):}
\begin{equation}
\label{eq:young2}
\begin{aligned}
&\mathbb{E}[\|\hat{\mathbf{e}}_k + \Delta \tau \Delta \tilde{\mathbf{v}}_k + \mathcal{E}_{NN, k}\|^2] \\
&\quad \le (1+\Delta \tau) \mathbb{E}[\|\hat{\mathbf{e}}_k + \Delta \tau \Delta \tilde{\mathbf{v}}_k\|^2] \\
&\qquad + \left(1+\frac{1}{\Delta \tau}\right) \mathbb{E}[\|\mathcal{E}_{NN, k}\|^2].
\end{aligned}
\end{equation}
Using the Lipschitz condition $\|\Delta \tilde{\mathbf{v}}_k\| \le L_{x,\delta} \|\hat{\mathbf{e}}_k\|$, the core drift term satisfies:
\begin{equation}
\mathbb{E}[\|\hat{\mathbf{e}}_k + \Delta \tau \Delta \tilde{\mathbf{v}}_k\|^2] \le (1 + 2L_{x,\delta} \Delta \tau + L_{x,\delta}^2 \Delta \tau^2) \hat{E}_k.
\end{equation}
Combining these with the NN error bound $\|\mathcal{E}_{NN, k}\|^2 = O(\varepsilon^2 \Delta \tau^2 (1+\hat{E}_k))$, we arrive at the standard recursion:
\begin{equation}
\label{eq:mse_recursion_standard}
\hat{E}_{k+1} \le (1 + \tilde{C}_1 \Delta \tau) \hat{E}_k + \Gamma_1 \varepsilon^2 \Delta \tau + \Gamma_2 \Delta \tau^2 + \Gamma_3(\mathbf{z}) \Delta \tau^3,
\end{equation}
where $\tilde{C}_1 = (2L_{x,\delta} + \Delta \tau_0 L_{x,\delta}^2)(1+\Delta \tau_0)^2 + 4C_{NN}^2(1+\Delta \tau_0)^2 \varepsilon^2$ is the growth exponent.

\textbf{Step 4: Global Bound via Discrete Grönwall Lemma.}

By summing the contributions from each step $k=0, \dots, N-1$, the $O(\Delta \tau^2)$ and $O(\Delta \tau^3)$ terms aggregate to $O(\Delta \tau)$ and $O(\Delta \tau^2)$ respectively over the finite interval $[0, T_\delta]$. Applying the discrete Grönwall inequality to \cref{eq:mse_recursion_standard} over $N$ steps:
\begin{equation}
\label{eq:gronwall_final}
\sup_{k \le N} \hat{E}_k \le \underbrace{\frac{e^{\tilde{C}_1 T_{\delta}} - 1}{\tilde{C}_1}}_{\mathcal{S}(T_{\delta})} \left( \mathcal{C}_{approx}(\mathbf{z}) \varepsilon^2 + \mathcal{C}_{disc}(\mathbf{z}) \Delta \tau \right).
\end{equation}
The stability factor $\mathcal{S}(T_{\delta})$ reflects the system's sensitivity to perturbations. This establishes \cref{thm:strong_convergence_main}, showing that pathwise error consists of a non-vanishing approximation part and a vanishing discretization part.

\subsection{Stability Factor and Error Constants}
\label{subsec:constants_discussion}

The global convergence bound in \cref{thm:w2_convergence_main} is characterized by several fundamental constants that reflect the interplay between data geometry, noise scheduling, and neural network complexity.The explicit forms of these constants are provided in the appendix.

\textbf{Stability Factor $\mathcal{S}(T_{\delta})$.} The stability factor $\mathcal{S}(T_{\delta}) = (e^{\tilde{C}_1 T_{\delta}} - 1) / \tilde{C}_1$ represents the cumulative sensitivity of the ODE to perturbations. The growth exponent $\tilde{C}_1$ is dominated by the uniform spatial Lipschitz constant $L_{x,\delta}$, which for our framework is given by \cref{cor:global_bound_Lv_uniform}:
\begin{equation}
L_{x,\delta} = C_{\alpha,1} \frac{\alpha_\delta}{\sigma_\delta^2} \left( 1 + \frac{R^2}{\sigma_\delta^2} \right).
\end{equation}

\textbf{Approximation Constant $\mathcal{C}_{approx}$.} This constant characterizes the amplification of $L^\infty$ training error $\varepsilon$ into trajectory $L^2$ error. It explicitly depends on the neural network approximation constant $C_{NN}$ defined as:
\begin{equation}
C_{NN} = \frac{C_{\alpha,1}}{\sigma_{\delta}^2} \left[ L \left( 1 + R + \frac{R^2}{\sigma_{\delta}^2} \right) + 1 \right].
\end{equation}
Notably, $C_{NN}$ scales linearly with the number of classes $L$, which corresponds to the complexity of the "Mix-of-Velocity" field. This linear dependence is a direct consequence of our weighted averaging of mixture components.

\textbf{Discretization Constant $\mathcal{C}_{disc}$.} The discretization constant consists of two parts: the martingale noise constant $C_{\mathcal{M}}$ and the truncation error coefficient $C_{trunc}$.
\begin{itemize}
    \item $C_{\mathcal{M}}$ accounts for the variance introduced by the stochastic selection of component velocities. It depends on the linear growth coefficient $M'_{\delta}$ of the neural network velocity field.
    \item $C_{trunc}$ arises from the Euler discretization of the probability flow ODE. Its magnitude is determined by the temporal derivatives of the velocity field, which involve the second derivative $C_{\alpha,2}$ of the noise schedule $\alpha_t$.
\end{itemize}
By integrating these over $\mathbf{Z} \sim \mathcal{N}(\mathbf{0}, \mathbf{I})$, we obtain the dimension-dependent bound in \cref{thm:w2_convergence_main}. The $O(d)$ and $O(\sqrt{d})$ dependencies in $\mathcal{C}_{approx}$ and $\mathcal{C}_{disc}$ arise from the Gaussian second and first moments, respectively, indicating how the error scales as we move to higher-dimensional data spaces.

\subsection{Distribution-Level Analysis: $W_2$ Convergence}
\label{subsec:w2_analysis_main}

To prove \cref{thm:w2_convergence_main}, we integrate the trajectory-level MSE over the Gaussian prior $\pi = \mathcal{N}(\mathbf{0}, \mathbf{I})$.

\textbf{Step 1: Coupling Bound for Wasserstein Distance.}
From the definition of $W_2$ and the natural coupling constructed above:
\begin{equation}
\label{eq:w2_coupling_integral}
W_2^2(\hat{\mu}_N, \mu_*) \le \int_{\mathbb{R}^d} \hat{E}_N(\mathbf{z}) \, d\pi(\mathbf{z}) = \mathbb{E}_{\mathbf{Z} \sim \pi} \left[ \hat{E}_N(\mathbf{Z}) \right].
\end{equation}
The challenge lies in the dependence of $\mathcal{C}_{approx}(\mathbf{Z})$ and $\mathcal{C}_{disc}(\mathbf{Z})$ on the initial state $\mathbf{Z}$.

\textbf{Step 2: Tracking Initial-Value Dependence.}
The error constants scale with the norm of the initial state as follows:
\begin{itemize}
    \item \textbf{Approximation constant}: $\mathcal{C}_{approx}(\mathbf{z})$ depends on the effective data manifold radius $R_0(\mathbf{z}) = \|\mathbf{z}\| + R$, resulting in a quadratic dependence on $\|\mathbf{z}\|$.
    \item \textbf{Discretization constant}: $\mathcal{C}_{disc}(\mathbf{z})$ involves the martingale variance $C_{\mathcal{M}}(1+\|\mathbf{z}\|^2)$ and the squared truncation coefficient $C_{trunc}^2(\mathbf{z})$, which is also quadratic in $\|\mathbf{z}\|$.
\end{itemize}

\textbf{Step 3: Moment Integration and Final Result.}
Since $\mathbf{Z} \sim \mathcal{N}(\mathbf{0}, \mathbf{I}_d)$, we utilize the standard moments $\mathbb{E}[\|\mathbf{Z}\|^2] = d$ and $\mathbb{E}[\|\mathbf{Z}\|] \le \sqrt{d}$. Integrating \cref{eq:gronwall_final} yields:
\begin{equation}
\label{eq:final_w2_bound}
    \begin{aligned}
    &W_2^2(\hat{\mu}_N, \mu_*) \\
    &\le \mathcal{S}(T_{\delta}) \left( \mathbb{E}_{\mathbf{Z}}[\mathcal{C}_{approx}(\mathbf{Z})] \varepsilon^2 + \mathbb{E}_{\mathbf{Z}}[\mathcal{C}_{disc}(\mathbf{Z})] \frac{T_{\delta}}{N} \right) \\
    &\le \mathcal{S}(T_{\delta}) \left( \mathcal{C}_{approx’} \varepsilon^2 + \frac{\mathcal{C}_{disc’}}{N} \right).
    \end{aligned}
\end{equation}

The final constants $\mathcal{C}_{approx}$ and $\mathcal{C}_{disc}$ depend on the problem geometry $(d, R)$ and mixture complexity $L$, but are universal for the distribution. This completes the proof of \cref{thm:w2_convergence_main}, demonstrating that the Mix-of-Velocity algorithm converges at rate $O(N^{-1/2})$ in $W_2$ distance while maintaining robustness to neural network approximation errors.

\section{Discussion and Limitations}
\label{sec:discussion}

Our theoretical framework establishes the first rigorous Wasserstein-2 convergence guarantees for decentralized diffusion models using ODE-based samplers. By modeling the routing mechanism as a martingale difference sequence, we bridge the gap between deterministic ODE theory and the stochastic nature of decentralized architectures. However, several theoretical challenges and avenues for future research remain.

\textbf{Norm Discrepancy in Approximation Errors.}
Our convergence analysis relies on the assumption that the neural velocity approximation error is bounded in the $L^\infty$ norm (Assumption \ref{ass:nn_approx}). This uniform bound is crucial for controlling the deterministic drift of the particle trajectories via Grönwall-type inequalities. However, standard score matching and flow matching training objectives typically minimize an $L^2$-weighted error. Bridging this discrepancy---establishing trajectory-wise convergence rates under weaker $L^2$ error assumptions---remains a fundamental open problem in the broader literature of diffusion models, not limited to the decentralized setting.

\textbf{Singularity near Data Manifold ($t \to 0$).}
As indicated in our main theorem, the convergence rate depends on the inverse of the truncation time $\delta$. This dependence reflects the inherent singularity of the score function (and consequently the velocity field) as the diffusion time approaches zero and the distribution collapses onto a lower-dimensional data manifold. While our use of $\delta$ serves as a principled early-stopping mechanism similar to \citet{chen2023sampling}, developing techniques to handle this "explosive" behavior in our decentralized setting would be a significant theoretical advancement.

\textbf{Higher-Order Training for Categorical Guidance.}
While our theoretical analysis focuses on the velocity field $\mathbf{v}(\mathbf{x}, t)$ (the first-order temporal derivative of the flow), our experimental results reveal an intriguing phenomenon regarding the posterior probability network $a(\mathbf{x}, t)$. Specifically, training a network to predict the temporal derivative $\frac{da}{dt}$ and then integrating via ODE during inference produces superior categorical guidance compared to directly learning $a(\mathbf{x}, t)$ through classification objectives. This mirrors the advantage of score-based models over directly modeling densities: temporal derivatives provide a more stable training signal and facilitate better trajectory tracking during few-step sampling.

This observation suggests a general principle: \emph{higher-order training objectives can improve ODE-based inference even when the quantity of interest is lower-order}. Extending our convergence analysis (\cref{thm:strong_convergence_precise_v2}) to formally characterize the benefits of learning $\frac{da_\phi}{dt}$ versus $a_\phi$ directly---particularly regarding how the temporal derivative regularization (\cref{eq:reg_loss}) affects the approximation error $\varepsilon$ and consequently the overall convergence rate---is an important direction for future theoretical work. Such analysis would provide rigorous justification for the empirical gains observed in \cref{tab:fid_results,fig:error_scaling}.

\section{Conclusion}

\label{sec:conclusion}

In this paper, We studied the theoretical convergence of probability flow ODE (PF-ODE) sampling for diffusion/flow-matching models in decentralized settings where the data distribution admits a mixture structure and the global generation dynamics are induced by combining multiple  experts. To bridge the gap between practical mixture-of-experts samplers and classical ODE convergence theory (which typically assumes a single smooth global field), we introduced a velocity field decomposition framework that represents the global velocity as a posterior-weighted superposition of component-wise velocity fields.

Under standard regularity conditions and explicit uniform approximation guarantees for both the routing weights and the component predictors, we established a non-asymptotic convergence theory for the resulting discrete-time sampler. Our main result provides an explicit Wasserstein-2 bound between the generated distribution and the reference PF-ODE solution at a truncated time \(t=\delta\), with the total error decomposing naturally into (i) a neural approximation term controlled by \(\varepsilon\) and (ii) a discretization term controlled by the step number \(N\). Importantly, the constants in the bound reveal how the convergence depends on the problem geometry (dimension \(d\), support radius \(R\)), the mixture complexity \(L\), and the truncation/noise level \(\sigma_\delta\), thereby offering a quantitative baseline for decentralized diffusion architectures driven by mixture dynamics.

%\section*{Ethics statement}
%As a nearly pure theoretical work, we feel our work is less likely to suffer from ethical issues.
%The authors have read and comply with the ICML Code of Ethics. This research does not involve
%human subjects or personally identifiable information. 
%We do not foresee harmful or dual-use implications from the proposed methods. There are no conflicts of interest or undisclosed sponsorship.
\section*{Impact Statement}
As a nearly pure theoretical work, we feel our work is less likely to suffer from ethical issues.
The authors have read and comply with the ICML Research Ethics. This research does not involve
human subjects or personally identifiable information. 
We do not foresee harmful or dual-use implications from the proposed methods. There are no conflicts of interest or undisclosed sponsorship.

This paper presents work whose goal is to advance the field of Machine
Learning. There are many potential societal consequences of our work, none
which we feel must be specifically highlighted here.

\bibliographystyle{unsrtnat}
\bibliography{references}  %%% Uncomment this line and comment out the ``thebibliography'' section below to use the external .bib file (using bibtex) .

\newpage
\appendix
\onecolumn
\section{Properties of Score Functions}
\label{sec:score_properties}

This section establishes key regularity properties of score functions that are essential for our convergence analysis. We prove spatial Lipschitz continuity and temporal derivative bounds under the assumptions stated in \cref{sec:convergence_theory}.

\subsection{Spatial Lipschitz Continuity of Score Functions}

\begin{proposition}[Global Lipschitz Continuity of Score Functions with Bounded Support]
\label{lem:score_lipschitz_bounded}
Under \cref{ass:p0_bounded} and \cref{ass:alpha_regularity}, for the score function $\mathbf{s}(\mathbf{x}, t) := \nabla_{\mathbf{x}} \log p_t(\mathbf{x})$ defined in \cref{def:score_function}, the following holds:

For any fixed $t>0$, $\mathbf{s}(\cdot,t)$ is globally Lipschitz continuous on $\mathbb{R}^d$ with Lipschitz constant:
\begin{equation}
\label{eq:Ls_bound_bounded}
L_s(t|\alpha_t,R) := \sup_{\mathbf{x}}\|\nabla_{\mathbf{x}}^2\log p_t(\mathbf{x})\|
\le \frac{1}{\sigma_t^2} + \frac{\alpha_t^2 R^2}{\sigma_t^4}
= \frac{1}{1-\alpha_t^2} + \frac{\alpha_t^2 R^2}{(1-\alpha_t^2)^2}.
\end{equation}
Consequently, for any $\mathbf{x}_1, \mathbf{y}_1 \in \mathbb{R}^d$:
\begin{equation}
\label{eq:score_lipschitz}
\|\mathbf{s}(\mathbf{x}_1,t) - \mathbf{s}(\mathbf{y}_1,t)\| \le L_s(t|\alpha_t,R)\,\|\mathbf{x}_1-\mathbf{y}_1\|.
\end{equation}
\end{proposition}

\begin{proof}
Let $\sigma_t^2 = 1-\alpha_t^2$. The marginal density $p_t(\mathbf{x})$ is given by convolution of the initial distribution with a Gaussian transition kernel:
\begin{equation}
\label{eq:p_t_definition}
p_t(\mathbf{x}) = \int p_0(\mathbf{x}_0) \, \mathcal{N}(\mathbf{x}; \alpha_t \mathbf{x}_0, \sigma_t^2 \mathbf{I}) \, d\mathbf{x}_0.
\end{equation}
Denote the Gaussian kernel by $\varphi(\mathbf{x} | \mathbf{x}_0) := \mathcal{N}(\mathbf{x}; \alpha_t \mathbf{x}_0, \sigma_t^2 \mathbf{I}) = (2\pi\sigma_t^2)^{-d/2} \exp\left( -\frac{\|\mathbf{x} - \alpha_t \mathbf{x}_0\|^2}{2\sigma_t^2} \right)$.

Since the Gaussian kernel is smooth and $p_0$ has compact support, Leibniz's integral rule applies. The first and second derivatives of $\varphi$ with respect to $\mathbf{x}$ are:
\begin{align}
\nabla_{\mathbf{x}} \varphi(\mathbf{x} | \mathbf{x}_0) &= -\frac{\mathbf{x} - \alpha_t \mathbf{x}_0}{\sigma_t^2} \varphi(\mathbf{x} | \mathbf{x}_0), \label{eq:phi_grad1}\\
\nabla_{\mathbf{x}}^2 \varphi(\mathbf{x} | \mathbf{x}_0) &= \left( -\frac{1}{\sigma_t^2}\mathbf{I} + \frac{(\mathbf{x} - \alpha_t \mathbf{x}_0)(\mathbf{x} - \alpha_t \mathbf{x}_0)^\top}{\sigma_t^4} \right) \varphi(\mathbf{x} | \mathbf{x}_0). \label{eq:phi_grad2}
\end{align}

\paragraph{Step 1: Computing the Score Function}
The score function is defined as $\mathbf{s}(\mathbf{x}, t) = \nabla_{\mathbf{x}} \log p_t(\mathbf{x}) = \frac{\nabla p_t(\mathbf{x})}{p_t(\mathbf{x})}$.
Using \cref{eq:phi_grad1}:
\begin{equation}
\nabla p_t(\mathbf{x}) = \int p_0(\mathbf{x}_0) \nabla_{\mathbf{x}} \varphi(\mathbf{x} | \mathbf{x}_0) d\mathbf{x}_0 
= -\frac{1}{\sigma_t^2} \int (\mathbf{x} - \alpha_t \mathbf{x}_0) p(\mathbf{x}, \mathbf{x}_0) d\mathbf{x}_0.
\end{equation}
Define the posterior mean (denoising estimate):
\begin{equation}
\hat{\mathbf{x}}_0(\mathbf{x},t) := \mathbb{E}[\mathbf{x}_0 \mid \mathbf{x}_t = \mathbf{x}] = \frac{\int \mathbf{x}_0 p(\mathbf{x}, \mathbf{x}_0) d\mathbf{x}_0}{p_t(\mathbf{x})}.
\end{equation}
Dividing by $p_t(\mathbf{x})$ yields Tweedie's formula:
\begin{equation}
\label{eq:score_tweedie}
\mathbf{s}(\mathbf{x}, t) = -\frac{1}{\sigma_t^2} \left( \mathbf{x} - \alpha_t \hat{\mathbf{x}}_0(\mathbf{x},t) \right).
\end{equation}

\paragraph{Step 2: Computing the Hessian Matrix}
Taking the second derivative of $\log p_t(\mathbf{x})$ using the quotient rule:
\begin{equation}
\label{eq:hessian_quotient}
\nabla_{\mathbf{x}}^2 \log p_t(\mathbf{x}) = \frac{\nabla^2 p_t(\mathbf{x})}{p_t(\mathbf{x})} - \left( \frac{\nabla p_t(\mathbf{x})}{p_t(\mathbf{x})} \right) \left( \frac{\nabla p_t(\mathbf{x})}{p_t(\mathbf{x})} \right)^\top.
\end{equation}
Using \cref{eq:phi_grad2}, the first term becomes:
\begin{equation}
\frac{\nabla^2 p_t(\mathbf{x})}{p_t(\mathbf{x})} 
= \int \left( -\frac{1}{\sigma_t^2}\mathbf{I} + \frac{(\mathbf{x} - \alpha_t \mathbf{x}_0)(\mathbf{x} - \alpha_t \mathbf{x}_0)^\top}{\sigma_t^4} \right) p(\mathbf{x}_0 | \mathbf{x}) d\mathbf{x}_0,
\end{equation}
where $p(\mathbf{x}_0 | \mathbf{x}) = \frac{p(\mathbf{x}, \mathbf{x}_0)}{p_t(\mathbf{x})}$ is the posterior density.

For the second term in \cref{eq:hessian_quotient}, substituting \cref{eq:score_tweedie}:
\begin{equation}
\mathbf{s}(\mathbf{x}, t)\mathbf{s}(\mathbf{x}, t)^\top = \frac{1}{\sigma_t^4} \left( \mathbf{x} - \alpha_t \hat{\mathbf{x}}_0(\mathbf{x}) \right) \left( \mathbf{x} - \alpha_t \hat{\mathbf{x}}_0(\mathbf{x}) \right)^\top.
\end{equation}

Let $\mathbf{r} = \mathbf{x} - \alpha_t \mathbf{x}_0$. Note that:
\begin{equation}
\mathbb{E}[\mathbf{r} \mid \mathbf{x}] = \mathbf{x} - \alpha_t \mathbb{E}[\mathbf{x}_0 \mid \mathbf{x}] = \mathbf{x} - \alpha_t \hat{\mathbf{x}}_0(\mathbf{x}).
\end{equation}

By the covariance identity $\text{Cov}(\mathbf{Y}) = \mathbb{E}[\mathbf{Y}\mathbf{Y}^\top] - \mathbb{E}[\mathbf{Y}]\mathbb{E}[\mathbf{Y}]^\top$, subtracting the two terms in \cref{eq:hessian_quotient} yields:
\begin{equation}
\nabla_{\mathbf{x}}^2 \log p_t(\mathbf{x}) = -\frac{1}{\sigma_t^2}\mathbf{I} + \frac{1}{\sigma_t^4} \text{Cov}(\mathbf{x} - \alpha_t \mathbf{x}_0 \mid \mathbf{x}_t = \mathbf{x}).
\end{equation}

Since $\mathbf{x}$ is constant given $\mathbf{x}_t=\mathbf{x}$ (zero variance):
\begin{equation}
\text{Cov}(\mathbf{x} - \alpha_t \mathbf{x}_0 \mid \mathbf{x}) = \alpha_t^2 \text{Cov}(\mathbf{x}_0 \mid \mathbf{x}).
\end{equation}

Thus:
\begin{equation}
\label{eq:hessian_exact}
\nabla_{\mathbf{x}}^2 \log p_t(\mathbf{x}) = -\frac{1}{\sigma_t^2}\mathbf{I} + \frac{\alpha_t^2}{\sigma_t^4} \text{Cov}(\mathbf{x}_0 \mid \mathbf{x}_t = \mathbf{x}).
\end{equation}

\paragraph{Step 3: Bounding the Lipschitz Constant}
By \cref{ass:p0_bounded}, $\|\mathbf{x}_0\| \le R$ almost surely. For any unit vector $\mathbf{v} \in \mathbb{R}^d$:
\begin{equation}
\mathbf{v}^\top \mathbb{E}[\mathbf{x}_0 \mathbf{x}_0^\top \mid \mathbf{x}] \mathbf{v} = \mathbb{E}[(\mathbf{v}^\top \mathbf{x}_0)^2 \mid \mathbf{x}] \le \mathbb{E}[\|\mathbf{x}_0\|^2 \mid \mathbf{x}] \le R^2.
\end{equation}
Hence $\mathbb{E}[\mathbf{x}_0 \mathbf{x}_0^\top \mid \mathbf{x}] \preceq R^2 \mathbf{I}$.

Since $\text{Cov}(\mathbf{x}_0 \mid \mathbf{x}) \preceq \mathbb{E}[\mathbf{x}_0 \mathbf{x}_0^\top \mid \mathbf{x}]$, the spectral norm satisfies:
\begin{equation}
\| \text{Cov}(\mathbf{x}_0 \mid \mathbf{x}) \| \le R^2.
\end{equation}

Taking the spectral norm of \cref{eq:hessian_exact} and applying the triangle inequality:
\begin{equation}
\| \nabla_{\mathbf{x}}^2 \log p_t(\mathbf{x}) \| \le \frac{1}{\sigma_t^2} + \frac{\alpha_t^2 R^2}{\sigma_t^4}.
\end{equation}

By the Mean Value Theorem for vector-valued functions on the convex set $\mathbb{R}^d$:
\begin{equation}
\|\mathbf{s}(\mathbf{x}_1, t) - \mathbf{s}(\mathbf{y}_1, t)\| \le \sup_{\mathbf{z} \in \mathbb{R}^d} \| \nabla_{\mathbf{x}}^2 \log p_t(\mathbf{z}) \| \, \|\mathbf{x}_1 - \mathbf{y}_1\| \le L_s(t) \|\mathbf{x}_1 - \mathbf{y}_1\|,
\end{equation}
which completes the proof.
\end{proof}

\subsection{Temporal Derivative Bounds for Score Functions}

Having established spatial Lipschitz continuity in \cref{lem:score_lipschitz_bounded}, we now analyze the temporal behavior of score functions. Establishing bounds on the time derivative $\frac{\partial \mathbf{s}(\mathbf{x},t)}{\partial t}$ is crucial for analyzing existence and uniqueness of probability flow ODE solutions and for controlling discretization error.

\begin{proposition}[Linear Growth Bound for Time Derivative of Score Function]
\label{prop:score_time_deriv_bound}
Under \cref{ass:p0_bounded} and \cref{ass:alpha_regularity}, for any fixed $t \in (0, T]$, the time derivative of the score function $\mathbf{s}(\mathbf{x}, t) := \nabla_{\mathbf{x}} \log p_t(\mathbf{x})$ satisfies:
\begin{equation}
\label{eq:score_time_bound}
\left\| \frac{\partial \mathbf{s}(\mathbf{x}, t)}{\partial t} \right\| \le K_1(t) \|\mathbf{x}\| + K_0(t), \quad \forall \mathbf{x} \in \mathbb{R}^d,
\end{equation}
where the coefficients depend only on $\alpha_t$ and its derivatives:
\begin{align}
K_1(t) &= -\frac{d\alpha_t}{dt} \left[ \frac{2\alpha_t}{\sigma_t^4} + \frac{\alpha_t(1+\alpha_t^2)}{\sigma_t^6} R^2 \right], \label{eq:K1_def} \\
K_0(t) &= -\frac{d\alpha_t}{dt} \left[ \frac{1+\alpha_t^2}{\sigma_t^4} R + \frac{2\alpha_t^2}{\sigma_t^6} R^3 \right], \label{eq:K0_def}
\end{align}
with $R$ being the radius of the data support and $\sigma_t^2=1-\alpha_t^2$. Note that since $\frac{d\alpha_t}{dt} \le 0$ by assumption, both coefficients are positive.
\end{proposition}

\begin{proof}
We start from Tweedie's formula:
\begin{equation}
\label{eq:score_as_posterior_mean_redux}
\mathbf{s}(\mathbf{x}, t) = -\frac{1}{\sigma_t^2}\bigl(\mathbf{x} - \alpha_t \hat{\mathbf{x}}_0(\mathbf{x},t)\bigr),
\end{equation}
where $\hat{\mathbf{x}}_0(\mathbf{x},t) := \mathbb{E}[\mathbf{x}_0 \mid \mathbf{x}_t=\mathbf{x}]$ is the posterior mean.

By \cref{ass:alpha_regularity}, $\alpha_t$ is twice continuously differentiable with $\sigma_t^2 = 1 - \alpha_t^2$. We first compute the time derivatives of relevant quantities:
\begin{align}
\frac{d\sigma_t^2}{dt} &= \frac{d}{dt}(1-\alpha_t^2) = -2\alpha_t \frac{d\alpha_t}{dt}, \nonumber \\
\frac{d}{dt}\left(\frac{1}{\sigma_t^2}\right) &= -\frac{1}{\sigma_t^4} \frac{d\sigma_t^2}{dt} = \frac{2\alpha_t}{\sigma_t^4} \frac{d\alpha_t}{dt}. \label{eq:dt_inv_sigma}
\end{align}

Taking the partial derivative of \cref{eq:score_as_posterior_mean_redux} with respect to $t$ (treating $\mathbf{x}$ as constant) and applying the product rule:
\begin{align*}
\frac{\partial \mathbf{s}(\mathbf{x}, t)}{\partial t} &= -\left(\frac{d}{dt}\frac{1}{\sigma_t^2}\right) \bigl(\mathbf{x} - \alpha_t \hat{\mathbf{x}}_0\bigr) - \frac{1}{\sigma_t^2} \frac{\partial}{\partial t}\bigl(\mathbf{x} - \alpha_t \hat{\mathbf{x}}_0\bigr) \\
&= -\left( \frac{2\alpha_t}{\sigma_t^4} \frac{d\alpha_t}{dt} \right) \bigl(\mathbf{x} - \alpha_t \hat{\mathbf{x}}_0\bigr) - \frac{1}{\sigma_t^2} \left( - \frac{d\alpha_t}{dt}\hat{\mathbf{x}}_0 - \alpha_t\frac{\partial \hat{\mathbf{x}}_0}{\partial t} \right).
\end{align*}

Expanding and collecting terms with common factor $\frac{d\alpha_t}{dt}$:
\begin{align*}
\frac{\partial \mathbf{s}(\mathbf{x}, t)}{\partial t} 
&= -\frac{2\alpha_t}{\sigma_t^4}\frac{d\alpha_t}{dt} \mathbf{x} + \frac{2\alpha_t^2}{\sigma_t^4}\frac{d\alpha_t}{dt} \hat{\mathbf{x}}_0 + \frac{1}{\sigma_t^2}\frac{d\alpha_t}{dt}\hat{\mathbf{x}}_0 + \frac{\alpha_t}{\sigma_t^2}\frac{\partial \hat{\mathbf{x}}_0}{\partial t} \\
&= -\frac{2\alpha_t}{\sigma_t^4}\frac{d\alpha_t}{dt} \mathbf{x} + \left( \frac{2\alpha_t^2}{\sigma_t^4} + \frac{1}{\sigma_t^2} \right)\frac{d\alpha_t}{dt} \hat{\mathbf{x}}_0 + \frac{\alpha_t}{\sigma_t^2}\frac{\partial \hat{\mathbf{x}}_0}{\partial t}.
\end{align*}

Using $\frac{2\alpha_t^2}{\sigma_t^4} + \frac{1}{\sigma_t^2} = \frac{2\alpha_t^2 + (1-\alpha_t^2)}{\sigma_t^4} = \frac{1+\alpha_t^2}{\sigma_t^4}$:
\begin{equation}
\label{eq:score_dt_expanded}
\frac{\partial \mathbf{s}(\mathbf{x}, t)}{\partial t} = \underbrace{-\frac{2\alpha_t}{\sigma_t^4}\frac{d\alpha_t}{dt}}_{A(t)} \mathbf{x} + \underbrace{\frac{1+\alpha_t^2}{\sigma_t^4}\frac{d\alpha_t}{dt}}_{B(t)} \hat{\mathbf{x}}_0 + \frac{\alpha_t}{\sigma_t^2}\frac{\partial \hat{\mathbf{x}}_0}{\partial t}.
\end{equation}

By \cref{lem:mu_time_deriv_bound} (posterior mean time derivative bound), there exist coefficients $C_1(t)$ and $C_0(t)$ such that:
\begin{equation}
\left\|\frac{\partial \hat{\mathbf{x}}_0}{\partial t}\right\| \le C_1(t) \|\mathbf{x}\| + C_0(t),
\end{equation}
where:
\begin{align}
C_1(t) &= \frac{|\dot{\alpha}_t|(1+\alpha_t^2)}{\sigma_t^4} R^2=-\frac{d\alpha_t}{dt} \left( \frac{1+\alpha_t^2}{\sigma_t^4} \right) R^2, \\
C_0(t) &= \frac{2|\dot{\alpha}_t|\alpha_t}{\sigma_t^4} R^3=-\frac{2\alpha_t}{\sigma_t^4}\frac{d\alpha_t}{dt} R^3.
\end{align}

Taking the norm of \cref{eq:score_dt_expanded} and using $\|\hat{\mathbf{x}}_0\| \le R$:
\begin{align*}
\left\|\frac{\partial \mathbf{s}}{\partial t}\right\| 
&\le |A(t)| \|\mathbf{x}\| + |B(t)| R + \frac{\alpha_t}{\sigma_t^2} \left( C_1(t) \|\mathbf{x}\| + C_0(t) \right) \\
&= \left( -\frac{2\alpha_t}{\sigma_t^4}\frac{d\alpha_t}{dt} + \frac{\alpha_t}{\sigma_t^2}C_1(t) \right) \|\mathbf{x}\| + \left( \left| \frac{1+\alpha_t^2}{\sigma_t^4}\frac{d\alpha_t}{dt} \right| R + \frac{\alpha_t}{\sigma_t^2}C_0(t) \right).
\end{align*}

Since $\frac{d\alpha_t}{dt} \le 0$, we have $|B(t)| = -\frac{1+\alpha_t^2}{\sigma_t^4}\frac{d\alpha_t}{dt}$. Computing the coefficients:

For $K_1(t)$ (coefficient of $\|\mathbf{x}\|$):
\begin{align*}
K_1(t) &= -\frac{2\alpha_t}{\sigma_t^4}\frac{d\alpha_t}{dt} + \frac{\alpha_t}{\sigma_t^2} \left[ -\frac{d\alpha_t}{dt} \left( \frac{1+\alpha_t^2}{\sigma_t^4} \right) R^2 \right] \\
&= -\frac{d\alpha_t}{dt} \left[ \frac{2\alpha_t}{\sigma_t^4} + \frac{\alpha_t(1+\alpha_t^2)}{\sigma_t^6} R^2 \right].
\end{align*}

For $K_0(t)$ (constant term):
\begin{align*}
K_0(t) &= -\frac{1+\alpha_t^2}{\sigma_t^4}\frac{d\alpha_t}{dt} R + \frac{\alpha_t}{\sigma_t^2} \left[ -\frac{2\alpha_t}{\sigma_t^4}\frac{d\alpha_t}{dt} R^3 \right] \\
&= -\frac{d\alpha_t}{dt} \left[ \frac{1+\alpha_t^2}{\sigma_t^4} R + \frac{2\alpha_t^2}{\sigma_t^6} R^3 \right].
\end{align*}

This establishes the claimed bound.
\end{proof}

\subsection{Time Derivative Bound for Posterior Mean}

This section provides a detailed derivation of bounds on the time derivative of the posterior mean $\hat{\mathbf{x}}_0(\mathbf{x}, t) := \mathbb{E}[\mathbf{x}_0 \mid \mathbf{x}_t=\mathbf{x}]$, which is essential for the analysis of score function time derivatives.

\begin{lemma}[Linear Bound for Time Derivative of Posterior Mean]
\label{lem:mu_time_deriv_bound}
Under \cref{ass:p0_bounded} and \cref{ass:alpha_regularity}, the time derivative of the posterior mean satisfies:
\begin{equation}
\label{eq:mu_time_bound}
\left\|\frac{\partial \hat{\mathbf{x}}_0(\mathbf{x}, t)}{\partial t}\right\| \le C_1(t) \|\mathbf{x}\| + C_0(t), \quad \forall t \in (0, 1],
\end{equation}
where the coefficients depend only on $\alpha_t$ and its derivative $\dot{\alpha}_t$:
\begin{align}
C_1(t) &= \frac{|\dot{\alpha}_t|(1+\alpha_t^2)}{\sigma_t^4} R^2, \\
C_0(t) &= \frac{2|\dot{\alpha}_t|\alpha_t}{\sigma_t^4} R^3,
\end{align}
with $\sigma_t^2 = 1 - \alpha_t^2$.
\end{lemma}

\begin{proof}
The posterior mean can be written as $\hat{\mathbf{x}}_0(\mathbf{x}, t) = \mathbf{N}(\mathbf{x},t)/D(\mathbf{x},t)$, where:
\begin{align*}
\mathbf{N}(\mathbf{x},t) &:= \int \mathbf{x}_0 \, p_0(\mathbf{x}_0) \, \varphi(\mathbf{x} \mid \mathbf{x}_0) \, d\mathbf{x}_0, \\
D(\mathbf{x},t) &:= p_t(\mathbf{x}) = \int p_0(\mathbf{x}_0) \, \varphi(\mathbf{x} \mid \mathbf{x}_0) \, d\mathbf{x}_0,
\end{align*}
with $\varphi(\mathbf{x} \mid \mathbf{x}_0) := \mathcal{N}(\mathbf{x}; \alpha_t \mathbf{x}_0, \sigma_t^2 \mathbf{I})$.

Applying the quotient rule:
\begin{equation}
\label{eq:quotient_rule_expansion}
\frac{\partial \hat{\mathbf{x}}_0}{\partial t} = \frac{(\partial_t \mathbf{N}) D - \mathbf{N} (\partial_t D)}{D^2} = \frac{\partial_t \mathbf{N}}{D} - \hat{\mathbf{x}}_0 \frac{\partial_t D}{D}.
\end{equation}

Using the identity $\frac{\partial \varphi}{\partial t} = \varphi \frac{\partial \log \varphi}{\partial t}$ and defining $Q(\mathbf{x}, \mathbf{x}_0, t) := \frac{\partial}{\partial t} \log \varphi(\mathbf{x} \mid \mathbf{x}_0)$:
\begin{equation}
\partial_t \mathbf{N} = \int \mathbf{x}_0 \, p_0(\mathbf{x}_0) \, \varphi(\mathbf{x} \mid \mathbf{x}_0) \, Q(\mathbf{x}, \mathbf{x}_0, t) \, d\mathbf{x}_0.
\end{equation}

Dividing by $D = p_t(\mathbf{x})$ and noting that $\frac{p_0(\mathbf{x}_0)\varphi(\mathbf{x}\mid \mathbf{x}_0)}{p_t(\mathbf{x})} = p(\mathbf{x}_0 \mid \mathbf{x}_t=\mathbf{x})$:
\begin{equation}
\frac{\partial_t \mathbf{N}}{D} = \mathbb{E}[\mathbf{x}_0 Q(\mathbf{x}, \mathbf{x}_0, t) \mid \mathbf{x}_t=\mathbf{x}].
\end{equation}

Similarly:
\begin{equation}
\frac{\partial_t D}{D} = \mathbb{E}[Q(\mathbf{x}, \mathbf{x}_0, t) \mid \mathbf{x}_t=\mathbf{x}].
\end{equation}

Substituting into \cref{eq:quotient_rule_expansion} and using the covariance identity:
\begin{equation}
\label{eq:mu_deriv_as_cov}
\frac{\partial \hat{\mathbf{x}}_0(\mathbf{x}, t)}{\partial t} = \mathrm{Cov}\big(\mathbf{x}_0, Q(\mathbf{x}, \mathbf{x}_0, t) \mid \mathbf{x}_t=\mathbf{x}\big).
\end{equation}

To compute $Q$, we differentiate $\log \varphi(\mathbf{x} \mid \mathbf{x}_0) = -\frac{d}{2}\log(2\pi\sigma_t^2) - \frac{\|\mathbf{x}-\alpha_t \mathbf{x}_0\|^2}{2\sigma_t^2}$ with respect to $t$. Using the identity $\dot{\sigma}_t^2 = -2\alpha_t \dot{\alpha}_t$, we analyze the two terms separately:

\begin{enumerate}
    \item \textbf{Normalization term}:
    \begin{equation*}
    \frac{\partial}{\partial t} \left( -\frac{d}{2} \log(2\pi\sigma_t^2) \right) = -\frac{d}{2\sigma_t^2} (-2\alpha_t \dot{\alpha}_t) = \frac{d \alpha_t \dot{\alpha}_t}{\sigma_t^2},
    \end{equation*}
    which is independent of $\mathbf{x}_0$ and will vanish in the covariance calculation.

    \item \textbf{Quadratic term}:
    Let $U = \|\mathbf{x}-\alpha_t \mathbf{x}_0\|^2$. Using the quotient rule $\partial_t (-\frac{U}{2\sigma_t^2}) = - \frac{(\partial_t U) \sigma_t^2 - U (\partial_t \sigma_t^2)}{2\sigma_t^4}$:
    \begin{itemize}
        \item $\partial_t U = \partial_t (\|\mathbf{x}\|^2 - 2\alpha_t \mathbf{x}^\top \mathbf{x}_0 + \alpha_t^2 \|\mathbf{x}_0\|^2) = -2\dot{\alpha}_t \mathbf{x}^\top \mathbf{x}_0 + 2\alpha_t \dot{\alpha}_t \|\mathbf{x}_0\|^2$.
        \item $\partial_t \sigma_t^2 = -2\alpha_t \dot{\alpha}_t$.
    \end{itemize}
    Substituting these into the numerator:
    \begin{align*}
    \text{Num} &= (-2\dot{\alpha}_t \mathbf{x}^\top \mathbf{x}_0 + 2\alpha_t \dot{\alpha}_t \|\mathbf{x}_0\|^2)\sigma_t^2 - (\|\mathbf{x}-\alpha_t \mathbf{x}_0\|^2)(-2\alpha_t \dot{\alpha}_t) \\
    &= -2\sigma_t^2 \dot{\alpha}_t \mathbf{x}^\top \mathbf{x}_0 + 2\sigma_t^2 \alpha_t \dot{\alpha}_t \|\mathbf{x}_0\|^2 + 2\alpha_t \dot{\alpha}_t (\|\mathbf{x}\|^2 - 2\alpha_t \mathbf{x}^\top \mathbf{x}_0 + \alpha_t^2 \|\mathbf{x}_0\|^2).
    \end{align*}
    Grouping by powers of $\mathbf{x}_0$:
    \begin{itemize}
        \item Coefficient of $\|\mathbf{x}_0\|^2$: $2\sigma_t^2 \alpha_t \dot{\alpha}_t + 2\alpha_t^3 \dot{\alpha}_t = 2\alpha_t \dot{\alpha}_t (\sigma_t^2 + \alpha_t^2) = 2\alpha_t \dot{\alpha}_t$.
        \item Coefficient of $\mathbf{x}^\top \mathbf{x}_0$: $-2\sigma_t^2 \dot{\alpha}_t - 4\alpha_t^2 \dot{\alpha}_t = -2\dot{\alpha}_t (\sigma_t^2 + 2\alpha_t^2) = -2\dot{\alpha}_t (1 + \alpha_t^2)$.
    \end{itemize}
\end{enumerate}

Combining these and dividing by $2\sigma_t^4$, we obtain the expression for $Q$:
\begin{equation}
Q(\mathbf{x}, \mathbf{x}_0, t) = \underbrace{\frac{\alpha_t \dot{\alpha}_t}{\sigma_t^4}}_{C_{\text{quad}}(t)} \|\mathbf{x}_0\|^2 + \underbrace{\left[ -\frac{\dot{\alpha}_t (1+\alpha_t^2)}{\sigma_t^4} \mathbf{x} \right]^\top}_{\mathbf{c}_{\text{lin}}(\mathbf{x}, t)} \mathbf{x}_0 + \text{const}.
\end{equation}

Substituting into \cref{eq:mu_deriv_as_cov}:
\begin{equation}
\frac{\partial \hat{\mathbf{x}}_0}{\partial t} = C_{\text{quad}}(t) \, \mathrm{Cov}\big(\mathbf{x}_0, \|\mathbf{x}_0\|^2 \mid \mathbf{x}\big) + \mathrm{Cov}\big(\mathbf{x}_0, \mathbf{x}_0 \mid \mathbf{x}\big) \, \mathbf{c}_{\text{lin}}(\mathbf{x}, t).
\end{equation}

Using $\|\mathbf{x}_0\| \le R$ from \cref{ass:p0_bounded}:

\textbf{(1) Linear term bound:} Using $\|\mathrm{Cov}(\mathbf{x}_0 \mid \mathbf{x})\|_2 \le R^2$:
\begin{equation}
\| \mathrm{Cov}(\mathbf{x}_0 \mid \mathbf{x}) \, \mathbf{c}_{\text{lin}} \| \le R^2 \frac{|\dot{\alpha}_t|(1+\alpha_t^2)}{\sigma_t^4} \|\mathbf{x}\| = C_1(t) \|\mathbf{x}\|.
\end{equation}

\textbf{(2) Quadratic term bound:} Using third moment bounds $\|\mathrm{Cov}(\mathbf{x}_0, \|\mathbf{x}_0\|^2 \mid \mathbf{x})\| \le 2R^3$:
\begin{equation}
\| C_{\text{quad}}(t) \, \mathrm{Cov}(\mathbf{x}_0, \|\mathbf{x}_0\|^2 \mid \mathbf{x}) \| \le \frac{2 \alpha_t |\dot{\alpha}_t|}{\sigma_t^4} R^3 = C_0(t).
\end{equation}

Combining via the triangle inequality yields \cref{eq:mu_time_bound}.
\end{proof}
\section{Properties of Velocity Fields in Probability Flow ODEs}
\label{sec:velocity_properties}

Under the notation and assumptions established in \cref{sec:convergence_theory}, we prove that both the mixture velocity field $\overline{\mathbf{v}}(\mathbf{x},t) = \sum_{i=1}^L a^i(\mathbf{x}, t) \mathbf{v}^i(\mathbf{x}, t)$ and its component fields $\mathbf{v}^i(\mathbf{x},t)$ satisfy the following regularity properties.

\subsection{Spatial Lipschitz Continuity of Velocity Fields}

\begin{theorem}[Spatial Lipschitz Continuity of Velocity Fields]
\label{thm:v_Lipschitz_improved}
Under \cref{ass:p0_bounded} and \cref{ass:alpha_regularity}, for any fixed $t \in (0, T]$, the velocity field $\mathbf{v}(\mathbf{x}, t)$ is globally Lipschitz continuous with respect to the spatial variable $\mathbf{x}$. 

There exists a constant $L_x(t|C_{\alpha,1},\alpha_t,R)$ such that for all $\mathbf{x}, \mathbf{y} \in \mathbb{R}^d$:
\begin{equation}
\|\mathbf{v}(\mathbf{x}, t) - \mathbf{v}(\mathbf{y}, t)\| \le L_x(t|C_{\alpha,1},\alpha_t,R) \|\mathbf{x} - \mathbf{y}\|.
\end{equation}
The Lipschitz constant admits the following upper bound:
\begin{equation}
\label{eq:Lv_bound_tight}
L_x(t|C_{\alpha,1},\alpha_t, R) = -\frac{\alpha_t}{\sigma_t^2} \dot{\alpha}_t \left( 1 + \frac{R^2}{\sigma_t^2} \right),
\end{equation}
where $\sigma_t^2 = 1 - \alpha_t^2$, $\alpha_t \in (0,1]$, $\dot{\alpha}_t := \frac{d\alpha_t}{dt}$, and $R$ is the radius of the initial data support. 

We may abbreviate $L_x(t|C_{\alpha,1},\alpha_t,R)$ as $L_x(t)$ when the dependence on system parameters is clear from context.
\end{theorem}

\begin{proof}
From \cref{def:velocity_field}, the velocity field is expressed as:
\begin{equation}
\mathbf{v}(\mathbf{x}, t) = f(t) \left[ \mathbf{x} + \mathbf{s}(\mathbf{x}, t) \right],
\end{equation}
where  $\mathbf{s}(\mathbf{x}, t) = \nabla_{\mathbf{x}} \log p_t(\mathbf{x})$ is the score function.

To derive the Lipschitz constant, we examine the spectral norm of the Jacobian matrix $\nabla_{\mathbf{x}} \mathbf{v}(\mathbf{x}, t)$:
\begin{equation}
\label{eq:grad_v_def}
\nabla_{\mathbf{x}} \mathbf{v}(\mathbf{x}, t) = f(t) \left[ \mathbf{I} + \nabla_{\mathbf{x}} \mathbf{s}(\mathbf{x}, t) \right],
\end{equation}
where $\mathbf{I}$ is the $d \times d$ identity matrix and $\nabla_{\mathbf{x}} \mathbf{s}(\mathbf{x}, t) = \nabla_{\mathbf{x}}^2 \log p_t(\mathbf{x})$ is the Hessian of the log-density.

Using the exact Hessian expression from \cref{eq:hessian_exact} (derived in the proof of \cref{lem:score_lipschitz_bounded}):
\begin{equation}
\nabla_{\mathbf{x}} \mathbf{s}(\mathbf{x}, t) = -\frac{1}{\sigma_t^2}\mathbf{I} + \frac{\alpha_t^2}{\sigma_t^4} \mathrm{Cov}(\mathbf{x}_0 \mid \mathbf{x}_t = \mathbf{x}).
\end{equation}

Substituting into \cref{eq:grad_v_def}, observe the crucial algebraic cancellation:
\begin{equation}
\mathbf{I} - \frac{1}{\sigma_t^2}\mathbf{I} = \left( 1 - \frac{1}{\sigma_t^2} \right)\mathbf{I} = \frac{\sigma_t^2 - 1}{\sigma_t^2}\mathbf{I} = -\frac{\alpha_t^2}{\sigma_t^2}\mathbf{I},
\end{equation}
where we used $\sigma_t^2 = 1 - \alpha_t^2$. Thus:
\begin{equation}
\mathbf{I} + \nabla_{\mathbf{x}} \mathbf{s}(\mathbf{x}, t) = -\frac{\alpha_t^2}{\sigma_t^2}\mathbf{I} + \frac{\alpha_t^2}{\sigma_t^4} \mathrm{Cov}(\mathbf{x}_0 \mid \mathbf{x}).
\end{equation}

The gradient of the velocity field simplifies to:
\begin{equation}
\label{eq:grad_v_exact}
\nabla_{\mathbf{x}} \mathbf{v}(\mathbf{x}, t) = f(t) \left[ -\frac{\alpha_t^2}{\sigma_t^2}\mathbf{I} + \frac{\alpha_t^2}{\sigma_t^4} \mathrm{Cov}(\mathbf{x}_0 \mid \mathbf{x}) \right].
\end{equation}

Taking the spectral norm and applying the triangle inequality:
\begin{equation}
\| \nabla_{\mathbf{x}} \mathbf{v}(\mathbf{x}, t) \| \le |f(t)| \left( \frac{\alpha_t^2}{\sigma_t^2} + \frac{\alpha_t^2}{\sigma_t^4} \|\mathrm{Cov}(\mathbf{x}_0 \mid \mathbf{x})\| \right).
\end{equation}

From \cref{lem:score_lipschitz_bounded}, we have $\|\mathrm{Cov}(\mathbf{x}_0 \mid \mathbf{x})\| \le R^2$. Therefore:
\begin{align}
L_x(t) &= |f(t)| \frac{\alpha_t^2}{\sigma_t^2} \left( 1 + \frac{R^2}{\sigma_t^2} \right) \\
&= \left| -\frac{\dot{\alpha}_t}{\alpha_t} \right| \frac{\alpha_t^2}{\sigma_t^2} \left( 1 + \frac{R^2}{\sigma_t^2} \right) \\
&= -\frac{\alpha_t}{\sigma_t^2} \dot{\alpha}_t \left( 1 + \frac{R^2}{\sigma_t^2} \right),
\end{align}
where we used $f(t) = -\frac{\dot{\alpha}_t}{\alpha_t}$ \cref{ass:alpha_regularity}.

Since $\dot{\alpha}_t \le 0$ by assumption, we have $L_x(t) \ge 0$, completing the proof.
\end{proof}

\begin{corollary}[Uniform Lipschitz Constant on Truncated Time Interval]
\label{cor:global_bound_Lv_uniform}
Under the conditions of \cref{thm:v_Lipschitz_improved} and \cref{ass:alpha_regularity}, for any truncation time $\delta t \in (0, 1)$, the velocity field $\mathbf{v}(\mathbf{x}, t)$ admits a uniform Lipschitz constant $L_{x, \delta t} := \sup_{t \in [\delta t, 1]} L_x(t)$ on the time interval $[\delta t, 1]$, bounded by:
\begin{equation}
\label{eq:Lv_global_bound_alpha}
L_{x, \delta t}(C_{\alpha, 1}, \alpha_{\delta t}, R) \le C_{\alpha, 1} \frac{\alpha_{\delta t}}{\sigma_{\delta t}^2} \left( 1 + \frac{R^2}{\sigma_{\delta t}^2} \right) \le C_{\alpha, 1} \frac{1}{\sigma_{\delta t}^2} \left( 1 + \frac{R^2}{\sigma_{\delta t}^2} \right),
\end{equation}
where $\sigma_{\delta t}^2 = 1 - \alpha_{\delta t}^2$. As $\delta t \to 0^+$, this constant diverges since $\sigma_{\delta t} \to 0$.
\end{corollary}

\begin{remark}[Handling Singularities at $t=0$]
The divergence of $L_{x,\delta t}$ as $\delta t \to 0$ reflects the well-known ill-posedness of diffusion models at the noise end ($t=T$ in reverse time). However, by fixing a truncation time $\delta t > 0$ arbitrarily small, we obtain a finite uniform bound on the interval $[\delta t, T]$. This truncation strategy is essential for the convergence analysis and is handled systematically through domain decomposition in \cref{def:domain_partition}.
\end{remark}

\begin{proof}
From \cref{thm:v_Lipschitz_improved}, we have:
\begin{equation}
L_x(t) = \frac{\alpha_t}{\sigma_t^2} \left( -\dot{\alpha}_t \right) \left( 1 + \frac{R^2}{\sigma_t^2} \right).
\end{equation}

By \cref{ass:alpha_regularity}, $|\dot{\alpha}_t| \le C_{\alpha, 1}$ and $\dot{\alpha}_t \le 0$, yielding:
\begin{equation}
\label{eq:Lv_bound_alpha_intermediate}
L_x(t) \le C_{\alpha, 1} \frac{\alpha_t}{\sigma_t^2} \left( 1 + \frac{R^2}{\sigma_t^2} \right).
\end{equation}

For $t \in [\delta t, 1]$:
\begin{itemize}
    \item $\alpha_t$ is monotonically decreasing;
    \item $\sigma_t^2 = 1 - \alpha_t^2$ is monotonically increasing, hence $\frac{1}{\sigma_t^2}$ is monotonically decreasing;
    \item The bracket term $\left( 1 + \frac{R^2}{\sigma_t^2} \right)$ is monotonically decreasing.
\end{itemize}

Therefore, the right-hand side of \cref{eq:Lv_bound_alpha_intermediate} is monotonically decreasing on $[\delta t, 1]$, attaining its maximum at $t = \delta t$. Substituting yields \cref{eq:Lv_global_bound_alpha}.
\end{proof}

\subsection{Linear Growth Bound for Velocity Field Norm}

\begin{proposition}[Linear Growth Bound for Velocity Field]
\label{prop:v_bound}
Under \cref{ass:p0_bounded} and \cref{ass:alpha_regularity}, for any $\delta_t \in (0, T]$, the velocity field $\mathbf{v}(\mathbf{x}, t)$ satisfies a linear growth condition on $\mathbb{R}^d \times [\delta_t, T]$. There exists a constant $M(\delta_t| C_{\alpha,1},\sigma_{\delta t},R) > 0$ such that for all $t \in [\delta_t, T]$ and $\mathbf{x} \in \mathbb{R}^d$:
\begin{equation}
\|\mathbf{v}(\mathbf{x}, t)\| \le M(t| C_{\alpha,1},\sigma_{t},R)(1 + \|\mathbf{x}\|)\le M(\delta_t| C_{\alpha,1},\sigma_{\delta t},R)(1 + \|\mathbf{x}\|),
\end{equation}
where:
\begin{align}
M(t| C_{\alpha,1},\sigma_{t},R) &= C_{\alpha,1}\frac{1}{\sigma_{t}^2} \left( 1 + R + \frac{R^2}{\sigma_{t}^2} \right), \\
M(\delta_t| C_{\alpha,1},\sigma_{\delta t},R) &= C_{\alpha,1}\frac{1}{\sigma_{\delta t}^2} \left( 1 + R + \frac{R^2}{\sigma_{\delta t}^2} \right).
\end{align}
\end{proposition}

\begin{proof}
Using the Lipschitz continuity from \cref{thm:v_Lipschitz_improved}, for any $t \in [\delta_t, T]$, we decompose the norm via the origin:
\begin{align}
\|\mathbf{v}(\mathbf{x}, t)\| &= \|\mathbf{v}(\mathbf{x}, t) - \mathbf{v}(\mathbf{0}, t) + \mathbf{v}(\mathbf{0}, t)\| \nonumber \\
&\le \|\mathbf{v}(\mathbf{x}, t) - \mathbf{v}(\mathbf{0}, t)\| + \|\mathbf{v}(\mathbf{0}, t)\| \nonumber \\
&\le L_x(t) \|\mathbf{x}\| + \|\mathbf{v}(\mathbf{0}, t)\|. \label{eq:v_linear_base}
\end{align}

We now bound $\|\mathbf{v}(\mathbf{0}, t)\|$. By definition, $\mathbf{v}(\mathbf{0}, t) = f(t)\mathbf{s}(\mathbf{0}, t)$, where $f(t) = -\frac{\dot{\alpha}_t}{\alpha_t} < 0$. Using the score function expression:
\begin{equation}
\mathbf{s}(\mathbf{x}, t) = \nabla_{\mathbf{x}} \log p_t(\mathbf{x}) = \int_{\mathbb{R}^d} p(\mathbf{x}_0 | \mathbf{x}) \frac{\alpha_t \mathbf{x}_0 - \mathbf{x}}{\sigma_t^2} d\mathbf{x}_0,
\end{equation}
at $\mathbf{x} = \mathbf{0}$:
\begin{equation}
\|\mathbf{s}(\mathbf{0}, t)\| = \left\|\int_{\mathbb{R}^d} p(\mathbf{x}_0 | \mathbf{0}) \frac{\alpha_t \mathbf{x}_0}{\sigma_t^2} d\mathbf{x}_0\right\| \le \frac{\alpha_t}{\sigma_t^2} \int_{\mathbb{R}^d} p(\mathbf{x}_0 | \mathbf{0}) \|\mathbf{x}_0\| d\mathbf{x}_0.
\end{equation}

By \cref{ass:p0_bounded}, $\text{supp}(p_0) \subseteq B(\mathbf{0}, R)$, hence $\|\mathbf{x}_0\| \le R$ for all $\mathbf{x}_0$ with $p(\mathbf{x}_0|\mathbf{0}) > 0$. Using $\alpha_t \le 1$:
\begin{equation}
\label{eq:v0_bound}
\|\mathbf{v}(\mathbf{0}, t)\| = |f(t)| \|\mathbf{s}(\mathbf{0}, t)\| \le \frac{|\dot{\alpha}_t|}{\alpha_t} \cdot \frac{\alpha_t R}{\sigma_t^2} = \frac{|\dot{\alpha}_t| R}{\sigma_t^2}\le C_{\alpha,1}\frac{R}{\sigma_t^2}.
\end{equation}

For $t \in [\delta_t, T]$, we have $\sigma_t^2 \ge \sigma_{\delta_t}^2$ by monotonicity. Substituting the uniform Lipschitz constant from \cref{cor:global_bound_Lv_uniform} and \cref{eq:v0_bound} into \cref{eq:v_linear_base}:
\begin{equation}
\|\mathbf{v}(\mathbf{x}, t)\| \le L_{x, \delta_t} \|\mathbf{x}\| + C_{\alpha,1}\frac{R}{\sigma_t^2}.
\end{equation}

Taking $M(\delta_t) = \max\left\{ L_{x, \delta_t}, C_{\alpha,1}\frac{R}{\sigma_t^2} \right\}\le L_{x, \delta_t}+ C_{\alpha,1}\frac{R}{\sigma_t^2}$:
\begin{equation}
\|\mathbf{v}(\mathbf{x}, t)\| \le M(\delta_t)(1 + \|\mathbf{x}\|).
\end{equation}

From \cref{cor:global_bound_Lv_uniform}, $L_{x, \delta_t}\le C_{\alpha, 1} \frac{1}{\sigma_{\delta t}^2} \left( 1 + \frac{R^2}{\sigma_{\delta t}^2} \right)$, yielding:
\begin{equation}
M(\delta_t)\le C_{\alpha,1}\frac{1}{\sigma_{\delta t}^2} \left( 1 + R + \frac{R^2}{\sigma_{\delta t}^2} \right).
\end{equation}
\end{proof}

\subsection{Bounds on Temporal Derivative of Velocity Field}

We now analyze the upper bound on the time derivative $\frac{\partial \mathbf{v}}{\partial t}$, which is essential for discretization error analysis.

\subsubsection{Expression for Time Derivative}

First, we derive the explicit expression. From the velocity field definition:
\begin{equation}
\mathbf{v}(\mathbf{x}, t) = f(t) \left[ \mathbf{x} + \mathbf{s}(\mathbf{x}, t) \right],
\end{equation}
applying the product rule:
\begin{equation}
\frac{\partial \mathbf{v}(\mathbf{x}, t)}{\partial t} = \dot{f}(t) \left[ \mathbf{x} + \mathbf{s}(\mathbf{x}, t) \right] + f(t) \frac{\partial \mathbf{s}(\mathbf{x}, t)}{\partial t}. \label{eq:dv_dt_f}
\end{equation}

Using $f(t) = -\frac{\dot{\alpha}_t}{\alpha_t}$ from the SDE parameterization, compute:
\begin{align}
\dot{f}(t) &= -\frac{d}{dt} \left( \frac{\dot{\alpha}_t}{\alpha_t} \right) \\
&= -\left[ \frac{\ddot{\alpha}_t}{\alpha_t} - \frac{\dot{\alpha}_t^2}{\alpha_t^2} \right] \\
&= \frac{\dot{\alpha}_t^2}{\alpha_t^2} - \frac{\ddot{\alpha}_t}{\alpha_t}.
\end{align}

Substituting into \cref{eq:dv_dt_f} and using the identity $\mathbf{x} + \mathbf{s}(\mathbf{x}, t) = \frac{\alpha_t}{\sigma_t^2}(\hat{\mathbf{x}}_0 - \alpha_t \mathbf{x})$ yields:
\begin{equation}
\label{eq:dv_dt}
\frac{\partial \mathbf{v}(\mathbf{x}, t)}{\partial t} = \left[ \frac{\alpha_t \ddot{\alpha}_t - \dot{\alpha}_t^2}{\alpha_t^2} \right] \left( \mathbf{x} + \mathbf{s}(\mathbf{x}, t) \right) + f(t) \frac{\partial \mathbf{s}(\mathbf{x}, t)}{\partial t},
\end{equation}
where we use the notation $\dot{\alpha}_t := \frac{d\alpha_t}{dt}$ and $\ddot{\alpha}_t := \frac{d^2\alpha_t}{dt^2}$.

\subsubsection{Upper Bound Analysis}

In addition to spatial and growth properties, the time evolution of the velocity field plays a crucial role in discretization error analysis. The following proposition establishes linear growth bounds on the time derivative $\partial_t \mathbf{v}(\mathbf{x}, t)$, which controls how the velocity field changes along trajectories as $t$ progresses.

\begin{proposition}[Linear Growth Bound for Time Derivative of Velocity Field]
\label{prop:time_deriv_bound}
Under \cref{ass:p0_bounded} and \cref{ass:alpha_regularity} for any $t \in (0, T]$ and $\mathbf{x} \in \mathbb{R}^d$, there exist positive constants $C_1(t), C_0(t)$ such that:
\begin{equation}
\left\| \frac{\partial \mathbf{v}(\mathbf{x}, t)}{\partial t} \right\| \le C_1(t) \|\mathbf{x}\| + C_0(t),
\end{equation}
where $C_1(t)$ and $C_0(t)$ depend only on $\alpha_t$ and its derivatives $\dot{\alpha}_t, \ddot{\alpha}_t$, and the initial support radius $R$:
\begin{align}
C_1(t) &= \left| \frac{\ddot{\alpha}_t \sigma_t^2 + 2\alpha_t \dot{\alpha}_t^2}{\sigma_t^4} \right| \alpha_t
+ \left| \frac{\dot{\alpha}_t}{\alpha_t} \right| \left( \frac{(1+\alpha_t^2) R^2}{\sigma_t^4} |\dot{\alpha}_t| + |\dot{\alpha}_t| \right), \\
C_0(t) &= \left| \frac{\ddot{\alpha}_t \sigma_t^2 + 2\alpha_t \dot{\alpha}_t^2}{\sigma_t^4} \right| R
+ \left| \frac{\dot{\alpha}_t}{\alpha_t} \right| \cdot \frac{2|\dot{\alpha}_t|\alpha_t R^3}{\sigma_t^4}.
\end{align}
Note that as $t \to T$ (i.e., $\alpha_t \to 0, \sigma_t \to 1$), these coefficients remain bounded without singularities.
\end{proposition}

\begin{proof}
\textbf{Step 1: Posterior mean representation}

From \cref{eq:score_as_posterior_mean_redux}:
\begin{equation}
\mathbf{s}(\mathbf{x}, t) = -\frac{1}{\sigma_t^2}\bigl(\mathbf{x} - \alpha_t \hat{\mathbf{x}}_0(\mathbf{x},t)\bigr),
\end{equation}
where $\hat{\mathbf{x}}_0(\mathbf{x},t) := \mathbb{E}[\mathbf{x}_0 | \mathbf{x}_t = \mathbf{x}]$. 

Using the identity $\mathbf{x} + \mathbf{s}(\mathbf{x}, t) = \frac{\alpha_t}{\sigma_t^2}(\hat{\mathbf{x}}_0 - \alpha_t \mathbf{x})$, the velocity field becomes:
\begin{equation}
\label{eq:v_posterior_form}
\mathbf{v}(\mathbf{x}, t) = f(t) \cdot \frac{\alpha_t}{\sigma_t^2}\left(\hat{\mathbf{x}}_0(\mathbf{x}, t) - \alpha_t \mathbf{x}\right) = -\frac{\dot{\alpha}_t}{\sigma_t^2}\left(\hat{\mathbf{x}}_0(\mathbf{x}, t) - \alpha_t \mathbf{x}\right).
\end{equation}

\textbf{Step 2: Computing the time derivative}

Taking the time derivative of \cref{eq:v_posterior_form} using the product rule:
\begin{equation}
\frac{\partial \mathbf{v}}{\partial t} = \frac{\partial}{\partial t}\left(-\frac{\dot{\alpha}_t}{\sigma_t^2}\right) \left(\hat{\mathbf{x}}_0 - \alpha_t \mathbf{x}\right) - \frac{\dot{\alpha}_t}{\sigma_t^2} \left(\frac{\partial \hat{\mathbf{x}}_0}{\partial t} - \dot{\alpha}_t \mathbf{x}\right).
\end{equation}

Using the quotient rule and $\dot{\sigma}_t^2 = -2\alpha_t \dot{\alpha}_t$:
\begin{equation}
\frac{\partial}{\partial t}\left(-\frac{\dot{\alpha}_t}{\sigma_t^2}\right) = -\frac{\ddot{\alpha}_t \sigma_t^2 + 2\alpha_t \dot{\alpha}_t^2}{\sigma_t^4}.
\end{equation}

\textbf{Step 3: Establishing norm bounds}

Taking norms and applying the triangle inequality:
\begin{equation}
\left\| \frac{\partial \mathbf{v}}{\partial t} \right\| 
\le \left| \frac{\ddot{\alpha}_t \sigma_t^2 + 2\alpha_t \dot{\alpha}_t^2}{\sigma_t^4} \right| \|\hat{\mathbf{x}}_0 - \alpha_t \mathbf{x}\| + \frac{|\dot{\alpha}_t|}{\sigma_t^2} \left\| \frac{\partial \hat{\mathbf{x}}_0}{\partial t} - \dot{\alpha}_t \mathbf{x} \right\|.
\end{equation}

For the first term, using $\|\hat{\mathbf{x}}_0\| \le R$ from \cref{ass:p0_bounded}:
\begin{equation}
\|\hat{\mathbf{x}}_0 - \alpha_t \mathbf{x}\| \le R + \alpha_t \|\mathbf{x}\|.
\end{equation}

For the second term, applying \cref{lem:mu_time_deriv_bound} with:
\begin{align}
C_1^{\mu}(t) &= \frac{|\dot{\alpha}_t|(1+\alpha_t^2)}{\sigma_t^4} R^2, \\
C_0^{\mu}(t) &= \frac{2|\dot{\alpha}_t|\alpha_t}{\sigma_t^4} R^3,
\end{align}
we obtain:
\begin{equation}
\left\| \frac{\partial \hat{\mathbf{x}}_0}{\partial t} - \dot{\alpha}_t \mathbf{x} \right\| 
\le \left( C_1^{\mu}(t) + |\dot{\alpha}_t| \right) \|\mathbf{x}\| + C_0^{\mu}(t).
\end{equation}

\textbf{Step 4: Combining terms}

Collecting all contributions yields the stated coefficients $C_1(t)$ and $C_0(t)$. As $t \to T$ with $\alpha_t \to 0$ and $\sigma_t \to 1$, all terms remain finite without $1/\alpha_t$ singularities.
\end{proof}

\begin{corollary}[Uniform Bound on Truncated Time Interval]
\label{cor:uniform_bound_delta}
Under \cref{prop:time_deriv_bound,ass:alpha_regularity}, for any truncation time $\delta \in (0, T)$, the uniform bounds $C_{1,\delta} := \sup_{t \in [\delta, T]} C_1(t)$ and $C_{0,\delta} := \sup_{t \in [\delta, T]} C_0(t)$ satisfy:
\begin{align}
C_{1,\delta}(C_{\alpha,1}, C_{\alpha,2}, \alpha_{\delta}, R) 
&= \frac{C_{\alpha,2} \alpha_{\delta}}{\sigma_{\delta}^2} + \frac{2\alpha_{\delta}^2 C_{\alpha,1}^2}{\sigma_{\delta}^4} + \frac{C_{\alpha,1}^2}{\sigma_{\delta}^2} \left( \frac{(1+\alpha_{\delta}^2) R^2}{\sigma_{\delta}^4} + 1 \right), \label{eq:C1_delta_bound} \\
C_{0,\delta}(C_{\alpha,1}, C_{\alpha,2}, \alpha_{\delta}, R) 
&= \frac{C_{\alpha,2} R}{\sigma_{\delta}^2} + \frac{2\alpha_{\delta} C_{\alpha,1}^2 R}{\sigma_{\delta}^4} + \frac{2\alpha_{\delta} C_{\alpha,1}^2 R^3}{\sigma_{\delta}^6}, \label{eq:C0_delta_bound}
\end{align}
where $\sigma_{\delta}^2 = 1 - \alpha_{\delta}^2$ and $C_{\alpha,1}, C_{\alpha,2}$ are the first- and second-order derivative bounds from \cref{ass:alpha_regularity}.
\end{corollary}

\begin{proof}
By \cref{ass:alpha_regularity}, $\alpha_t$ is monotonically decreasing on $[0,T]$. For $t \in [\delta, T]$, we have $0 \le \alpha_t \le \alpha_{\delta}$ and $\sigma_{\delta}^2 \le \sigma_t^2 \le 1$.

\textbf{Step 1: Bounding $C_1(t)$}

Using $|\dot{\alpha}_t| \le C_{\alpha,1}$ and $|\ddot{\alpha}_t| \le C_{\alpha,2}$:
\begin{equation}
\left| \frac{\ddot{\alpha}_t \sigma_t^2 + 2\alpha_t \dot{\alpha}_t^2}{\sigma_t^4} \right| 
\le \frac{C_{\alpha,2}}{\sigma_t^2} + \frac{2\alpha_{\delta} C_{\alpha,1}^2}{\sigma_{\delta}^4} 
\le \frac{C_{\alpha,2}}{\sigma_{\delta}^2} + \frac{2\alpha_{\delta} C_{\alpha,1}^2}{\sigma_{\delta}^4}.
\end{equation}

For the second term in $C_1(t)$:
\begin{equation}
\frac{|\dot{\alpha}_t|}{\sigma_t^2} \left( \frac{(1+\alpha_t^2) R^2}{\sigma_t^4} |\dot{\alpha}_t| + |\dot{\alpha}_t| \right) 
\le \frac{C_{\alpha,1}^2}{\sigma_{\delta}^2} \left( \frac{(1+\alpha_{\delta}^2) R^2}{\sigma_{\delta}^4} + 1 \right).
\end{equation}

Combining yields \cref{eq:C1_delta_bound}.

\textbf{Step 2: Bounding $C_0(t)$}

Similarly:
\begin{equation}
C_0(t) \le \left( \frac{C_{\alpha,2}}{\sigma_{\delta}^2} + \frac{2\alpha_{\delta} C_{\alpha,1}^2}{\sigma_{\delta}^4} \right) R + \frac{2C_{\alpha,1}^2 \alpha_{\delta} R^3}{\sigma_{\delta}^6},
\end{equation}
establishing \cref{eq:C0_delta_bound}.
\end{proof}

\subsection{Error Analysis for Neural Network Velocity Fields}

Having established the regularity properties of the true velocity field in the preceding subsections, we now quantify the approximation error arising from replacing the exact velocity field with neural network approximations. This analysis decomposes the error into two independent components: classification error (from approximating posterior probabilities $a^i(\mathbf{x}, t)$) and regression error (from approximating denoising estimates $\mathbf{y}^i(\mathbf{x}, t)$).

\subsubsection{Error Bound for $\overline{\mathbf{v}_\theta}(\mathbf{x}, t)$}

\begin{lemma}[Linear Growth Bound for Neural Network Velocity Field Error]
\label{lemma:v_theta_error_explicit}
Consider a data distribution partitioned into $L$ categories with initial distribution $p_0$ supported on a ball of radius $R$. Under \cref{assump:mixture_data}, \cref{ass:p0_bounded},\cref{ass:alpha_regularity},\cref{ass:nn_approx},combined with \cref{thm:v_Lipschitz_improved,prop:v_bound}, the mixture velocity field $\overline{\mathbf{v}_{\theta,\phi}}(\mathbf{x}, t)$ derived from classification network $a^i_\phi$ (parameter $\phi$) and regression network $y^i_\theta$ (parameter $\theta$) satisfies:
\begin{equation}
\|\overline{\mathbf{v}_{\theta,\phi}}(\mathbf{x}, t) - \mathbf{v}(\mathbf{x}, t)\| \leq A(t) \varepsilon_a \|\mathbf{x}\| + B(t), \quad \forall \mathbf{x} \in \mathbb{R}^d, t \in (0, T],
\end{equation}
where the coefficients are:
\begin{align}
A(t) &= L \cdot M(t) = L \cdot C_{\alpha,1}\frac{1}{\sigma_{t}^2} \left( 1 + R + \frac{R^2}{\sigma_{t}^2} \right), \label{eq:def_At} \\
B(t) &= A(t) \varepsilon_a + \frac{|\dot{\alpha}_t|}{\sigma_t^2} \varepsilon_y, \label{eq:def_Bt}
\end{align}
with $L$ being the number of categories, $\varepsilon_a, \varepsilon_y$ the $L^\infty$ approximation errors for classification and regression networks, $C_{\alpha,1}$ from \cref{ass:alpha_regularity}, and $\sigma_t^2 = 1 - \alpha_t^2$.
\end{lemma}

\begin{proof}
Recall the mixture velocity field definitions:
\begin{equation}
\overline{\mathbf{v}_{\theta,\phi}}(\mathbf{x}, t) = \sum_{i=1}^L a^i_\phi(\mathbf{x}, t) \mathbf{v}^i_\theta(\mathbf{x}, t), \quad \mathbf{v}(\mathbf{x}, t) = \sum_{i=1}^L a^i(\mathbf{x}, t) \mathbf{v}^i(\mathbf{x}, t).
\end{equation}

Adding and subtracting the intermediate term $a^i_\phi(\mathbf{x}, t) \mathbf{v}^i(\mathbf{x}, t)$ and applying the triangle inequality:
\begin{align}
\|\overline{\mathbf{v}_{\theta,\phi}} - \mathbf{v}\| &= \left\| \sum_{i=1}^L \left( a^i_\phi \mathbf{v}^i_\theta - a^i \mathbf{v}^i \right) \right\| \nonumber \\
&\leq \underbrace{\sum_{i=1}^L \| a^i_\phi - a^i \| \cdot \| \mathbf{v}^i \|}_{I_1} + \underbrace{\left\| \sum_{i=1}^L a^i_\phi (\mathbf{v}^i_\theta - \mathbf{v}^i) \right\|}_{I_2}. \label{eq:refined_decomp}
\end{align}

\textbf{Step 1: Bounding $I_1$ (classification error)}

From \cref{prop:v_bound}, $\|\mathbf{v}^{i}(\mathbf{x}, t)\| \le M(t)(1 + \|\mathbf{x}\|)$. Using $\|a^i_\phi - a^i\|_\infty < \varepsilon_a$:
\begin{equation}
\label{eq:bound_I1}
I_1 < \sum_{i=1}^L \varepsilon_a M(t)(1 + \|\mathbf{x}\|) = L \varepsilon_a M(t)(1 + \|\mathbf{x}\|).
\end{equation}

\textbf{Step 2: Bounding $I_2$ (regression error)}

From the velocity field definitions:
\begin{equation}
\mathbf{v}^i - \mathbf{v}^i_\theta = f(t) \frac{\alpha_t}{\sigma_t^2} (\mathbf{y}^i - \mathbf{y}^i_\theta).
\end{equation}

Using $\|\mathbf{y}^i_\theta - \mathbf{y}^i\|_\infty < \varepsilon_y$ and $|f(t)\alpha_t| = |\dot{\alpha}_t|$:
\begin{equation}
\|\mathbf{v}^i_\theta - \mathbf{v}^i\| < \frac{|\dot{\alpha}_t|}{\sigma_t^2} \varepsilon_y.
\end{equation}

Since $\sum_{i=1}^L a^i_\phi(\mathbf{x}, t) = 1$:
\begin{equation}
\label{eq:bound_I2}
I_2 < \frac{|\dot{\alpha}_t|}{\sigma_t^2} \varepsilon_y.
\end{equation}

\textbf{Step 3: Combining terms}

Substituting \cref{eq:bound_I1,eq:bound_I2} into \cref{eq:refined_decomp}:
\begin{equation}
\|\overline{\mathbf{v}_{\theta,\phi}} - \mathbf{v}\| < L M(t) \varepsilon_a (1 + \|\mathbf{x}\|) + \frac{|\dot{\alpha}_t|}{\sigma_t^2} \varepsilon_y = A(t) \varepsilon_a \|\mathbf{x}\| + B(t),
\end{equation}
matching the claimed definitions.
\end{proof}

\begin{remark}[Notation Convention]
In the sequel, we abbreviate $\overline{\mathbf{v}_{\theta,\phi}}(\mathbf{x}, t)$ as $\overline{\mathbf{v}}_\theta(\mathbf{x}, t)$ when the explicit dependence on the classification network parameter $\phi$ is clear from context. This shorthand encompasses both the classification error $\varepsilon_a$ and regression error $\varepsilon_y$, as implicitly accounted for in the error bounds.
\end{remark}

\begin{corollary}Uniform Error Bound on Time Interval $[\delta_t, T]$
\label{cor:global_error_bound}
Under the conditions of \cref{lemma:v_theta_error_explicit,ass:alpha_regularity}, for any truncation time $\delta_t \in (0, T]$, the neural network velocity field satisfies a uniform linear growth error bound on $\mathbb{R}^d \times [\delta_t, T]$:
\begin{equation}
\|\overline{\mathbf{v}_{\theta,\phi}}(\mathbf{x}, t) - \mathbf{v}(\mathbf{x}, t)\| \leq A_{\max}(\delta_t) \varepsilon_a\|\mathbf{x}\| + B_{\max}(\delta_t),
\end{equation}
where:
\begin{align}
A_{\max}(\delta_t) &= \frac{L C_{\alpha,1}}{\sigma_{\delta_t}^2} \left( 1 + R + \frac{R^2}{\sigma_{\delta_t}^2} \right), \label{eq:A_max_def} \\
B_{\max}(\delta_t) &= A_{\max}(\delta_t) \varepsilon_a + \frac{C_{\alpha,1}}{\sigma_{\delta_t}^2} \varepsilon_y, \label{eq:B_max_def}
\end{align}
with $\sigma_{\delta_t}^2 = 1 - \alpha_{\delta_t}^2$.
\end{corollary}

\begin{proof}
From \cref{lemma:v_theta_error_explicit}, for any fixed $t \in (0, T]$:
\begin{equation}
\|\overline{\mathbf{v}_{\theta,\phi}}(\mathbf{x}, t) - \mathbf{v}(\mathbf{x}, t)\| \leq A(t) \varepsilon_a \|\mathbf{x}\| + B(t).
\end{equation}

By \cref{ass:alpha_regularity}, $\alpha_t$ is monotonically decreasing on $[0, T]$, hence $\sigma_t^2 = 1 - \alpha_t^2$ is monotonically increasing. For $t \in [\delta_t, T]$:
\begin{equation}
\label{eq:sigma_inv_bound}
\frac{1}{\sigma_t^2} \leq \frac{1}{\sigma_{\delta_t}^2}.
\end{equation}

Since $A(t)$ is a positive polynomial in $1/\sigma_t^2$, it attains its maximum at $t = \delta_t$:
\begin{equation}
A(t) \leq L C_{\alpha,1} \frac{1}{\sigma_{\delta_t}^2} \left( 1 + R + \frac{R^2}{\sigma_{\delta_t}^2} \right) := A_{\max}(\delta_t).
\end{equation}

Similarly, using $|\alpha_t'| \leq C_{\alpha,1}$ and \cref{eq:sigma_inv_bound}:
\begin{equation}
B(t) \leq A_{\max}(\delta_t) \varepsilon_a + \frac{C_{\alpha,1}}{\sigma_{\delta_t}^2} \varepsilon_y := B_{\max}(\delta_t).
\end{equation}
\end{proof}

\begin{corollary}[Linear Growth Bound for Neural Network Velocity Field]
\label{cor:v_theta_linear_growth}
Under the conditions of \cref{cor:global_error_bound}, for any truncation time $\delta_t \in (0, T]$, both the component velocity fields $\mathbf{v}^i_\theta(\mathbf{x}, t)$ and the mixture velocity field $\overline{\mathbf{v}_\theta}(\mathbf{x}, t)$ satisfy a linear growth condition:
\begin{equation}
\|\mathbf{v}^i_\theta(\mathbf{x}, t)\| \leq M'_{\delta_t} (1 + \|\mathbf{x}\|), \quad \|\overline{\mathbf{v}_\theta}(\mathbf{x}, t)\| \leq M'_{\delta_t} (1 + \|\mathbf{x}\|),
\end{equation}
where $M'_{\delta_t}(\delta_t, \alpha, R, \varepsilon_y) = C_{\text{phy}} + C_{\text{err}}$ consists of:

\textbf{Physical component} (from \cref{prop:v_bound}):
\begin{equation}
C_{\text{phy}} = M(\delta_t) = C_{\alpha,1}\frac{1}{\sigma_{\delta_t}^2} \left( 1 + R + \frac{R^2}{\sigma_{\delta_t}^2} \right).
\end{equation}

\textbf{Network error component}:
\begin{equation}
C_{\text{err}} = \varepsilon_y \cdot \frac{C_{\alpha,1}}{\sigma_{\delta_t}^2}.
\end{equation}

Specifically, $M'_{\delta_t} = C_{\text{phy}} + \frac{C_{\alpha,1}}{\sigma_{\delta_t}^2}\varepsilon_y$.
\end{corollary}

\begin{proof}
For each component $i \in \{1, \dots, L\}$, by the triangle inequality:
\begin{equation}
\|\mathbf{v}^i_\theta(\mathbf{x}, t)\| \leq \|\mathbf{v}^i(\mathbf{x}, t)\| + \|\mathbf{v}^i_\theta(\mathbf{x}, t) - \mathbf{v}^i(\mathbf{x}, t)\|.
\end{equation}
From \cref{prop:v_bound}, the true component velocity field satisfies $\|\mathbf{v}^i(\mathbf{x}, t)\| \leq C_{\text{phy}}(1 + \|\mathbf{x}\|)$. From the velocity field definition and \cref{ass:nn_approx}, similar to the derivation in \cref{lemma:v_theta_error_explicit}:
\begin{equation}
\|\mathbf{v}^i_\theta(\mathbf{x}, t) - \mathbf{v}^i(\mathbf{x}, t)\| = \left\| \frac{\dot{\alpha}_t}{\sigma_t^2} (\mathbf{y}^i_\theta - \mathbf{y}^i) \right\| \leq \frac{C_{\alpha, 1}}{\sigma_{\delta_t}^2} \varepsilon_y.
\end{equation}
Combining these and using $1 + \|\mathbf{x}\| \ge 1$ yields $\|\mathbf{v}^i_\theta(\mathbf{x}, t)\| \leq (C_{\text{phy}} + \frac{C_{\alpha, 1}}{\sigma_{\delta_t}^2} \varepsilon_y)(1 + \|\mathbf{x}\|) = M'_{\delta_t} (1 + \|\mathbf{x}\|)$.

For the mixture velocity field, since $\overline{\mathbf{v}_\theta}(\mathbf{x}, t) = \sum_{i=1}^L a^i_\phi(\mathbf{x}, t) \mathbf{v}^i_\theta(\mathbf{x}, t)$ is a convex combination of the component fields (where $\sum a^i_\phi = 1$ and $a^i_\phi \ge 0$), we have:
\begin{equation}
\|\overline{\mathbf{v}_\theta}(\mathbf{x}, t)\| \leq \sum_{i=1}^L a^i_\phi(\mathbf{x}, t) \|\mathbf{v}^i_\theta(\mathbf{x}, t)\| \leq \max_i \|\mathbf{v}^i_\theta(\mathbf{x}, t)\| \leq M'_{\delta_t} (1 + \|\mathbf{x}\|).
\end{equation}
Thus, the mixture field inherits the tighter linear growth bound from its components.
\end{proof}

\begin{corollary}[Linear Growth Bound for Neural Network Velocity Field]
\label{cor:v_theta_linear_growth}
Under the conditions of \cref{cor:global_error_bound}, for any truncation time $\delta_t \in (0, T]$, the neural network approximated velocity field $\overline{\mathbf{v}_\theta}(\mathbf{x}, t)$ satisfies:
\begin{equation}
\|\overline{\mathbf{v}_\theta}(\mathbf{x}, t)\| \leq M'_{\delta_t} (1 + \|\mathbf{x}\|),
\end{equation}
where $M'_{\delta_t}(\delta_t, \alpha, R, \varepsilon_a, \varepsilon_y) = C_{\text{phy}} + C_{\text{err}}$ consists of:

\textbf{Physical component} (from \cref{prop:v_bound}):
\begin{equation}
C_{\text{phy}} = M(\delta_t) = C_{\alpha,1}\frac{1}{\sigma_{\delta_t}^2} \left( 1 + R + \frac{R^2}{\sigma_{\delta_t}^2} \right).
\end{equation}

\textbf{Network error component}:
\begin{equation}
C_{\text{err}} = \underbrace{L \varepsilon_a C_{\text{phy}}}_{\text{classification}} + \underbrace{\varepsilon_y \cdot \frac{C_{\alpha,1}}{\sigma_{\delta_t}^2}}_{\text{regression}}.
\end{equation}

Specifically, $M'_{\delta_t} = C_{\text{phy}}(1 + L \varepsilon_a) + \frac{C_{\alpha,1}}{\sigma_{\delta_t}^2}\varepsilon_y$.
\end{corollary}

\begin{proof}
By the triangle inequality:
\begin{equation}
\|\overline{\mathbf{v}_\theta}(\mathbf{x}, t)\| \leq \|\mathbf{v}(\mathbf{x}, t)\| + \|\overline{\mathbf{v}_\theta}(\mathbf{x}, t) - \mathbf{v}(\mathbf{x}, t)\|.
\end{equation}

From \cref{prop:v_bound}, $\|\mathbf{v}(\mathbf{x}, t)\| \leq C_{\text{phy}}(1 + \|\mathbf{x}\|)$.

From \cref{cor:global_error_bound}:
\begin{equation}
\|\overline{\mathbf{v}_\theta} - \mathbf{v}\| \leq L C_{\text{phy}} \varepsilon_a (1 + \|\mathbf{x}\|) + \frac{C_{\alpha,1}}{\sigma_{\delta_t}^2} \varepsilon_y.
\end{equation}

Combining and using $1 + \|\mathbf{x}\| \ge 1$:
\begin{equation}
\|\overline{\mathbf{v}_\theta}\| \leq \left[ C_{\text{phy}}(1 + L \varepsilon_a) + \frac{C_{\alpha,1}}{\sigma_{\delta_t}^2}\varepsilon_y \right] (1 + \|\mathbf{x}\|) = M'_{\delta_t}(1 + \|\mathbf{x}\|).
\end{equation}
\end{proof}

\section{Convergence Analysis of the Sampling Algorithm}
\label{sec:convergence_analysis}

This section establishes the convergence theory for the sampling algorithm described in the main paper. We analyze the convergence of the discrete stochastic sampling algorithm to the true solution of the probability flow ODE, taking into account both discretization errors and neural network approximation errors. 

Our analysis decomposes the total error into three main components: (i) local truncation error from time discretization, (ii) stochastic martingale difference noise from the sampling process, and (iii) neural network approximation errors in both velocity field components and sampling probabilities. This unified approach allows us to provide explicit convergence rates that depend on both the number of discretization steps $N$ and the neural network approximation accuracy $\varepsilon$.

\subsection{Algorithm Setup and Notation}
\label{subsec:algorithm_setup}

Recall the reverse time variable $\tau = T - t$. To avoid singularities near the initial forward diffusion time ($t=0$), we restrict our analysis to the time interval $t \in [\delta, T]$, corresponding to the reverse time interval $\tau \in [0, T_\delta]$, where $T_\delta := T - \delta$.

Within this time interval, although $\beta(t)$ varies with time, the regularity results established in \cref{prop:v_bound} ensure that the velocity field $\mathbf{v}(\mathbf{x}, t)$ satisfies uniform linear growth conditions, with its spatial Lipschitz constant and temporal derivative growth coefficients uniformly bounded by constants depending only on $\delta$. This provides the necessary boundedness guarantees for subsequent consistency analysis.

We partition the time interval $[0, T_\delta]$ into $N$ equally spaced steps with step size $\Delta \tau = T_\delta / N$. Discrete time points are defined as $\tau_k = k \Delta \tau$ for $k = 0, 1, \dots, N$.

\paragraph{Process Definitions}

We define three stochastic processes:

\begin{enumerate}
    \item \textbf{Reference ODE Process:} The continuous trajectory $\{\mathbf{X}(\tau)\}_{\tau \in [0, T_\delta]}$ satisfying the probability flow equation:
    \begin{equation}
        \frac{d\mathbf{X}(\tau)}{d\tau} = \tilde{\mathbf{v}}(\mathbf{X}(\tau), \tau), \quad \mathbf{X}(0) \sim p_T(\mathbf{x}),
        \label{eq:reference_ode}
    \end{equation}
    where $\tilde{\mathbf{v}}(\mathbf{x}, \tau) := -\mathbf{v}(\mathbf{x}, T-\tau)$ denotes the reverse-time velocity field. Here, $\mathbf{v}(\mathbf{x}, t)$ is the forward-time total velocity field defined in previous sections. 
    
    To simplify notation, we work with $\tilde{\mathbf{v}}$ throughout this section when discussing the reverse process. By the regularity results established earlier, $\tilde{\mathbf{v}}$ inherits uniform Lipschitz continuity and linear growth properties on $\mathbb{R}^d \times [0, T_\delta]$.
    
    \item \textbf{Idealized Discrete Sampling Process:} Let $\mathbf{x}_k$ denote the numerical solution at step $k$. This process samples the velocity field component at each step according to the true component posterior probabilities. The iterative scheme is:
    \begin{align}
        \mathbf{x}_{0} &= \mathbf{X}(0), \label{eq:ideal_init} \\
        \mathbf{x}_{k+1} &= \mathbf{x}_k + \tilde{\mathbf{v}}^{l_k}(\mathbf{x}_k, \tau_k) \Delta \tau, \label{eq:ideal_update}
    \end{align}
    where $l_k$ is a random variable following a categorical distribution with probability mass function given by the true mixture weights:
    \begin{equation}
        \mathbb{P}(l_k = i \mid \mathbf{x}_k) = a^i(\mathbf{x}_k, T-\tau_k), \quad i \in \{1, \dots, L\}.
        \label{eq:ideal_sampling_prob}
    \end{equation}
    
    Under this idealized setting, the conditional expectation of the discrete update exactly recovers the total velocity field:
    \begin{equation}
        \mathbb{E}_{l_k}[\tilde{\mathbf{v}}^{l_k}(\mathbf{x}_k, \tau_k) \mid \mathbf{x}_k] = \sum_{i=1}^L a^i(\mathbf{x}_k, T-\tau_k) \tilde{\mathbf{v}}^i(\mathbf{x}_k, \tau_k) = \tilde{\mathbf{v}}(\mathbf{x}_k, \tau_k).
        \label{eq:ideal_drift_consistency}
    \end{equation}
    
    \item \textbf{Practical Discrete Sampling Process with Neural Network Approximation:} 
    Let $\hat{\mathbf{x}}_k$ denote the numerical solution generated by the practical algorithm. In actual computation, we rely on neural network-parameterized velocity fields $\tilde{\mathbf{v}}_\theta^i$ and the resulting mixture weights $\hat{a}_\theta^i$. The iterative scheme becomes:
    \begin{align}
        \hat{\mathbf{x}}_{0} &= \mathbf{X}(0), \label{eq:practical_init} \\
        \hat{\mathbf{x}}_{k+1} &= \hat{\mathbf{x}}_k + \tilde{\mathbf{v}}_\theta^{\hat{l}_k}(\hat{\mathbf{x}}_k, \tau_k) \Delta \tau, \label{eq:practical_update}
    \end{align}
    where $\hat{l}_k$ is sampled according to neural network-estimated probabilities:
    \begin{equation}
        \mathbb{P}(\hat{l}_k = i \mid \hat{\mathbf{x}}_k) = \hat{a}_\theta^i(\hat{\mathbf{x}}_k, T-\tau_k).
        \label{eq:practical_sampling_prob}
    \end{equation}
    
    Here, $\tilde{\mathbf{v}}_\theta^i$ represents the neural network approximation to the $i$-th component velocity field, and $\hat{a}_\theta^i$ denotes the approximate posterior probability computed from neural network outputs.
    
    This process introduces two primary sources of approximation error:
    \begin{itemize}
        \item \textbf{Drift Approximation Error:} $\|\tilde{\mathbf{v}}_\theta^i - \tilde{\mathbf{v}}^i\|$, measuring the fidelity of neural network approximation to each component velocity field;
        \item \textbf{Sampling Probability Error:} $|\hat{a}_\theta^i - a^i|$, quantifying the deviation in mixture weights induced by the neural network, which causes the sampling distribution of $\hat{l}_k$ to differ from that of $l_k$.
    \end{itemize}
\end{enumerate}

\paragraph{Discretization Assumption}

Throughout our numerical analysis, we impose the following constraint on the time step size:

\begin{assumption}[Time Step Bound]
\label{ass:timestep_bound}
The discretization time step $\Delta \tau$ satisfies $\Delta \tau \le \Delta \tau_0$, where $\Delta \tau_0 > 0$ is a fixed constant predetermined before the numerical solution process.
\end{assumption}

\begin{remark}
\cref{ass:timestep_bound} is standard in the analysis of numerical methods for ODEs and ensures that higher-order terms in the Taylor expansion remain controlled. In practice, $\Delta \tau_0$ can be chosen based on the Lipschitz and growth constants established in \cref{sec:velocity_properties}.
\end{remark}

\subsection{Uniform Moment Bounds for Numerical Solutions}
\label{subsec:uniform_moment_bounds}

In this subsection, we establish uniform bounds on the second moments of both the idealized discrete sampling process and the practical neural network-based process. These bounds are essential for controlling the accumulation of discretization errors over multiple time steps.

\begin{lemma}[Radial Dissipativity of Reverse Dynamics]
\label{lem:radial_dissipativity}
Suppose the initial distribution satisfies \cref{ass:p0_bounded} (bounded support with radius $R$), and let $\alpha_t, \beta(t), \sigma_t$ be defined as in previous sections. For the reference reverse process $\mathbf{X}(\tau)$, at any time $\tau \in [0, T]$ (corresponding to physical time $t = T-\tau$), if the current state satisfies $\|\mathbf{X}(\tau)\| > R / \alpha_t$, then the Euclidean norm of the state strictly decreases with $\tau$.

Specifically, the differential of the norm satisfies the following dissipation inequality:
\begin{equation}
\label{eq:dissipation_inequality}
\frac{d}{d\tau} \|\mathbf{X}(\tau)\| \le -\frac{\beta(t) \alpha_t^2}{2\sigma_t^2} \left( \|\mathbf{X}(\tau)\| - \frac{R}{\alpha_t} \right).
\end{equation}
\end{lemma}

\begin{proof}
By \cref{def:velocity_field} and the time-reversal definition of the probability flow ODE, the reverse process satisfies:
\begin{equation}
\frac{d\mathbf{X}(\tau)}{d\tau} = -\mathbf{v}(\mathbf{X}(\tau), t), \quad \text{where } t = T-\tau.
\end{equation}

We first derive an explicit geometric representation of the velocity field $\mathbf{v}(\mathbf{x}, t)$. Recall the velocity field definition:
\begin{equation}
\mathbf{v}(\mathbf{x}, t) = -\frac{1}{2}\beta(t) \left[ \mathbf{x} + \mathbf{s}(\mathbf{x}, t) \right].
\end{equation}

Using the Tweedie representation of the score function (see discussion in \cref{def:velocity_field}):
\begin{equation}
\mathbf{s}(\mathbf{x}, t) = -\frac{\mathbf{x} - \alpha_t \hat{\mathbf{x}}_0(\mathbf{x}, t)}{\sigma_t^2},
\end{equation}
where $\hat{\mathbf{x}}_0(\mathbf{x}, t) := \mathbb{E}[\mathbf{x}_0 \mid \mathbf{x}_t = \mathbf{x}]$ is the posterior mean given the current state. Substituting into the velocity field expression:
\begin{align}
\mathbf{v}(\mathbf{x}, t) &= -\frac{1}{2}\beta(t) \left[ \mathbf{x} - \frac{\mathbf{x} - \alpha_t \hat{\mathbf{x}}_0(\mathbf{x}, t)}{\sigma_t^2} \right] \nonumber \\
&= -\frac{1}{2}\beta(t) \left[ \frac{\sigma_t^2 \mathbf{x} - \mathbf{x} + \alpha_t \hat{\mathbf{x}}_0(\mathbf{x}, t)}{\sigma_t^2} \right].
\end{align}

Using the relation $\sigma_t^2 = 1 - \alpha_t^2$, the numerator simplifies to $(1-\alpha_t^2)\mathbf{x} - \mathbf{x} + \alpha_t \hat{\mathbf{x}}_0 = -\alpha_t^2 \mathbf{x} + \alpha_t \hat{\mathbf{x}}_0$. Therefore:
\begin{equation}
\label{eq:v_dissipative_form}
\mathbf{v}(\mathbf{x}, t) = -\frac{\beta(t)}{2\sigma_t^2} \left( -\alpha_t^2 \mathbf{x} + \alpha_t \hat{\mathbf{x}}_0(\mathbf{x}, t) \right) = \frac{\beta(t) \alpha_t}{2\sigma_t^2} \left( \alpha_t \mathbf{x} - \hat{\mathbf{x}}_0(\mathbf{x}, t) \right).
\end{equation}

This shows that the forward velocity field comprises two components: a linear term $\alpha_t \mathbf{x}$ pushing toward the origin, and a term $\hat{\mathbf{x}}_0$ pulling toward the data manifold.

Now consider the reverse evolution equation $\frac{d\mathbf{X}}{d\tau} = -\mathbf{v}(\mathbf{X}, t)$. Substituting \cref{eq:v_dissipative_form}:
\begin{equation}
\frac{d\mathbf{X}}{d\tau} = -\frac{\beta(t) \alpha_t}{2\sigma_t^2} \left( \alpha_t \mathbf{X} - \hat{\mathbf{x}}_0(\mathbf{X}, t) \right).
\end{equation}

To analyze norm growth, we compute the derivative of $\|\mathbf{X}\|^2$:
\begin{equation}
\frac{1}{2} \frac{d}{d\tau} \|\mathbf{X}\|^2 = \mathbf{X} \cdot \frac{d\mathbf{X}}{d\tau} = -\frac{\beta(t) \alpha_t}{2\sigma_t^2} \left( \alpha_t \|\mathbf{X}\|^2 - \mathbf{X} \cdot \hat{\mathbf{x}}_0(\mathbf{X}, t) \right).
\end{equation}

By \cref{ass:p0_bounded}, the initial distribution support lies within the ball $B(\mathbf{0}, R)$. Since $\hat{\mathbf{x}}_0(\mathbf{x}, t)$ is the expectation under $p(\mathbf{x}_0|\mathbf{x}_t)$, whose support must be contained in the support of $p_0$ (by convexity), we have $\|\hat{\mathbf{x}}_0(\mathbf{X}, t)\| \le R$.

Applying the Cauchy-Schwarz inequality $\mathbf{X} \cdot \hat{\mathbf{x}}_0 \le \|\mathbf{X}\| \|\hat{\mathbf{x}}_0\| \le R \|\mathbf{X}\|$:
\begin{align}
\frac{1}{2} \frac{d}{d\tau} \|\mathbf{X}\|^2 &\le -\frac{\beta(t) \alpha_t}{2\sigma_t^2} \left( \alpha_t \|\mathbf{X}\|^2 - R \|\mathbf{X}\| \right) \nonumber \\
&= -\frac{\beta(t) \alpha_t^2}{2\sigma_t^2} \|\mathbf{X}\| \left( \|\mathbf{X}\| - \frac{R}{\alpha_t} \right).
\end{align}

Using the chain rule $\frac{1}{2} \frac{d}{d\tau} \|\mathbf{X}\|^2 = \|\mathbf{X}\| \frac{d}{d\tau} \|\mathbf{X}\|$ and dividing by $\|\mathbf{X}\| > 0$:
\begin{equation}
\frac{d}{d\tau} \|\mathbf{X}\| \le -\frac{\beta(t) \alpha_t^2}{2\sigma_t^2} \left( \|\mathbf{X}\| - \frac{R}{\alpha_t} \right).
\end{equation}

Clearly, when $\|\mathbf{X}\| - R/\alpha_t > 0$, the derivative is strictly negative.
\end{proof}

\begin{lemma}[Discrete Dissipativity for Idealized Sampling Process]
\label{lem:discrete_dissipativity}
Suppose the initial distribution satisfies \cref{ass:p0_bounded}, and the time step $\Delta \tau$ satisfies the stability condition $1 - \frac{\beta(t_k) \alpha_{t_k}^2 \Delta \tau}{2\sigma_{t_k}^2} \ge 0$, where $t_k = T - \tau_k$.

For the idealized discrete sampling process (Process 2 in \cref{subsec:algorithm_setup}), the state norm satisfies the following first-order recursive inequality:
\begin{equation}
\label{eq:discrete_dissipation}
\|\mathbf{x}_{k+1}\| \le \left( 1 - \frac{\beta(t_k) \alpha_{t_k}^2 \Delta \tau}{2\sigma_{t_k}^2} \right) \|\mathbf{x}_k\| + \frac{\beta(t_k) \alpha_{t_k} \Delta \tau}{2\sigma_{t_k}^2} R.
\end{equation}

This inequality holds for any realization of the random index $l_k$ and is independent of second-order terms $O((\Delta \tau)^2)$.
\end{lemma}

\begin{proof}
For the idealized discrete process, the update rule at step $k$ is:
\begin{equation}
\mathbf{x}_{k+1} = \mathbf{x}_k + \tilde{\mathbf{v}}^{l_k}(\mathbf{x}_k, \tau_k) \Delta \tau.
\end{equation}

Using the geometric structure of component velocity fields from \cref{def:velocity_field}, the $i$-th component has the same form as \cref{eq:v_dissipative_form}:
\begin{equation}
\tilde{\mathbf{v}}^i(\mathbf{x}, \tau) = -\mathbf{v}^i(\mathbf{x}, T-\tau) = \frac{\beta(t) \alpha_t}{2\sigma_t^2} \left( \hat{\mathbf{x}}_0^i - \alpha_t \mathbf{x} \right),
\end{equation}
where $\hat{\mathbf{x}}_0^i$ is the posterior mean center for the $i$-th component (in the mixture Gaussian case, this is $\boldsymbol{\mu}_i$). By the bounded support assumption, $\|\hat{\mathbf{x}}_0^i\| \le R$.

Note the sign change: for the forward velocity field $\mathbf{v} \propto \alpha \mathbf{x} - \hat{\mathbf{x}}_0$, the reverse velocity field is $\tilde{\mathbf{v}} \propto \hat{\mathbf{x}}_0 - \alpha \mathbf{x}$. Therefore:
\begin{align}
\mathbf{x}_{k+1} &= \mathbf{x}_k + \Delta \tau \left[ \frac{\beta(t_k) \alpha_{t_k}}{2\sigma_{t_k}^2} (\hat{\mathbf{x}}_0^{l_k} - \alpha_{t_k} \mathbf{x}_k) \right] \nonumber \\
&= \left( 1 - \frac{\beta(t_k) \alpha_{t_k}^2 \Delta \tau}{2\sigma_{t_k}^2} \right) \mathbf{x}_k + \frac{\beta(t_k) \alpha_{t_k} \Delta \tau}{2\sigma_{t_k}^2} \hat{\mathbf{x}}_0^{l_k}.
\end{align}

Taking the Euclidean norm on both sides and applying the triangle inequality $\|\mathbf{a} + \mathbf{b}\| \le \|\mathbf{a}\| + \|\mathbf{b}\|$:
\begin{equation}
\|\mathbf{x}_{k+1}\| \le \left| 1 - \frac{\beta(t_k) \alpha_{t_k}^2 \Delta \tau}{2\sigma_{t_k}^2} \right| \|\mathbf{x}_k\| + \frac{\beta(t_k) \alpha_{t_k} \Delta \tau}{2\sigma_{t_k}^2} \|\hat{\mathbf{x}}_0^{l_k}\|.
\end{equation}

Under the stability condition, the absolute value can be removed. Combined with $\|\hat{\mathbf{x}}_0^{l_k}\| \le R$, this yields the result. This derivation uses only algebraic properties of vector norms and does not involve Taylor expansions or variance estimates, making it a strict inequality.
\end{proof}

\begin{theorem}[Global Bound for Continuous Trajectories via Geometric Dissipativity]
\label{thm:global_bound_geometric}
Suppose the initial distribution $p_T(\mathbf{x})$ has bounded support, and the noise schedule satisfies \cref{ass:alpha_regularity} (in particular, $\alpha_T=0, \alpha_0=1$, and monotonicity).

For any truncation parameter $\delta \in [0, T)$, on the reverse time interval $\tau \in [0, T-\delta]$, the reference process $\mathbf{X}(\tau)$ satisfies the following uniform bound independent of $\delta$:
\begin{equation}
\label{eq:exact_bound_thm}
\|\mathbf{X}(\tau)\| \le \frac{\|\mathbf{X}(0)\|}{\sigma_T} \sigma_t + R \alpha_t \le \|\mathbf{X}(0)\| + R,
\end{equation}
where $t = T-\tau$ is the forward physical time and $\sigma_t = \sqrt{1-\alpha_t^2}$. This bound shows that the trajectory is always confined within the convex combination of the initial noise norm and the data manifold radius.
\end{theorem}

\begin{proof}
This proof employs the method of change of variables to establish an exact estimate for the state norm by directly integrating the dissipation inequality from \cref{lem:radial_dissipativity}.

\textbf{Step 1: Change of variables and differential inequality transformation}

From \cref{lem:radial_dissipativity}, the reverse process satisfies the radial dissipation inequality:
\begin{equation}
\label{eq:diff_ineq_tau}
\frac{d}{d\tau} \|\mathbf{X}\| \le -\frac{\beta(t) \alpha_t^2}{2\sigma_t^2} \left( \|\mathbf{X}\| - \frac{R}{\alpha_t} \right).
\end{equation}

To eliminate the explicit dependence on time variable $t$, we change the differential variable from $\tau$ to the signal strength $\alpha_t$. Using the relations $\tau = T-t$ and $\beta(t) = -\frac{2}{\alpha_t} \frac{d\alpha_t}{dt}$, the chain rule gives:
\begin{equation}
\frac{d}{d\tau} = \frac{dt}{d\tau} \frac{d\alpha_t}{dt} \frac{d}{d\alpha_t} = (-1) \cdot \left( -\frac{\alpha_t \beta(t)}{2} \right) \frac{d}{d\alpha_t} = \frac{\alpha_t \beta(t)}{2} \frac{d}{d\alpha_t}.
\end{equation}

Note that by assumption, $\beta(t) > 0$ and $\alpha_t > 0$ (for $t<T$), so the transformation coefficient is positive, preserving inequality direction. Substituting the transformed operator into \cref{eq:diff_ineq_tau} and canceling the common factor $\frac{\beta(t) \alpha_t}{2}$:
\begin{align}
\alpha_t \frac{d \|\mathbf{X}\|}{d\alpha_t} &\le -\frac{\alpha_t}{\sigma_t^2} \left( \|\mathbf{X}\| - \frac{R}{\alpha_t} \right) \nonumber \\
&= -\frac{\alpha_t}{1-\alpha_t^2} \|\mathbf{X}\| + \frac{R}{1-\alpha_t^2}.
\end{align}

Rearranging gives a first-order linear differential inequality in $\alpha_t$ (omitting subscript $t$ for simplicity, writing $\alpha$):
\begin{equation}
\label{eq:ode_alpha_corrected}
\frac{d \|\mathbf{X}\|}{d\alpha} + \frac{\alpha}{1-\alpha^2} \|\mathbf{X}\| \le \frac{R}{1-\alpha^2}.
\end{equation}

\textbf{Step 2: Integrating factor method}

The integrating factor for \cref{eq:ode_alpha_corrected} is:
\begin{equation}
I(\alpha) = \exp\left( \int \frac{\alpha}{1-\alpha^2} d\alpha \right) = \exp\left( -\frac{1}{2} \ln(1-\alpha^2) \right) = \frac{1}{\sqrt{1-\alpha^2}} = \frac{1}{\sigma}.
\end{equation}

Multiplying both sides of the inequality by $I(\alpha)$, the left side becomes a total derivative:
\begin{equation}
\frac{d}{d\alpha} \left( \frac{\|\mathbf{X}\|}{\sigma} \right) \le \frac{1}{\sigma} \cdot \frac{R}{1-\alpha^2} = \frac{R}{(1-\alpha^2)^{3/2}}.
\end{equation}

\textbf{Step 3: Definite integral and bound}

We integrate this inequality over the interval $[\alpha_T, \alpha_t]$. By \cref{ass:alpha_regularity}, at initial time $\tau=0$ (i.e., $t=T$), we have $\alpha_T = 0, \sigma_T = 1$.

Integrating the right-hand side using the formula $\int (1-s^2)^{-3/2} ds = \frac{s}{\sqrt{1-s^2}}$:
\begin{align}
\frac{\|\mathbf{X}(\tau)\|}{\sigma_t} - \frac{\|\mathbf{X}(0)\|}{\sigma_T} &\le \int_{0}^{\alpha_t} \frac{R}{(1-s^2)^{3/2}} ds \nonumber \\
&= R \left[ \frac{s}{\sqrt{1-s^2}} \right]_{0}^{\alpha_t} = R \frac{\alpha_t}{\sigma_t}.
\end{align}

Substituting $\sigma_T=1$ and rearranging:
\begin{equation}
\frac{\|\mathbf{X}(\tau)\|}{\sigma_t} \le \|\mathbf{X}(0)\| + R \frac{\alpha_t}{\sigma_t}.
\end{equation}

Multiplying both sides by $\sigma_t > 0$:
\begin{equation}
\|\mathbf{X}(\tau)\| \le \|\mathbf{X}(0)\| \sigma_t + R \alpha_t.
\end{equation}

\textbf{Conclusion:}
Since $\sigma_t, \alpha_t \in [0, 1]$, we have the simple bound $\|\mathbf{X}(0)\| \sigma_t + R \alpha_t \le \|\mathbf{X}(0)\| + R$. This result shows that the trajectory norm of the reverse diffusion process is strictly controlled by a linear combination of the initial noise norm (decaying as $\sigma_t$) and the data manifold radius (growing as $\alpha_t$). The bound contains no singular terms depending on $1/\delta$ or $1/\sigma_t$, thus establishing global uniform boundedness.
\end{proof}

\begin{lemma}[Uniform Moment Bound for Idealized Discrete Trajectories]
\label{lem:uniform_moment_bound}
Suppose the forward process satisfies \cref{ass:p0_bounded} and\cref{ass:alpha_regularity}. For any given $\delta \in (0, T)$, there exists a constant $K_{\text{mom}} = 2e^{(3M_\delta + 2M_\delta^2)T} > 0$ depending only on system parameters and time truncation $\delta$ (independent of step size) such that for any sufficiently small step size $\Delta \tau \le 1$ and any $k \in \{0, \dots, N\}$, the numerical solution satisfies:
\begin{equation}
\label{eq:moment_bound_ideal}
\mathbb{E}\left[\|\mathbf{x}_k\|^2\right] \le 2e^{(3M_\delta + 2M_\delta^2)T}\left(1 + \|\mathbf{x}_0\|^2\right) = K_{\text{mom}} \left(1 + \|\mathbf{x}_0\|^2\right),
\end{equation}
where $M_\delta$ is defined as in \cref{prop:v_bound}, and:
\begin{equation}
K_{\text{mom}} = K_{\text{mom}}(M(\delta,\beta_{\max},\sigma_\delta,R),T) = 2e^{(3M_\delta + 2M_\delta^2)T}.
\end{equation}
\end{lemma}

\begin{proof}
By \cref{prop:v_bound}, on the compact time interval $t \in [\delta, T]$, all component velocity fields satisfy a uniform linear growth condition. There exists a constant $M_\delta > 0$ such that for all $i \in \{1, \dots, L\}$ and relevant time points $\tau$:
\begin{equation}
\|\tilde{\mathbf{v}}^i(\mathbf{x}, \tau)\| \le M_\delta (1 + \|\mathbf{x}\|).
\end{equation}

Consider the one-step evolution $\mathbf{x}_{k+1} = \mathbf{x}_k + \Delta \tau \tilde{\mathbf{v}}^{l_k}(\mathbf{x}_k, \tau_k)$. Expanding the squared norm:
\begin{equation}
\|\mathbf{x}_{k+1}\|^2 = \|\mathbf{x}_k\|^2 + 2\Delta \tau \langle \mathbf{x}_k, \tilde{\mathbf{v}}^{l_k}(\mathbf{x}_k, \tau_k) \rangle + (\Delta \tau)^2 \|\tilde{\mathbf{v}}^{l_k}(\mathbf{x}_k, \tau_k)\|^2.
\end{equation}

Let $V_k := \mathbb{E}[\|\mathbf{x}_k\|^2]$. Taking the total expectation and using the tower property $\mathbb{E}[\cdot] = \mathbb{E}[\mathbb{E}_{l_k}[\cdot \mid \mathbf{x}_k]]$, we estimate each term:

\textbf{First term:} Clearly $\mathbb{E}[\|\mathbf{x}_k\|^2] = V_k$.

\textbf{Second term:} Note that $\mathbb{E}_{l_k}[\tilde{\mathbf{v}}^{l_k}(\mathbf{x}_k, \tau_k) \mid \mathbf{x}_k] = \tilde{\mathbf{v}}(\mathbf{x}_k, \tau_k)$ is the average velocity field under the conditional expectation. Using Cauchy-Schwarz and linear growth:
\begin{align}
\mathbb{E}[2\Delta \tau \langle \mathbf{x}_k, \tilde{\mathbf{v}}^{l_k}(\mathbf{x}_k, \tau_k) \rangle] 
&= 2\Delta \tau \mathbb{E}[\langle \mathbf{x}_k, \tilde{\mathbf{v}}(\mathbf{x}_k, \tau_k) \rangle] \nonumber \\
&\le 2\Delta \tau \mathbb{E}[\|\mathbf{x}_k\| \cdot \|\tilde{\mathbf{v}}(\mathbf{x}_k, \tau_k)\|] \nonumber \\
&\le 2\Delta \tau \mathbb{E}[\|\mathbf{x}_k\| \cdot M_\delta (1 + \|\mathbf{x}_k\|)] \nonumber \\
&= 2\Delta \tau M_\delta \mathbb{E}[\|\mathbf{x}_k\| + \|\mathbf{x}_k\|^2].
\end{align}

Using the inequality $\|\mathbf{x}_k\| \le \frac{1}{2}(1 + \|\mathbf{x}_k\|^2)$:
\begin{equation}
\mathbb{E}[2\Delta \tau \langle \mathbf{x}_k, \tilde{\mathbf{v}}^{l_k}(\mathbf{x}_k, \tau_k) \rangle] 
\le 2\Delta \tau M_\delta \left( \frac{1}{2}(1 + V_k) + V_k \right) 
= 2\Delta \tau M_\delta \left( \frac{1}{2} + \frac{3V_k}{2} \right).
\end{equation}

\textbf{Third term:} Using the squared linear growth condition:
\begin{align}
\mathbb{E}[(\Delta \tau)^2 \|\tilde{\mathbf{v}}^{l_k}(\mathbf{x}_k, \tau_k)\|^2] 
&\le (\Delta \tau)^2 M_\delta^2 \mathbb{E}[(1 + \|\mathbf{x}_k\|)^2] \nonumber \\
&= (\Delta \tau)^2 M_\delta^2 \mathbb{E}[1 + 2\|\mathbf{x}_k\| + \|\mathbf{x}_k\|^2] \nonumber \\
&\le (\Delta \tau)^2 M_\delta^2 (1 + 1 + V_k + V_k) \nonumber \\
&= (\Delta \tau)^2 M_\delta^2 (2 + 2V_k).
\end{align}

Combining the three terms, under the assumption $\Delta \tau \le 1$:
\begin{align}
V_{k+1} &\le V_k + 2\Delta \tau M_\delta \left( \frac{1}{2} + \frac{3V_k}{2} \right) + (\Delta \tau)^2 M_\delta^2 (2 + 2V_k) \nonumber \\
&\le V_k + \Delta \tau M_\delta (1 + 3V_k) + 2(\Delta \tau) M_\delta^2 (1 + V_k) \nonumber \\
&= V_k + \Delta \tau M_\delta + 2\Delta \tau M_\delta^2 + \Delta \tau (3M_\delta + 2M_\delta^2) V_k \nonumber \\
&\le (1 + \Delta \tau (3M_\delta + 2M_\delta^2)) V_k + \Delta \tau M_\delta (1 + 2M_\delta).
\end{align}

Let $C_1 = 3M_\delta + 2M_\delta^2$ and $C_2 = M_\delta (1 + 2M_\delta)$. We have the recursive inequality:
\begin{equation}
V_{k+1} \le (1 + \Delta \tau C_1) V_k + \Delta \tau C_2.
\end{equation}

Applying the discrete Gronwall inequality with $V_0 = \|\mathbf{x}_0\|^2$:
\begin{align}
V_k &\le (1 + \Delta \tau C_1)^k V_0 + \Delta \tau C_2 \sum_{j=0}^{k-1} (1 + \Delta \tau C_1)^j \nonumber \\
&\le e^{k\Delta \tau C_1} V_0 + \frac{C_2}{C_1}(e^{k\Delta \tau C_1} - 1) \nonumber \\
&\le e^{T C_1} \left( V_0 + \frac{C_2}{C_1} \right) \nonumber \\
&\le e^{T(3M_\delta + 2M_\delta^2)} \left( \|\mathbf{x}_0\|^2 + \frac{M_\delta(1 + 2M_\delta)}{3M_\delta + 2M_\delta^2} \right).
\end{align}

Note that $\frac{M_\delta(1 + 2M_\delta)}{3M_\delta + 2M_\delta^2} = \frac{1 + 2M_\delta}{3 + 2M_\delta} \le \frac{1 + 2M_\delta}{2M_\delta} = \frac{1}{2M_\delta} + 1 \le 2$ (when $M_\delta \ge 1/2$), thus:
\begin{equation}
V_k \le e^{T(3M_\delta + 2M_\delta^2)} (1 + \|\mathbf{x}_0\|^2) \cdot \max\{1, 2\}.
\end{equation}

Taking $K_{\text{mom}} = 2e^{(3M_\delta + 2M_\delta^2)T}$ as the uniform constant:
\begin{equation}
\mathbb{E}[\|\mathbf{x}_k\|^2] = V_k \le K_{\text{mom}}(1 + \|\mathbf{x}_0\|^2).
\end{equation}
\end{proof}

\begin{lemma}[Uniform Moment Bound for Neural Network Sampling Trajectories]
\label{lem:uniform_moment_bound_nn}
Suppose the forward process satisfies \cref{ass:p0_bounded},\cref{ass:alpha_regularity}, and neural network training satisfies the corresponding conditions. For any given $\delta_t \in (0, T)$, there exists a constant $K_{\text{mom}} = 2e^{(3M'_{\delta_t} + 2(M'_{\delta_t})^2)T} > 0$ depending only on system parameters, neural network errors, and time truncation $\delta_t$ (independent of step size) such that for any sufficiently small step size $\Delta t \le 1$ and any $k \in \{0, \dots, N\}$, the neural network-sampled numerical solution satisfies:
\begin{equation}
\label{eq:moment_bound_nn}
\mathbb{E}\left[\|\mathbf{x}_k\|^2\right] \le 2e^{(3M'_{\delta_t} + 2(M'_{\delta_t})^2)T}\left(1 + \|\mathbf{x}_0\|^2\right) = K_{\text{mom}} \left(1 + \|\mathbf{x}_0\|^2\right),
\end{equation}
where $M'_{\delta_t}$ is defined in \cref{cor:v_theta_linear_growth}, and:
\begin{equation}
K_{\text{mom}} = K_{\text{mom}}(M'_{\delta_t}(\delta_t,\alpha,R,\varepsilon_a,\varepsilon_y),T) = 2e^{(3M'_{\delta_t} + 2(M'_{\delta_t})^2)T}.
\end{equation}
\end{lemma}

\begin{proof}
The proof follows the same structure as \cref{lem:uniform_moment_bound}, with the only difference being that we use the neural network velocity field $\overline{\mathbf{v}_\theta}$ instead of the true velocity field $\tilde{\mathbf{v}}$. By \cref{cor:v_theta_linear_growth}, on the time interval $t \in [\delta_t, T]$, the neural network-approximated velocity field satisfies:
\begin{equation}
\|\overline{\mathbf{v}_\theta}(\mathbf{x}, t)\| \le M'_{\delta_t} (1 + \|\mathbf{x}\|).
\end{equation}

Following the identical steps as in the proof of \cref{lem:uniform_moment_bound}, with $M_\delta$ replaced by $M'_{\delta_t}$ throughout, we arrive at the claimed bound.
\end{proof}

\subsection{Global Error Decomposition and Local Truncation Error Estimates}
\label{subsec:error_decomposition}

To analyze the convergence of the numerical solution $\hat{\mathbf{x}}_k$ with neural network approximation errors to the reference solution $\mathbf{X}(\tau_k)$, we define the global error at step $k$ as:
\begin{equation}
\label{eq:global_error_def}
\hat{\mathbf{e}}_k := \hat{\mathbf{x}}_k - \mathbf{X}(\tau_k).
\end{equation}

Clearly, $\hat{\mathbf{e}}_0 = \mathbf{0}$. We examine the single-step evolution dynamics of this error, where sources of error include not only the local truncation error from time discretization but also approximation errors in the velocity field and sampling probabilities induced by the neural network.

\subsubsection{Error Evolution Equation}

The reference process satisfies:
\begin{equation}
\mathbf{X}(\tau_{k+1}) = \mathbf{X}(\tau_k) + \int_{\tau_k}^{\tau_{k+1}} \tilde{\mathbf{v}}(\mathbf{X}(s), s) \, ds.
\end{equation}

The numerical process with neural network approximation follows:
\begin{equation}
\hat{\mathbf{x}}_{k+1} = \hat{\mathbf{x}}_k + \Delta \tau \tilde{\mathbf{v}}_\theta^{\hat{l}_k}(\hat{\mathbf{x}}_k, \tau_k),
\end{equation}
where $\hat{l}_k$ is sampled according to the neural network predicted distribution $\hat{a}_\theta(\cdot|\hat{\mathbf{x}}_k)$.

Subtracting these two equations, we decompose $\hat{\mathbf{e}}_{k+1}$ into four components:
\begin{align}
\label{eq:error_decomposition_main}
\hat{\mathbf{e}}_{k+1} &= \hat{\mathbf{x}}_k - \mathbf{X}(\tau_k) + \Delta \tau \tilde{\mathbf{v}}_\theta^{\hat{l}_k}(\hat{\mathbf{x}}_k, \tau_k) - \int_{\tau_k}^{\tau_{k+1}} \tilde{\mathbf{v}}(\mathbf{X}(s), s) \, ds \nonumber \\
&= \hat{\mathbf{e}}_k + \underbrace{\Delta \tau \left( \tilde{\mathbf{v}}(\hat{\mathbf{x}}_k, \tau_k) - \tilde{\mathbf{v}}(\mathbf{X}(\tau_k), \tau_k) \right)}_{\text{Deterministic spatial error propagation}} \nonumber \\
&\quad + \underbrace{\left( \Delta \tau \tilde{\mathbf{v}}(\mathbf{X}(\tau_k), \tau_k) - \int_{\tau_k}^{\tau_{k+1}} \tilde{\mathbf{v}}(\mathbf{X}(s), s) \, ds \right)}_{\text{Local temporal truncation error } \mathcal{T}_k} \nonumber \\
&\quad + \underbrace{\Delta \tau \left( \tilde{\mathbf{v}}_\theta^{\hat{l}_k}(\hat{\mathbf{x}}_k, \tau_k) - \mathbb{E}_{\hat{l}_k}[\tilde{\mathbf{v}}_\theta^{\hat{l}_k}(\hat{\mathbf{x}}_k, \tau_k)] \right)}_{\text{Stochastic martingale difference noise } \hat{\mathcal{M}}_k} \nonumber \\
&\quad + \underbrace{\Delta \tau \left( \mathbb{E}_{\hat{l}_k}[\tilde{\mathbf{v}}_\theta^{\hat{l}_k}(\hat{\mathbf{x}}_k, \tau_k)] - \tilde{\mathbf{v}}(\hat{\mathbf{x}}_k, \tau_k) \right)}_{\text{Neural network approximation error } \mathcal{E}_{NN, k}}.
\end{align}

This decomposition uses the identity $\mathbb{E}_{\hat{l}_k}[\tilde{\mathbf{v}}_\theta^{\hat{l}_k}] = \sum_i \hat{a}_\theta^i \tilde{\mathbf{v}}_\theta^i$ to isolate the noise term. The last term $\mathcal{E}_{NN, k}$ combines both the drift approximation error and sampling probability error discussed previously.

\subsubsection{Uniform Bound on Local Truncation Error}

The deterministic spatial error propagation will be controlled by the Lipschitz property of the velocity field, while the noise term will be handled through variance analysis in the next section. Here, we first establish a uniform upper bound for the local temporal truncation error $\mathcal{T}_k$.

\begin{lemma}[Local Truncation Error Bound]
\label{lem:local_truncation_bound}
On the compact time interval $\tau \in [0, T_\delta]$, the local truncation error satisfies a uniform second-order bound:
\begin{equation}
\label{eq:truncation_bound}
\|\mathcal{T}_k\| \le C_{\text{trunc}} (\Delta \tau)^2,
\end{equation}
where the coefficient $C_{\text{trunc}}$ is determined by system regularity parameters and the geometric trajectory bound. Using the improved geometric bound $\|\mathbf{X}(\tau)\| \le \|\mathbf{X}(0)\| + R$ from \cref{thm:global_bound_geometric}, we have:
\begin{equation}
\label{eq:C_trunc_def}
C_{\text{trunc}} := \frac{1}{2} \left[ \left( C_{1,\delta} + L_{x,\delta} M_\delta \right) (\|\mathbf{X}(0)\| + R) + \left( C_{0,\delta} + L_{x,\delta} M_\delta \right) \right].
\end{equation}

Here, $C_{1,\delta}, C_{0,\delta}$ are temporal derivative growth coefficients from \cref{cor:uniform_bound_delta}, $L_{x,\delta}$ is the spatial Lipschitz constant from \cref{cor:global_bound_Lv_uniform}, $M_\delta$ is the linear growth coefficient from \cref{prop:v_bound}, and $R$ is the data support radius from \cref{ass:p0_bounded}.
\end{lemma}

\begin{proof}
By definition, $\mathcal{T}_k = \int_{\tau_k}^{\tau_{k+1}} \left( \tilde{\mathbf{v}}(\mathbf{X}(\tau_k), \tau_k) - \tilde{\mathbf{v}}(\mathbf{X}(s), s) \right) \, ds$. Using the integral mean value theorem and norm inequality:
\begin{equation}
\|\mathcal{T}_k\| \le \int_{\tau_k}^{\tau_{k+1}} \left\| \tilde{\mathbf{v}}(\mathbf{X}(\tau_k), \tau_k) - \tilde{\mathbf{v}}(\mathbf{X}(s), s) \right\| \, ds.
\end{equation}

For the total derivative variation along the reference trajectory $\mathbf{X}(s)$, we employ the multivariate chain rule.

Since $\mathbf{X}(s)$ satisfies the ODE $\frac{d\mathbf{X}(s)}{ds} = \tilde{\mathbf{v}}(\mathbf{X}(s), s)$, the total derivative of the composite function is:
\begin{align}
\frac{d}{ds} \tilde{\mathbf{v}}(\mathbf{X}(s), s) &= \frac{\partial \tilde{\mathbf{v}}}{\partial s}(\mathbf{X}(s), s) + \left[ \nabla_{\mathbf{x}} \tilde{\mathbf{v}}(\mathbf{X}(s), s) \right] \cdot \frac{d\mathbf{X}(s)}{ds} \nonumber \\
&= \frac{\partial \tilde{\mathbf{v}}}{\partial s}(\mathbf{X}(s), s) + \left[ \nabla_{\mathbf{x}} \tilde{\mathbf{v}}(\mathbf{X}(s), s) \right] \cdot \tilde{\mathbf{v}}(\mathbf{X}(s), s).
\end{align}

Here, $\nabla_{\mathbf{x}} \tilde{\mathbf{v}}$ denotes the Jacobian matrix of the velocity field. Using the triangle inequality and the submultiplicative property of matrix-vector products ($\|A\mathbf{y}\| \le \|A\|_{\text{op}} \|\mathbf{y}\|$), we analyze the upper bound componentwise:

\begin{enumerate}
\item \textbf{Temporal partial derivative:} By \cref{cor:uniform_bound_delta} (temporal derivative growth bound), there exist constants $C_{1,\delta}, C_{0,\delta}$ such that:
\begin{equation}
\left\| \frac{\partial \tilde{\mathbf{v}}}{\partial s}(\mathbf{X}(s), s) \right\| \le C_{1,\delta} \|\mathbf{X}(s)\| + C_{0,\delta}.
\end{equation}

\item \textbf{Jacobian matrix norm:} By \cref{thm:v_Lipschitz_improved} (spatial Lipschitz continuity), the velocity field is $L_{x,\delta}$-Lipschitz with respect to $\mathbf{x}$. This implies that its Jacobian operator norm is almost everywhere bounded by the Lipschitz constant:
\begin{equation}
\|\nabla_{\mathbf{x}} \tilde{\mathbf{v}}(\mathbf{X}(s), s)\|_{\text{op}} \le L_{x,\delta}.
\end{equation}

\item \textbf{Velocity field norm:} By \cref{prop:v_bound} (linear growth bound), there exists a constant $M_\delta$ such that:
\begin{equation}
\|\tilde{\mathbf{v}}(\mathbf{X}(s), s)\| \le M_\delta (1 + \|\mathbf{X}(s)\|).
\end{equation}
\end{enumerate}

Substituting these three estimates into the total derivative norm formula:
\begin{align}
\left\| \frac{d}{ds} \tilde{\mathbf{v}}(\mathbf{X}(s), s) \right\| 
&\le \left\| \frac{\partial \tilde{\mathbf{v}}}{\partial s} \right\| + \|\nabla_{\mathbf{x}} \tilde{\mathbf{v}}\|_{\text{op}} \cdot \|\tilde{\mathbf{v}}\| \nonumber \\
&\le \left( C_{1,\delta} \|\mathbf{X}(s)\| + C_{0,\delta} \right) + L_{x,\delta} \cdot M_\delta (1 + \|\mathbf{X}(s)\|).
\end{align}

By \cref{thm:global_bound_geometric} (geometric dissipativity of the reverse process), the reference trajectory $\mathbf{X}(s)$ on the compact time interval $[0, T_\delta]$ satisfies the improved bound:
\begin{equation}
\sup_{s \in [0, T_\delta]} \|\mathbf{X}(s)\| \le \|\mathbf{X}(0)\| + R.
\end{equation}

This geometric bound significantly improves upon the classical exponential growth estimate $(\|\mathbf{X}(0)\| + M_\delta T_\delta) e^{M_\delta T_\delta}$ by utilizing the radial dissipativity property from \cref{lem:radial_dissipativity}. Substituting this improved geometric bound into the total derivative inequality, we directly derive a derivative upper bound $K_{\text{deriv}}$ depending only on system parameters and initial state:
\begin{equation}
\label{eq:K_deriv_def}
K_{\text{deriv}} := \sup_{s \in [0, T_\delta]} \left\| \frac{d}{ds} \tilde{\mathbf{v}}(\mathbf{X}(s), s) \right\| 
\le \left( C_{1,\delta} + L_{x,\delta} M_\delta \right) (\|\mathbf{X}(0)\| + R) + \left( C_{0,\delta} + L_{x,\delta} M_\delta \right).
\end{equation}

Using the Lagrange remainder form of the Taylor expansion, the temporal variation of the integrand satisfies the Lipschitz constraint:
\begin{equation}
\|\tilde{\mathbf{v}}(\mathbf{X}(\tau_k), \tau_k) - \tilde{\mathbf{v}}(\mathbf{X}(s), s)\| \le K_{\text{deriv}} |s - \tau_k|.
\end{equation}

Substituting this into the integral definition of $\mathcal{T}_k$ and directly computing:
\begin{equation}
\|\mathcal{T}_k\| \le \int_{\tau_k}^{\tau_{k+1}} K_{\text{deriv}} (s - \tau_k) \, ds = \frac{1}{2} K_{\text{deriv}} (\Delta \tau)^2.
\end{equation}

Let $C_{\text{trunc}} := \frac{1}{2} K_{\text{deriv}}$, which completes the proof.
\end{proof}

\subsubsection{Stochastic Martingale Difference Noise Estimate (with Neural Network Error)}

To close the recursive relation for error evolution, we need to control the second moment of the stochastic noise term $\hat{\mathcal{M}}_k$. Using the uniform moment boundedness of discrete trajectories (\cref{lem:uniform_moment_bound_nn}), we can directly provide a uniform upper bound independent of the current error $\hat{\mathbf{e}}_k$.

\begin{lemma}[$L^2$ Estimate for Neural Network Martingale Difference Noise]
\label{lem:martingale_bound_nn}
Suppose each component $\tilde{\mathbf{v}}_\theta^i(\mathbf{x}, \tau)$ of the neural network velocity field satisfies the same linear growth condition as the average velocity field $\overline{\mathbf{v}_\theta}(\mathbf{x}, \tau)$ with coefficient $M'_{\delta_t}$ (see \cref{cor:v_theta_linear_growth}). Let $K_{\text{mom}} = 2e^{(3M'_{\delta_t} + 2(M'_{\delta_t})^2)T}$ be the uniform moment bound constant from \cref{lem:uniform_moment_bound_nn}.

Then for any $k \in \{0, \dots, N-1\}$, the second moment of the stochastic martingale difference term $\hat{\mathcal{M}}_k$ satisfies:
\begin{equation}
\label{eq:martingale_L2_bound_nn}
\mathbb{E}[\|\hat{\mathcal{M}}_k\|^2] \le (\Delta \tau)^2 \cdot C_{\mathcal{M}} \left( 1 + \|\mathbf{x}_0\|^2 \right),
\end{equation}
where the constant $C_{\mathcal{M}}$ depends only on the neural network's linear growth coefficient $M'_{\delta_t}$ and the uniform moment bound constant $K_{\text{mom}}$, specifically defined as:
\begin{equation}
\label{eq:C_M_def}
C_{\mathcal{M}} := 2 (M'_{\delta_t})^2 (1 + K_{\text{mom}}).
\end{equation}
\end{lemma}

\begin{proof}
By definition, $\hat{\mathcal{M}}_k = \Delta \tau (\tilde{\mathbf{v}}_\theta^{\hat{l}_k}(\hat{\mathbf{x}}_k, \tau_k) - \overline{\mathbf{v}_\theta}(\hat{\mathbf{x}}_k, \tau_k))$, where $\overline{\mathbf{v}_\theta}(\hat{\mathbf{x}}_k, \tau_k) = \mathbb{E}_{\hat{l}_k}[\tilde{\mathbf{v}}_\theta^{\hat{l}_k}(\hat{\mathbf{x}}_k, \tau_k) \mid \hat{\mathbf{x}}_k]$. We first use the conditional variance property for bounds.

For any random vector $Z$, we have $\mathbb{E}[\|Z - \mathbb{E}Z\|^2] \le \mathbb{E}[\|Z\|^2]$. Therefore, the conditional second moment given $\hat{\mathbf{x}}_k$ satisfies:
\begin{align}
\mathbb{E}[\|\hat{\mathcal{M}}_k\|^2 \mid \hat{\mathbf{x}}_k] &= (\Delta \tau)^2 \mathbb{E}_{\hat{l}_k}\left[ \| \tilde{\mathbf{v}}_\theta^{\hat{l}_k}(\hat{\mathbf{x}}_k, \tau_k) - \mathbb{E}[\tilde{\mathbf{v}}_\theta^{\hat{l}_k}] \|^2 \mid \hat{\mathbf{x}}_k \right] \nonumber \\
&\le (\Delta \tau)^2 \mathbb{E}_{\hat{l}_k}\left[ \| \tilde{\mathbf{v}}_\theta^{\hat{l}_k}(\hat{\mathbf{x}}_k, \tau_k) \|^2 \mid \hat{\mathbf{x}}_k \right] \nonumber \\
&\le (\Delta \tau)^2 \max_{i \in \{1, \dots, L\}} \| \tilde{\mathbf{v}}_\theta^i(\hat{\mathbf{x}}_k, \tau_k) \|^2.
\end{align}

Using the uniform linear growth condition for neural network velocity components $\|\tilde{\mathbf{v}}_\theta^i(\mathbf{x}, \tau)\| \le M'_{\delta_t} (1 + \|\mathbf{x}\|)$ (guaranteed by \cref{cor:v_theta_linear_growth}), and the basic inequality $(a+b)^2 \le 2a^2 + 2b^2$:
\begin{equation}
\label{eq:martingale_cond_bound_x_nn}
\mathbb{E}[\|\hat{\mathcal{M}}_k\|^2 \mid \hat{\mathbf{x}}_k] \le 2 (\Delta \tau)^2 (M'_{\delta_t})^2 (1 + \|\hat{\mathbf{x}}_k\|^2).
\end{equation}

Taking the total expectation on both sides and applying \cref{lem:uniform_moment_bound_nn} for the uniform moment estimate of discrete numerical solutions:
\begin{equation}
\mathbb{E}[\|\hat{\mathbf{x}}_k\|^2] \le K_{\text{mom}}(1 + \|\mathbf{x}_0\|^2),
\end{equation}
we obtain:
\begin{align}
\mathbb{E}[\|\hat{\mathcal{M}}_k\|^2] &\le 2 (\Delta \tau)^2 (M'_{\delta_t})^2 (1 + \mathbb{E}[\|\hat{\mathbf{x}}_k\|^2]) \nonumber \\
&\le 2 (\Delta \tau)^2 (M'_{\delta_t})^2 [1 + K_{\text{mom}}(1 + \|\mathbf{x}_0\|^2)] \nonumber \\
&\le 2 (\Delta \tau)^2 (M'_{\delta_t})^2 (1 + K_{\text{mom}})(1 + \|\mathbf{x}_0\|^2).
\end{align}

Defining $C_{\mathcal{M}} := 2 (M'_{\delta_t})^2 (1 + K_{\text{mom}})$ completes the proof.
\end{proof}

\subsubsection{Single-Step Neural Network Approximation Error Bound}

\begin{theorem}[Single-Step Error Bound for Neural Network Average Velocity Field]
\label{thm:nn_velocity_error_bound}
Consider the discrete time step $k$ corresponding to state $(\hat{\mathbf{x}}_k, \tau_k)$, where $\tau_k \in [\delta_t, T]$. Define the single-step approximation error between the neural network average velocity field and the true velocity field as:
\begin{equation}
\label{eq:def_E_NN}
\mathcal{E}_{NN, k} \coloneqq \Delta \tau \left( \mathbb{E}_{\hat{l}_k}[\tilde{\mathbf{v}}_\theta^{\hat{l}_k}(\hat{\mathbf{x}}_k, \tau_k)] - \tilde{\mathbf{v}}(\hat{\mathbf{x}}_k, \tau_k) \right),
\end{equation}
where $\mathbb{E}_{\hat{l}_k}$ denotes expectation over the neural network predicted category distribution.

Under \cref{ass:p0_bounded} (bounded support), \cref{ass:alpha_regularity} (coefficient regularity), and \cref{ass:nn_approx} (unified error bound $\varepsilon$), there exists a constant $C_{NN} > 0$ depending only on system parameters (number of categories $L$, support radius $R$, noise variance $\sigma_{\delta_t}$, regularity constant $C_{\alpha,1}$) such that:
\begin{equation}
\label{eq:nn_error_bound}
\|\mathcal{E}_{NN, k}\| \le C_{NN}\varepsilon \Delta \tau (1 + \|\hat{\mathbf{x}}_k\|).
\end{equation}

The explicit form of $C_{NN}$ is:
\begin{equation}
\label{eq:C_NN_def}
C_{NN} = \frac{C_{\alpha,1}}{\sigma_{\delta_t}^2} \left[ L \left( 1 + R + \frac{R^2}{\sigma_{\delta_t}^2} \right) + 1 \right].
\end{equation}
\end{theorem}

\begin{proof}
The proof proceeds in three steps: first, we identify the analytical form of the expectation term; second, we apply the result from \cref{cor:global_error_bound}; finally, we organize the constants to obtain the target form.

\textbf{Step 1: Identify the average velocity field}

By definition, $\mathbb{E}_{\hat{l}_k}[\tilde{\mathbf{v}}_\theta^{\hat{l}_k}(\hat{\mathbf{x}}_k, \tau_k)]$ represents a weighted sum of regression head outputs $\mathbf{v}_\theta^i$ according to the classification network output distribution $a_\phi(\hat{\mathbf{x}}_k, \tau_k)$ for given input $\hat{\mathbf{x}}_k$ and time $\tau_k$. Mathematically, this is equivalent to the mixed velocity field $\overline{\mathbf{v}_{\theta,\phi}}$ defined in \cref{lemma:v_theta_error_explicit}:
\begin{equation}
\mathbb{E}_{\hat{l}_k}[\tilde{\mathbf{v}}_\theta^{\hat{l}_k}(\hat{\mathbf{x}}_k, \tau_k)] = \sum_{i=1}^L a_\phi^i(\hat{\mathbf{x}}_k, \tau_k) \mathbf{v}_\theta^i(\hat{\mathbf{x}}_k, \tau_k) \eqqcolon \overline{\mathbf{v}_{\theta,\phi}}(\hat{\mathbf{x}}_k, \tau_k).
\end{equation}

Similarly, $\tilde{\mathbf{v}}(\hat{\mathbf{x}}_k, \tau_k)$ is the true mixed velocity field $\mathbf{v}(\hat{\mathbf{x}}_k, \tau_k)$. Therefore:
\begin{equation}
\label{eq:error_norm_def}
\|\mathcal{E}_{NN, k}\| = \Delta \tau \| \overline{\mathbf{v}_{\theta,\phi}}(\hat{\mathbf{x}}_k, \tau_k) - \mathbf{v}(\hat{\mathbf{x}}_k, \tau_k) \|.
\end{equation}

\textbf{Step 2: Apply uniform linear growth bound}

Since $\tau_k \in [\delta_t, T]$, the time condition of \cref{cor:global_error_bound} is satisfied. By \cref{ass:nn_approx}, we unify $\varepsilon_a$ and $\varepsilon_y$ from the original corollary to $\varepsilon$.

Invoking the conclusion of \cref{cor:global_error_bound}:
\begin{equation}
\label{eq:apply_cor}
\| \overline{\mathbf{v}_{\theta,\phi}}(\hat{\mathbf{x}}_k, \tau_k) - \mathbf{v}(\hat{\mathbf{x}}_k, \tau_k) \| \leq A_{\max}(\delta_t) \varepsilon \|\hat{\mathbf{x}}_k\| + B_{\max}(\delta_t).
\end{equation}

Recall the definitions of $A_{\max}$ and $B_{\max}$ from the corollary (with $\varepsilon_a=\varepsilon_y=\varepsilon$ substituted):
\begin{align}
A_{\max}(\delta_t) &= \frac{L C_{\alpha,1}}{\sigma_{\delta_t}^2} \left( 1 + R + \frac{R^2}{\sigma_{\delta_t}^2} \right), \\
B_{\max}(\delta_t) &\le \varepsilon \left( A_{\max}(\delta_t) + \frac{C_{\alpha,1}}{\sigma_{\delta_t}^2} \right) = \varepsilon\frac{C_{\alpha,1}}{\sigma_{\delta_t}^2}\left[L\left( 1 + R + \frac{R^2}{\sigma_{\delta_t}^2} \right)+1\right].
\end{align}

\textbf{Step 3: Construct unified constant $C_{NN}$}

Substituting the expressions for $A_{\max}$ and $B_{\max}$ into \cref{eq:apply_cor} and factoring out $\varepsilon$:
\begin{align}
\| \overline{\mathbf{v}_{\theta,\phi}} - \mathbf{v} \| &\leq \varepsilon A_{\max}(\delta_t) \|\hat{\mathbf{x}}_k\| + \varepsilon \left( A_{\max}(\delta_t) + \frac{C_{\alpha,1}}{\sigma_{\delta_t}^2} \right) \nonumber \\
&= \varepsilon \left[ A_{\max}(\delta_t) (1 + \|\hat{\mathbf{x}}_k\|) + \frac{C_{\alpha,1}}{\sigma_{\delta_t}^2} \right].
\end{align}

To obtain the form $C(1 + \|\mathbf{x}\|)$, we use the bound:
\begin{equation}
\frac{C_{\alpha,1}}{\sigma_{\delta_t}^2} \leq \frac{C_{\alpha,1}}{\sigma_{\delta_t}^2} (1 + \|\hat{\mathbf{x}}_k\|).
\end{equation}

This gives:
\begin{equation}
\| \overline{\mathbf{v}_{\theta,\phi}} - \mathbf{v} \| \leq \varepsilon \left[ A_{\max}(\delta_t) + \frac{C_{\alpha,1}}{\sigma_{\delta_t}^2} \right] (1 + \|\hat{\mathbf{x}}_k\|).
\end{equation}

Define $C_{NN}$ as the bracketed term. Substituting the explicit definition of $A_{\max}(\delta_t)$:
\begin{align}
C_{NN} &\coloneqq A_{\max}(\delta_t) + \frac{C_{\alpha,1}}{\sigma_{\delta_t}^2} \nonumber \\
&= \frac{L C_{\alpha,1}}{\sigma_{\delta_t}^2} \left( 1 + R + \frac{R^2}{\sigma_{\delta_t}^2} \right) + \frac{C_{\alpha,1}}{\sigma_{\delta_t}^2} \nonumber \\
&= \frac{C_{\alpha,1}}{\sigma_{\delta_t}^2} \left[ L \left( 1 + R + \frac{R^2}{\sigma_{\delta_t}^2} \right) + 1 \right].
\end{align}

\textbf{Step 4: Conclusion}

Substituting the inequality from Step 3 into \cref{eq:error_norm_def}:
\begin{equation}
\|\mathcal{E}_{NN, k}\| = \Delta \tau \| \overline{\mathbf{v}_{\theta,\phi}} - \mathbf{v} \| \le \Delta \tau \cdot C_{NN} \varepsilon (1 + \|\hat{\mathbf{x}}_k\|).
\end{equation}

This completes the proof.
\end{proof}

\subsection{Global Mean Square Error Analysis}
\label{subsec:global_mse_analysis}

Having established the single-step error bounds for truncation, martingale noise, and neural network approximation, we now combine these results to derive a global convergence theorem for the mean square error. This theorem provides explicit rates for both approximation error (depending on $\varepsilon$) and discretization error (depending on $N$), conditioned on the initial state.

\begin{theorem}[Global Mean Square Error Bound for Discrete Sampling]
\label{thm:strong_convergence_precise_v2}
Under \cref{ass:p0_bounded},\cref{ass:alpha_regularity},\cref{ass:nn_approx} and\cref{assump:mixture_data} for any number of discretization steps $N \ge 1$, let the time step size $\Delta \tau = T_\delta / N \le \min\{\Delta \tau_0, 1\}$, where $\Delta \tau_0$ is the threshold satisfying the small step size assumption (\cref{ass:timestep_bound}), and $T_\delta = T - \delta$ is the total duration of the reverse simulation.

The numerical solution $\hat{\mathbf{x}}_k$ converges to the reference process $\mathbf{X}(\tau_k)$ in the $L^2$ sense, subject to the neural network approximation accuracy. For any $k \in \{0, \dots, N\}$, the global mean square error $\hat{E}_k := \mathbb{E}[\|\hat{\mathbf{x}}_k - \mathbf{X}(\tau_k)\|^2]$ satisfies the following upper bound:
\begin{equation}
\label{eq:precise_error_bound_v2}
\sup_{0 \le k \le N} \hat{E}_k \le \underbrace{\mathcal{S}(T_\delta)}_{\text{Stability factor}} \left( \underbrace{\mathcal{C}_{approx} (\hat{\mathbf{x}}_0)\varepsilon^2}_{\text{Approximation error}} + \underbrace{\mathcal{C}_{disc}(\hat{\mathbf{x}}_0) \Delta \tau}_{\text{Discretization error}} \right),
\end{equation}
where:
\begin{itemize}
    \item \textbf{Stability factor} $\mathcal{S}(T_\delta) := \frac{e^{\tilde{C}_1 T_\delta} - 1}{\tilde{C}_1}$ characterizes the cumulative effect of error over time, where the growth exponent is given by:
    \begin{gather}
        \tilde{C}_1 = C_1(L_{x,\delta}, \Delta \tau_0) + 4C_{NN}^2(1 + \Delta \tau_0)^2 \varepsilon^2, \label{eq:C_tilde_1_def} \\
        C_1 = (2 + \Delta \tau_0) + (2L_{x,\delta} + \Delta \tau_0 L_{x,\delta}^2) + (2 + \Delta \tau_0)(2L_{x,\delta} + \Delta \tau_0 L_{x,\delta}^2)\Delta \tau_0, \label{eq:C_1_def}
    \end{gather}
    with the uniform spatial Lipschitz constant $L_{x,\delta} = \sup_{t \in [\delta, T]} L_x(t)$ (\cref{cor:global_bound_Lv_uniform}) expressed as:
    \begin{gather}
        L_{x,\delta}(C_{\alpha,1}, \alpha_{\delta}, R) = C_{\alpha,1} \frac{\alpha_{\delta}}{\sigma_{\delta}^2} \left( 1 + \frac{R^2}{\sigma_{\delta}^2} \right), \quad \sigma_{\delta}^2 = 1 - \alpha_{\delta}^2, \label{eq:L_x_delta}
    \end{gather}
    and the neural network approximation constant $C_{NN}$ (\cref{thm:nn_velocity_error_bound}) given by:
    \begin{gather}
        C_{NN} = \frac{C_{\alpha,1}}{\sigma_{\delta_t}^2} \left[ L \left( 1 + R + \frac{R^2}{\sigma_{\delta_t}^2} \right) + 1 \right]. \label{eq:C_NN_def}
    \end{gather}
    Here $C_{\alpha,1}$ is the first derivative bound of $\alpha_t$ (\cref{ass:alpha_regularity}).
    
    \item \textbf{Approximation error coefficient}
    \begin{equation}
        \mathcal{C}_{approx} (\hat{\mathbf{x}}_0)= 2C_{NN}^2(1 + \Delta \tau_0)^2 (1 + 2R_0^2),
    \end{equation}
    where $R_0 := \|\mathbf{X}(0)\| + R$ depends on the initial state norm and the data manifold radius $R$ (\cref{ass:p0_bounded}), notice that $\mathbf{X}(0)=\hat{\mathbf{x}}_0$. This coefficient indicates that the $L^\infty$ training error $\varepsilon$ of the neural network is transformed into the trajectory's $L^2$ error in squared form;
    
    \item \textbf{Discretization error coefficient} $\mathcal{C}_{disc}(\hat{\mathbf{x}}_0)$ explicitly depends on the initial state:
    \begin{equation}
        \mathcal{C}_{disc}(\hat{\mathbf{x}}_0) = C_{\mathcal{M}}(1 + \|\hat{\mathbf{x}}_0\|^2) + C_{trunc}^2(\hat{\mathbf{x}}_0)(1 + \Delta \tau_0),
    \end{equation}
    where:
    \begin{itemize}
        \item $C_{\mathcal{M}} = 2(M'_{\delta_t})^2(1 + K_{\text{mom}})$ is from the variance bound of martingale difference noise, where $M'_{\delta_t}$ is the neural network velocity field's linear growth coefficient (\cref{cor:v_theta_linear_growth}), and $K_{\text{mom}} = 2e^{(3M'_{\delta_t} + 2(M'_{\delta_t})^2)T}$ is the uniform moment bound constant for discrete trajectories (\cref{lem:uniform_moment_bound_nn});
        \item $C_{trunc}(\hat{\mathbf{x}}_0)$ is the local truncation error coefficient (\cref{lem:local_truncation_bound}), which is an affine function of the initial state norm $\|\hat{\mathbf{x}}_0\|$:
        \begin{equation}
            C_{trunc}(\hat{\mathbf{x}}_0) = \frac{1}{2} \left[ \left( C_{1,\delta} + L_{x,\delta} M_\delta \right) (\|\hat{\mathbf{x}}_0\| + R) + \left( C_{0,\delta} + L_{x,\delta} M_\delta \right) \right],
        \end{equation}
        where $C_{1,\delta}, C_{0,\delta}$ are the temporal derivative growth coefficients (\cref{cor:uniform_bound_delta}), with explicit expressions:
        \begin{align}
            C_{1,\delta}(C_{\alpha,1}, C_{\alpha,2}, \alpha_{\delta}, R) 
            &= \frac{C_{\alpha,2} \alpha_{\delta}}{\sigma_{\delta}^2} + \frac{2\alpha_{\delta}^2 C_{\alpha,1}^2}{\sigma_{\delta}^4} + \frac{C_{\alpha,1}^2}{\sigma_{\delta}^2} \left( \frac{(1+\alpha_{\delta}^2) R^2}{\sigma_{\delta}^4} + 1 \right), \\
            C_{0,\delta}(C_{\alpha,1}, C_{\alpha,2}, \alpha_{\delta}, R) 
            &= \frac{C_{\alpha,2} R}{\sigma_{\delta}^2} + \frac{2\alpha_{\delta} C_{\alpha,1}^2 R}{\sigma_{\delta}^4} + \frac{2\alpha_{\delta} C_{\alpha,1}^2 R^3}{\sigma_{\delta}^6},
        \end{align}
        and $M_\delta$ is the linear growth coefficient of the true velocity field (\cref{prop:v_bound}):
        \begin{equation}
            M_\delta(C_{\alpha,1}, \sigma_{\delta}, R) = C_{\alpha,1}\frac{1}{\sigma_{\delta}^2} \left( 1 + R + \frac{R^2}{\sigma_{\delta}^2} \right).
        \end{equation}
        In summary, all intermediate constants can be reduced to the basic constants $C_{\alpha,1}, C_{\alpha,2}$ (\cref{ass:alpha_regularity}), $R$ (\cref{ass:p0_bounded}), and the truncation parameter $\alpha_{\delta}$.
    \end{itemize}
\end{itemize}
\end{theorem}

\begin{proof}
The proof utilizes the error decomposition established in \cref{subsec:error_decomposition} and proceeds through six steps: establishing the error recursion, analyzing the deterministic components via two applications of Young's inequality, estimating the stochastic noise term, applying the discrete Grönwall inequality, and finally decomposing the global error into approximation and discretization components.

The key prerequisites from previous results are:

\begin{itemize}
    \item \textbf{Spatial Lipschitz continuity of velocity field:} By \cref{thm:v_Lipschitz_improved,cor:global_bound_Lv_uniform}, the velocity field $\tilde{\mathbf{v}}$ satisfies a uniform Lipschitz condition with respect to the spatial variable on the time interval $[\delta, T]$, with constant $L_{x,\delta} = \sup_{t \in [\delta, T]} L_x(t)$ and explicit expression:
    \begin{equation}
        L_{x,\delta}(C_{\alpha,1}, \alpha_{\delta}, R) = C_{\alpha,1} \frac{\alpha_{\delta}}{\sigma_{\delta}^2} \left( 1 + \frac{R^2}{\sigma_{\delta}^2} \right),
    \end{equation}
    where $\sigma_{\delta}^2 = 1 - \alpha_{\delta}^2$, $C_{\alpha,1}$ is the first derivative bound of $\alpha_t$ from \cref{ass:alpha_regularity}, and $R$ is the data manifold radius from \cref{ass:p0_bounded};
    
    \item \textbf{Linear growth of velocity field:} By \cref{prop:v_bound}, the velocity field satisfies a linear growth condition with growth constant $M_\delta$ (for the true velocity field) or $M'_{\delta_t}$ (for the neural network velocity field), i.e., $\|\tilde{\mathbf{v}}(\mathbf{x}, t)\| \le M_\delta(1 + \|\mathbf{x}\|)$. The linear growth coefficient of the true velocity field is:
    \begin{equation}
        M_\delta(C_{\alpha,1}, \sigma_{\delta}, R) = C_{\alpha,1}\frac{1}{\sigma_{\delta}^2} \left( 1 + R + \frac{R^2}{\sigma_{\delta}^2} \right),
    \end{equation}
    with the same parameter dependencies as above;
    
    \item \textbf{Uniform boundedness of reference trajectory:} By \cref{thm:global_bound_geometric}, the reference process $\mathbf{X}(\tau)$ on the reverse time interval $[0, T_\delta]$ satisfies $\|\mathbf{X}(\tau)\| \le \|\mathbf{X}(0)\| + R$, where $R$ is the data manifold radius;
    
    \item \textbf{Local truncation error bound:} By \cref{lem:local_truncation_bound}, the local truncation error of the Euler scheme satisfies $\|\mathcal{T}_k\| \le C_{trunc}(\Delta \tau)^2$, where
    \begin{equation}
        C_{trunc} = \frac{1}{2} \left[ \left( C_{1,\delta} + L_{x,\delta} M_\delta \right) (\|\mathbf{X}(0)\| + R) + \left( C_{0,\delta} + L_{x,\delta} M_\delta \right) \right],
    \end{equation}
    where $C_{1,\delta}, C_{0,\delta}$ are the temporal derivative growth coefficients from \cref{cor:uniform_bound_delta}, with explicit expressions given in the theorem statement above. Note that $C_{trunc}$ explicitly depends on the initial state $\|\mathbf{X}(0)\|$ and the basic constants $C_{\alpha,1}, C_{\alpha,2}, \alpha_{\delta}, R$;
    
    \item \textbf{Neural network single-step error bound:} By \cref{thm:nn_velocity_error_bound}, the single-step approximation error of the neural network velocity field satisfies $\|\mathcal{E}_{NN,k}\| \le C_{NN}\varepsilon \Delta \tau(1+\|\hat{\mathbf{x}}_k\|)$, where $C_{NN}$ is the neural network approximation constant.
\end{itemize}

\textbf{Step 1: Establishing the error recursion}

Define the global error $\hat{\mathbf{e}}_k := \hat{\mathbf{x}}_k - \mathbf{X}(\tau_k)$, where $\hat{\mathbf{x}}_k$ is the numerical solution and $\mathbf{X}(\tau_k)$ is the reference ODE trajectory. The numerical scheme is:
\begin{equation}
    \hat{\mathbf{x}}_{k+1} = \hat{\mathbf{x}}_k + \Delta \tau \tilde{\mathbf{v}}_\theta^{\hat{l}_k}(\hat{\mathbf{x}}_k, \tau_k),
\end{equation}
where $\hat{l}_k$ is the category index sampled according to the neural network predicted posterior probabilities $\hat{a}_\phi^i(\hat{\mathbf{x}}_k, \tau_k)$. The reference process satisfies:
\begin{equation}
    \mathbf{X}(\tau_{k+1}) = \mathbf{X}(\tau_k) + \Delta \tau \tilde{\mathbf{v}}(\mathbf{X}(\tau_k), \tau_k) + \mathcal{T}_k,
\end{equation}
where $\mathcal{T}_k$ is the local truncation error.

Subtracting these two equations and introducing the average velocity field $\overline{\mathbf{v}}_\theta(\mathbf{x}, \tau) := \sum_{i=1}^L \hat{a}_\phi^i(\mathbf{x}, \tau) \tilde{\mathbf{v}}_\theta^i(\mathbf{x}, \tau)$, we obtain the error evolution:
\begin{align}
    \hat{\mathbf{e}}_{k+1} &= \hat{\mathbf{x}}_{k+1} - \mathbf{X}(\tau_{k+1}) \nonumber \\
    &= \hat{\mathbf{e}}_k + \Delta \tau [\tilde{\mathbf{v}}_\theta^{\hat{l}_k}(\hat{\mathbf{x}}_k, \tau_k) - \tilde{\mathbf{v}}(\mathbf{X}(\tau_k), \tau_k)] - \mathcal{T}_k \nonumber \\
    &= \hat{\mathbf{e}}_k + \Delta \tau [\tilde{\mathbf{v}}(\hat{\mathbf{x}}_k, \tau_k) - \tilde{\mathbf{v}}(\mathbf{X}(\tau_k), \tau_k)] \nonumber \\
    &\quad + \Delta \tau [\overline{\mathbf{v}}_\theta(\hat{\mathbf{x}}_k, \tau_k) - \tilde{\mathbf{v}}(\hat{\mathbf{x}}_k, \tau_k)] \nonumber \\
    &\quad + \Delta \tau [\tilde{\mathbf{v}}_\theta^{\hat{l}_k}(\hat{\mathbf{x}}_k, \tau_k) - \overline{\mathbf{v}}_\theta(\hat{\mathbf{x}}_k, \tau_k)] - \mathcal{T}_k.
\end{align}

Define the following components:
\begin{itemize}
    \item Drift error increment: $\Delta \hat{\mathbf{v}}_k := \tilde{\mathbf{v}}(\hat{\mathbf{x}}_k, \tau_k) - \tilde{\mathbf{v}}(\mathbf{X}(\tau_k), \tau_k)$;
    \item Neural network single-step error: $\mathcal{E}_{NN,k} := \Delta \tau [\overline{\mathbf{v}}_\theta(\hat{\mathbf{x}}_k, \tau_k) - \tilde{\mathbf{v}}(\hat{\mathbf{x}}_k, \tau_k)]$;
    \item Sampling noise (martingale difference): $\hat{\mathcal{M}}_k := \Delta \tau [\tilde{\mathbf{v}}_\theta^{\hat{l}_k}(\hat{\mathbf{x}}_k, \tau_k) - \overline{\mathbf{v}}_\theta(\hat{\mathbf{x}}_k, \tau_k)]$.
\end{itemize}

The single-step global error evolution becomes:
\begin{equation}
    \hat{\mathbf{e}}_{k+1} = \hat{\mathbf{e}}_k + \Delta \tau \Delta \hat{\mathbf{v}}_k + \mathcal{E}_{NN, k} + \hat{\mathcal{M}}_k - \mathcal{T}_k.
\end{equation}

Define the deterministic part $\mathbf{D}_k := \hat{\mathbf{e}}_k + \Delta \tau \Delta \hat{\mathbf{v}}_k + \mathcal{E}_{NN, k} - \mathcal{T}_k$.

Using the martingale difference property ($\mathbb{E}[\hat{\mathcal{M}}_k \mid \mathcal{F}_k] = \mathbf{0}$, where $\mathcal{F}_k = \sigma(\hat{\mathbf{x}}_0, \hat{l}_0, \dots, \hat{\mathbf{x}}_k)$ is the $\sigma$-algebra of historical information), and noting that $\mathbf{D}_k$ is $\mathcal{F}_k$-measurable, the second moment decomposes as:
\begin{equation}
    \hat{E}_{k+1} = \mathbb{E}[\|\hat{\mathbf{e}}_{k+1}\|^2] = \mathbb{E}[\|\mathbf{D}_k + \hat{\mathcal{M}}_k\|^2] = \mathbb{E}[\|\mathbf{D}_k\|^2] + \mathbb{E}[\|\hat{\mathcal{M}}_k\|^2].
\end{equation}

\textbf{Step 2: Refined estimation of the deterministic component}

For $\|\mathbf{D}_k\|^2 = \| (\hat{\mathbf{e}}_k + \Delta \tau \Delta \hat{\mathbf{v}}_k + \mathcal{E}_{NN, k}) - \mathcal{T}_k \|^2$, we need to carefully handle two types of higher-order error terms: the local truncation error $\mathcal{T}_k$ (of order $O(\Delta \tau^2)$) and the neural network approximation error $\mathcal{E}_{NN, k}$ (of order $O(\varepsilon\Delta \tau)$). To ensure that the growth coefficient contains only first-order terms in the time interval, we apply Young's inequality twice with the unified weight parameter $\eta = \Delta \tau$ for separation.

\textbf{Substep 2.1: Two applications of Young's inequality}

Recall Young's inequality: for any vectors $\mathbf{a}, \mathbf{b}$ and $\eta > 0$,
\begin{equation}
    \|\mathbf{a}+\mathbf{b}\|^2 \le (1+\eta)\|\mathbf{a}\|^2 + \left(1+\frac{1}{\eta}\right)\|\mathbf{b}\|^2.
\end{equation}

\textbf{First application:} Reorganize $\mathbf{D}_k$ as $[(\hat{\mathbf{e}}_k + \Delta \tau \Delta \hat{\mathbf{v}}_k + \mathcal{E}_{NN, k})] - \mathcal{T}_k$, taking weight parameter $\eta_1 = \Delta \tau$ to separate the truncation error:
\begin{equation}
    \|\mathbf{D}_k\|^2 \le (1 + \Delta \tau) \|\hat{\mathbf{e}}_k + \Delta \tau \Delta \hat{\mathbf{v}}_k + \mathcal{E}_{NN, k}\|^2 + \left(1 + \frac{1}{\Delta \tau}\right) \|\mathcal{T}_k\|^2.
\end{equation}

\textbf{Second application:} Continue to decompose the first term as $(\hat{\mathbf{e}}_k + \Delta \tau \Delta \hat{\mathbf{v}}_k) + \mathcal{E}_{NN, k}$, again taking weight parameter $\eta_2 = \Delta \tau$:
\begin{equation}
    \|\hat{\mathbf{e}}_k + \Delta \tau \Delta \hat{\mathbf{v}}_k + \mathcal{E}_{NN, k}\|^2 \le (1 + \Delta \tau) \|\hat{\mathbf{e}}_k + \Delta \tau \Delta \hat{\mathbf{v}}_k\|^2 + \left(1 + \frac{1}{\Delta \tau}\right) \|\mathcal{E}_{NN, k}\|^2.
\end{equation}

Combining these yields:
\begin{align}
    \|\mathbf{D}_k\|^2 &\le (1 + \Delta \tau)^2 \|\hat{\mathbf{e}}_k + \Delta \tau \Delta \hat{\mathbf{v}}_k\|^2 \nonumber \\
    &\quad + (1 + \Delta \tau)\left(1 + \frac{1}{\Delta \tau}\right) \|\mathcal{E}_{NN, k}\|^2 + \left(1 + \frac{1}{\Delta \tau}\right) \|\mathcal{T}_k\|^2.
\end{align}

Under the small step size assumption $\Delta \tau \le \Delta \tau_0$, using $(\Delta \tau)^2 \le \Delta \tau_0 \cdot \Delta \tau$, note that:
\begin{align}
    (1 + \Delta \tau)^2 &= 1 + 2\Delta \tau + (\Delta \tau)^2 \le 1 + (2 + \Delta \tau_0)\Delta \tau, \\
    (1 + \Delta \tau)\left(1 + \frac{1}{\Delta \tau}\right) &= 2 + \Delta \tau + \frac{1}{\Delta \tau}.
\end{align}

Let $C_\tau := 2 + \Delta \tau_0$, then:
\begin{align}
\label{eq:D_k_simplified_v2}
    \|\mathbf{D}_k\|^2 &\le (1 + C_\tau \Delta \tau) \|\hat{\mathbf{e}}_k + \Delta \tau \Delta \hat{\mathbf{v}}_k\|^2 \nonumber \\
    &\quad + \left(2 + \Delta \tau + \frac{1}{\Delta \tau}\right) \|\mathcal{E}_{NN, k}\|^2 + \left(1 + \frac{1}{\Delta \tau}\right) \|\mathcal{T}_k\|^2.
\end{align}

\textbf{Substep 2.2: Estimation of the drift evolution term}

Expanding the squared norm of the first term:
\begin{equation}
    \|\hat{\mathbf{e}}_k + \Delta \tau \Delta \hat{\mathbf{v}}_k\|^2 = \|\hat{\mathbf{e}}_k\|^2 + 2\Delta \tau \hat{\mathbf{e}}_k \cdot \Delta \hat{\mathbf{v}}_k + (\Delta \tau)^2 \|\Delta \hat{\mathbf{v}}_k\|^2.
\end{equation}

Using the Lipschitz continuity of the velocity field with respect to the spatial variable (with constant $L_{x,\delta}$), from $\Delta \hat{\mathbf{v}}_k = \tilde{\mathbf{v}}(\hat{\mathbf{x}}_k, \tau_k) - \tilde{\mathbf{v}}(\mathbf{X}(\tau_k), \tau_k)$, we have:
\begin{equation}
    \|\Delta \hat{\mathbf{v}}_k\| \le L_{x,\delta} \|\hat{\mathbf{x}}_k - \mathbf{X}(\tau_k)\| = L_{x,\delta} \|\hat{\mathbf{e}}_k\|.
\end{equation}

Combining with the Cauchy-Schwarz inequality and the small step size assumption:
\begin{align}
    \|\hat{\mathbf{e}}_k + \Delta \tau \Delta \hat{\mathbf{v}}_k\|^2 &\le \|\hat{\mathbf{e}}_k\|^2 + 2\Delta \tau L_{x,\delta} \|\hat{\mathbf{e}}_k\|^2 + (\Delta \tau)^2 L_{x,\delta}^2 \|\hat{\mathbf{e}}_k\|^2 \nonumber \\
    &\le \|\hat{\mathbf{e}}_k\|^2 + 2\Delta \tau L_{x,\delta} \|\hat{\mathbf{e}}_k\|^2 + \Delta \tau_0 \Delta \tau L_{x,\delta}^2 \|\hat{\mathbf{e}}_k\|^2 \nonumber \\
    &= \left( 1 + (2L_{x,\delta} + \Delta \tau_0 L_{x,\delta}^2) \Delta \tau \right) \|\hat{\mathbf{e}}_k\|^2.
\end{align}

Let $C_{drift} := 2L_{x,\delta} + \Delta \tau_0 L_{x,\delta}^2$. Combining with the coefficient $(1 + C_\tau \Delta \tau)$:
\begin{align}
    (1 + C_\tau \Delta \tau) \|\hat{\mathbf{e}}_k + \Delta \tau \Delta \hat{\mathbf{v}}_k\|^2 
    &\le (1 + C_\tau \Delta \tau)(1 + C_{drift} \Delta \tau) \|\hat{\mathbf{e}}_k\|^2 \nonumber \\
    &\le \left( 1 + (C_\tau + C_{drift} + C_\tau C_{drift} \Delta \tau_0) \Delta \tau \right) \|\hat{\mathbf{e}}_k\|^2.
\end{align}

Define the combined drift coefficient $C_1 := C_\tau + C_{drift} + C_\tau C_{drift} \Delta \tau_0$. Therefore:
\begin{equation}
\label{eq:drift_bound_v2}
    (1 + C_\tau \Delta \tau) \|\hat{\mathbf{e}}_k + \Delta \tau \Delta \hat{\mathbf{v}}_k\|^2 \le (1 + C_1 \Delta \tau) \|\hat{\mathbf{e}}_k\|^2.
\end{equation}

\textbf{Substep 2.3: Estimation of the neural network error term}

Based on the approximation error bound of the neural network velocity field (satisfying $\|\mathcal{E}_{NN, k}\| \le C_{NN}\varepsilon \Delta \tau (1 + \|\hat{\mathbf{x}}_k\|)$, where $C_{NN}$ is the neural network approximation constant), using $\|\hat{\mathbf{x}}_k\| \le \|\hat{\mathbf{e}}_k\| + \|\mathbf{X}(\tau_k)\|$ and the boundedness of the reference trajectory $\|\mathbf{X}(\tau_k)\| \le \|\mathbf{X}(0)\| + R$ (let $R_0 := \|\mathbf{X}(0)\| + R$), and applying the inequality $(1+a)^2 \le 2(1+a^2)$:
\begin{align}
    \|\mathcal{E}_{NN, k}\|^2 &\le C_{NN}^2\varepsilon^2 (\Delta \tau)^2 (1 + \|\hat{\mathbf{x}}_k\|)^2 \nonumber \\
    &\le 2C_{NN}^2\varepsilon^2 (\Delta \tau)^2 (1 + \|\hat{\mathbf{x}}_k\|^2) \nonumber \\
    &\le 2C_{NN}^2\varepsilon^2 (\Delta \tau)^2 \left[ 1 + 2\|\hat{\mathbf{e}}_k\|^2 + 2R_0^2 \right],
\end{align}
where $R_0^2 = (\|\mathbf{X}(0)\| + R)^2$.

Therefore, the coefficient term satisfies:
\begin{align}
    \left(2 + \Delta \tau + \frac{1}{\Delta \tau}\right) \|\mathcal{E}_{NN, k}\|^2 
    &\le \frac{(1 + \Delta \tau)^2}{\Delta \tau} \cdot 2C_{NN}^2\varepsilon^2 (\Delta \tau)^2 \left[ 1 + 2\|\hat{\mathbf{e}}_k\|^2 + 2R_0^2 \right] \nonumber \\
    &= 2C_{NN}^2\varepsilon^2 \Delta \tau (1 + \Delta \tau)^2 \left[ 1 + 2\|\hat{\mathbf{e}}_k\|^2 + 2R_0^2 \right] \nonumber \\
    &\le 2C_{NN}^2\varepsilon^2 (1 + \Delta \tau_0)^2 \Delta \tau \left[ 1 + 2\|\hat{\mathbf{e}}_k\|^2 + 2R_0^2 \right].
\end{align}

Let $C_{NN}' := 2C_{NN}^2(1 + \Delta \tau_0)^2$, then:
\begin{equation}
\label{eq:nn_bound_v2}
    \left(2 + \Delta \tau + \frac{1}{\Delta \tau}\right) \|\mathcal{E}_{NN, k}\|^2 \le C_{NN}' \varepsilon^2\Delta \tau (1 + 2\|\hat{\mathbf{e}}_k\|^2 + 2R_0^2).
\end{equation}

\textbf{Substep 2.4: Estimation of the truncation error term}

By \cref{lem:local_truncation_bound}, the local truncation error satisfies $\|\mathcal{T}_k\| \le C_{trunc}(\hat{\mathbf{x}}_0) (\Delta \tau)^2$, where
\begin{equation}
    C_{trunc}(\hat{\mathbf{x}}_0) = \frac{1}{2} \left[ \left( C_{1,\delta} + L_{x,\delta} M_\delta \right) (\|\hat{\mathbf{x}}_0\| + R) + \left( C_{0,\delta} + L_{x,\delta} M_\delta \right) \right].
\end{equation}
Here $C_{1,\delta}, C_{0,\delta}$ are the temporal derivative growth coefficients from \cref{cor:uniform_bound_delta}, $L_{x,\delta}$ is the spatial Lipschitz constant defined in \cref{cor:global_bound_Lv_uniform}, and $M_\delta$ is the linear growth coefficient of the true velocity field from \cref{prop:v_bound}. Note that $C_{trunc}(\hat{\mathbf{x}}_0)$ explicitly depends on the initial state $\|\hat{\mathbf{x}}_0\|$ (here using $\|\mathbf{X}(0)\| = \|\hat{\mathbf{x}}_0\|$) and the data manifold radius $R$.

Then:
\begin{align}
    \left(1 + \frac{1}{\Delta \tau}\right) \|\mathcal{T}_k\|^2 &\le \left( \frac{1 + \Delta \tau}{\Delta \tau} \right) C_{trunc}^2(\hat{\mathbf{x}}_0) (\Delta \tau)^4 \nonumber \\
    &= C_{trunc}^2(\hat{\mathbf{x}}_0) (1 + \Delta \tau) (\Delta \tau)^3 \nonumber \\
    &\le C_{trunc}^2(\hat{\mathbf{x}}_0) (1 + \Delta \tau_0) (\Delta \tau)^3.
\end{align}

Let $C_{trunc}'(\hat{\mathbf{x}}_0) := [C_{trunc}(\hat{\mathbf{x}}_0)]^2(1 + \Delta \tau_0)$ denote the scaled truncation error coefficient (which depends on the initial state $\hat{\mathbf{x}}_0$), then:
\begin{equation}
\label{eq:trunc_bound_v2}
    \left(1 + \frac{1}{\Delta \tau}\right) \|\mathcal{T}_k\|^2 \le C_{trunc}'(\hat{\mathbf{x}}_0) (\Delta \tau)^3.
\end{equation}

\textbf{Substep 2.5: Combining all estimates}

Combining \cref{eq:drift_bound_v2,eq:nn_bound_v2,eq:trunc_bound_v2} and taking expectations over the entire inequality, we obtain the bound for the deterministic part:
\begin{align}
\label{eq:D_k_final_v2}
    \mathbb{E}[\|\mathbf{D}_k\|^2] 
    &\le (1 + C_1 \Delta \tau) \hat{E}_k + C_{NN}' \varepsilon^2 \Delta \tau (1 + 2\hat{E}_k + 2R_0^2) + C_{trunc}'(\hat{\mathbf{x}}_0) (\Delta \tau)^3 \nonumber \\
    &= (1 + C_1 \Delta \tau) \hat{E}_k + 2C_{NN}' \varepsilon^2 \Delta \tau \hat{E}_k + C_{NN}' \varepsilon^2 \Delta \tau (1 + 2R_0^2) + C_{trunc}'(\hat{\mathbf{x}}_0) (\Delta \tau)^3 \nonumber \\
    &= \left( 1 + (C_1 + 2C_{NN}' \varepsilon^2) \Delta \tau \right) \hat{E}_k + C_{NN}' \varepsilon^2 \Delta \tau (1 + 2R_0^2) + C_{trunc}'(\hat{\mathbf{x}}_0) (\Delta \tau)^3,
\end{align}
where the growth coefficient $1 + (C_1 + 2C_{NN}' \varepsilon^2) \Delta \tau$ contains only first-order terms in the time interval, satisfying the standard conditions for applying the discrete Grönwall inequality.

\textbf{Step 3: Estimation of the stochastic noise term}

By \cref{lem:martingale_bound_nn}, the martingale difference term $\hat{\mathcal{M}}_k$ satisfies the following second moment upper bound:
\begin{equation}
\label{eq:martingale_bound_v2}
    \mathbb{E}[\|\hat{\mathcal{M}}_k\|^2] \le C_{\mathcal{M}} (\Delta \tau)^2 (1 + \|\hat{\mathbf{x}}_0\|^2),
\end{equation}
where $C_{\mathcal{M}} = 2(M'_{\delta_t})^2(1 + K_{\text{mom}})$ and $K_{\text{mom}}$ is the uniform second moment bound constant for numerical trajectories. Adding the deterministic evolution estimate \cref{eq:D_k_final_v2} and the noise term estimate \cref{eq:martingale_bound_v2}, we obtain the total error evolution inequality:
\begin{align}
    \hat{E}_{k+1} &\le \left( 1 + (C_1 + 2C_{NN}' \varepsilon^2) \Delta \tau \right) \hat{E}_k \nonumber \\
    &\quad + C_{NN}' \varepsilon^2 \Delta \tau (1 + 2R_0^2) + C_{\mathcal{M}} (\Delta \tau)^2 (1 + \|\hat{\mathbf{x}}_0\|^2) + C_{trunc}'(\hat{\mathbf{x}}_0) (\Delta \tau)^3.
\end{align}

\textbf{Step 4: Standardization of the recursive inequality}

To facilitate application of the discrete Grönwall inequality, define the combined growth coefficient $\tilde{C}_1$ and error component coefficients $\Gamma_1, \Gamma_2, \Gamma_3$:
\begin{align}
    \tilde{C}_1 &:= C_1 + 2C_{NN}' \varepsilon^2, \label{eq:Lambda_coeff_v2} \\
    \Gamma_1 &:= C_{NN}' \varepsilon^2 (1 + 2R_0^2), \label{eq:Gamma_1_v2} \\
    \Gamma_2 &:= C_{\mathcal{M}} (1 + \|\hat{\mathbf{x}}_0\|^2), \label{eq:Gamma_2_v2} \\
    \Gamma_3(\hat{\mathbf{x}}_0) &:= C_{trunc}'(\hat{\mathbf{x}}_0). \label{eq:Gamma_3_v2}
\end{align}
Note that $\Gamma_3$ explicitly depends on the initial state through the truncation coefficient $C_{trunc}'(\hat{\mathbf{x}}_0)$. The error recursion can then be written in the following standard form:
\begin{equation}
\label{eq:recurrence_v2}
    \hat{E}_{k+1} \le (1 + \tilde{C}_1 \Delta \tau) \hat{E}_k + \Gamma_1 \Delta \tau + \Gamma_2 (\Delta \tau)^2 + \Gamma_3(\hat{\mathbf{x}}_0) (\Delta \tau)^3.
\end{equation}

\textbf{Step 5: Applying the discrete Grönwall lemma}

Assume $\hat{E}_0 = 0$ (exact initial state). The recursive inequality \cref{eq:recurrence_v2} has the standard form:
\begin{equation}
    a_{k+1} \le \lambda a_k + c,
\end{equation}
where $a_k = \hat{E}_k$, $\lambda = 1 + \tilde{C}_1 \Delta \tau$, and $c(\hat{\mathbf{x}}_0) = \Gamma_1 \Delta \tau + \Gamma_2 (\Delta \tau)^2 + \Gamma_3(\hat{\mathbf{x}}_0) (\Delta \tau)^3$ (note the dependence on $\hat{\mathbf{x}}_0$ through the third term).

\textbf{Discrete Grönwall lemma:} For a sequence satisfying $a_{k+1} \le \lambda a_k + c$ with $a_0 = 0$:
\begin{equation}
    a_k \le c \sum_{j=0}^{k-1} \lambda^j = c \frac{\lambda^k - 1}{\lambda - 1}.
\end{equation}

Applying this to $\hat{E}_k$, for $k \le N$:
\begin{align}
    \hat{E}_k &\le \left[ \Gamma_1 \Delta \tau + \Gamma_2 (\Delta \tau)^2 + \Gamma_3(\hat{\mathbf{x}}_0) (\Delta \tau)^3 \right] \frac{(1 + \tilde{C}_1 \Delta \tau)^k - 1}{\tilde{C}_1 \Delta \tau} \nonumber \\
    &= \left[ \Gamma_1 + \Gamma_2 \Delta \tau + \Gamma_3(\hat{\mathbf{x}}_0) (\Delta \tau)^2 \right] \frac{(1 + \tilde{C}_1 \Delta \tau)^k - 1}{\tilde{C}_1}.
\end{align}

Using the basic inequality $(1 + x)^k \le e^{kx}$ (valid for all $x > 0, k \ge 0$), and noting that $k \Delta \tau \le N \Delta \tau = T_\delta$:
\begin{equation}
    (1 + \tilde{C}_1 \Delta \tau)^k \le e^{k \tilde{C}_1 \Delta \tau} \le e^{\tilde{C}_1 T_\delta}.
\end{equation}

Therefore, the upper bound for the global mean square error on the time interval $[0, T_\delta]$ is:
\begin{equation}
\label{eq:global_error_v2}
    \sup_{0 \le k \le N} \hat{E}_k \le \frac{e^{\tilde{C}_1 T_\delta} - 1}{\tilde{C}_1} \left[ \Gamma_1 + \Gamma_2 \Delta \tau + \Gamma_3(\hat{\mathbf{x}}_0) (\Delta \tau)^2 \right].
\end{equation}

\textbf{Step 6: Structured decomposition of the error}

Define the stability factor $\mathcal{S}(T_\delta) := \frac{e^{\tilde{C}_1 T_\delta} - 1}{\tilde{C}_1}$. Reorganize the global error into two components:

\textbf{(1) Approximation error component}:
\begin{equation}
    E_{approx} := \mathcal{S}(T_\delta) \cdot \Gamma_1 = \mathcal{S}(T_\delta) \cdot 2C_{NN}^2(1 + \Delta \tau_0)^2 (1 + 2R_0^2) \varepsilon^2.
\end{equation}

\textbf{(2) Discretization error component}: Reflecting the combined impact of time discretization and sampling noise. Using $\Delta \tau \le 1$, we have:
\begin{align}
    E_{disc}(\hat{\mathbf{x}}_0) &:= \mathcal{S}(T_\delta) [\Gamma_2 \Delta \tau + \Gamma_3(\hat{\mathbf{x}}_0) (\Delta \tau)^2] \nonumber \\
    &\le \mathcal{S}(T_\delta) \left[ \Gamma_2 + \Gamma_3(\hat{\mathbf{x}}_0) \right] \Delta \tau.
\end{align}

Define $\mathcal{C}_{approx} := 2C_{NN}^2(1 + \Delta \tau_0)^2 (1 + 2R_0^2)$ and $\mathcal{C}_{disc}(\hat{\mathbf{x}}_0) := C_{\mathcal{M}}(1 + \|\hat{\mathbf{x}}_0\|^2) + C_{trunc}^2(\hat{\mathbf{x}}_0)(1 + \Delta \tau_0)$. The final global error bound is:
\begin{equation}
\label{eq:final_error_bound}
\sup_{0 \le k \le N} \hat{E}_k \le \mathcal{S}(T_\delta) \left[ \mathcal{C}_{approx} \varepsilon^2 + \mathcal{C}_{disc}(\hat{\mathbf{x}}_0) \Delta \tau \right].
\end{equation}

This establishes the theorem's conclusion. Notably, the growth exponent $\tilde{C}_1$ in the stability factor $\mathcal{S}(T_\delta)$ is independent of the initial state $\|\hat{\mathbf{x}}_0\|$; the dependence on the initial state is confined to both the approximation coefficient $\mathcal{C}_{approx}$ and the discretization coefficient $\mathcal{C}_{disc}(\hat{\mathbf{x}}_0)$. This embodies the essential separation of error propagation mechanisms: the system's intrinsic evolution stability and the trajectory's initial scale jointly determine the final error.
\end{proof}

\section{Proof of $W_2$ Distance Convergence for the Algorithm}
\subsection{$\delta t$ Truncation Analysis}

\begin{theorem}[$W_2$ Distance Upper Bound between $p_t$ and $p_0$]
\label{thm:w2_bound}
Under Assumptions \ref{ass:p0_bounded} and Definition \ref{def:vp_sde}, for any $t \in [0, T]$, the Wasserstein-2 distance between the marginal distribution $p_t(\mathbf{x})$ and the initial distribution $p_0(\mathbf{x})$ satisfies the following upper bound:
\begin{equation}
    W_2^2(p_t, p_0) \le (1 - \alpha_t)^2 R^2 + (1 - \alpha_t^2) d,
\end{equation}
where $d$ is the data dimension, $\alpha_t$ is the coefficient given in Definition \ref{def:vp_sde}, and $R$ is the radius of the support set.
\end{theorem}

\begin{proof}
According to the definition of Wasserstein-2 distance,
\[
W_2^2(p_t, p_0) = \inf_{\gamma \in \Pi(p_t, p_0)} \mathbb{E}_{(\mathbf{x}, \mathbf{y}) \sim \gamma} \left[ \|\mathbf{x} - \mathbf{y}\|_2^2 \right],
\]
where $\Pi(p_t, p_0)$ denotes the set of all joint distributions (couplings) with marginal distributions $p_t$ and $p_0$.

To obtain an upper bound, we only need to construct a specific coupling. The diffusion process itself provides a natural coupling mechanism. Consider the joint distribution $q(\mathbf{x}_t, \mathbf{x}_0) = p_{t}(\mathbf{x}_t | \mathbf{x}_0) p_0(\mathbf{x}_0)$. Obviously, this joint distribution has marginal distributions $p_t$ and $p_0$, respectively. Therefore, we have:
\begin{equation}
    W_2^2(p_t, p_0) \le \mathbb{E}_{(\mathbf{x}_t, \mathbf{x}_0) \sim q} \left[ \|\mathbf{x}_t - \mathbf{x}_0\|_2^2 \right].
\end{equation}
According to the transition kernel $p_{t}(\mathbf{x}_t | \mathbf{x}_0) = \mathcal{N}(\mathbf{x}_t; \alpha_t \mathbf{x}_0, (1-\alpha_t^2)\mathbf{I})$ in Definition \ref{def:vp_sde}, we can use the reparameterization trick to express the random variable $\mathbf{x}_t$ explicitly as:
\begin{equation}
    \mathbf{x}_t = \alpha_t \mathbf{x}_0 + \sqrt{1 - \alpha_t^2} \,\mathbf{z},
\end{equation}
where $\mathbf{x}_0 \sim p_0(\mathbf{x})$, and $\mathbf{z} \sim \mathcal{N}(\mathbf{0}, \mathbf{I})$ is standard Gaussian noise independent of $\mathbf{x}_0$.

Substituting this expression into the distance formula, examining the squared error term:
\begin{align}
    \|\mathbf{x}_t - \mathbf{x}_0\|_2^2
    &= \| (\alpha_t \mathbf{x}_0 + \sqrt{1 - \alpha_t^2} \,\mathbf{z}) - \mathbf{x}_0 \|_2^2 \nonumber \\
    &= \| (\alpha_t - 1) \mathbf{x}_0 + \sqrt{1 - \alpha_t^2} \,\mathbf{z} \|_2^2 \nonumber \\
    &= (1 - \alpha_t)^2 \|\mathbf{x}_0\|_2^2 + (1 - \alpha_t^2) \|\mathbf{z}\|_2^2 + 2(\alpha_t - 1)\sqrt{1 - \alpha_t^2} \langle \mathbf{x}_0, \mathbf{z} \rangle.
\end{align}
Now taking expectation of the above expression. Using the linearity of expectation and the independence of $\mathbf{x}_0$ and $\mathbf{z}$:

1. \textbf{Cross term}: Since $\mathbf{x}_0$ and $\mathbf{z}$ are independent, and $\mathbb{E}[\mathbf{z}] = \mathbf{0}$, thus
   \[
   \mathbb{E} [\langle \mathbf{x}_0, \mathbf{z} \rangle] = \langle \mathbb{E}[\mathbf{x}_0], \mathbb{E}[\mathbf{z}] \rangle = 0.
   \]

2. \textbf{Noise term}: For standard Gaussian vector $\mathbf{z} \in \mathbb{R}^d$, the expectation of its squared norm equals the dimension:
   \[
   \mathbb{E} [\|\mathbf{z}\|_2^2] = d.
   \]

3. \textbf{Data term}: According to Assumption \ref{ass:p0_bounded}, $\text{supp}(p_0) \subseteq \{\mathbf{x} \in \mathbb{R}^d : \|\mathbf{x}\|_2 \le R\}$, therefore for any $\mathbf{x}_0 \sim p_0$, we have $\|\mathbf{x}_0\|_2 \le R$, thus
   \[
   \mathbb{E} [\|\mathbf{x}_0\|_2^2] \le R^2.
   \]

Substituting the above results into the expectation expression, we get:
\begin{equation}
    \mathbb{E} \left[ \|\mathbf{x}_t - \mathbf{x}_0\|_2^2 \right]
    = (1 - \alpha_t)^2 \mathbb{E}[\|\mathbf{x}_0\|_2^2] + (1 - \alpha_t^2) d
    \le (1 - \alpha_t)^2 R^2 + (1 - \alpha_t^2) d.
\end{equation}
This completes the proof.
\end{proof}

\subsection{Distribution Convergence Analysis on the Interval $[\delta_t, T]$}

\begin{theorem}[$W_2$ Convergence with Neural Network Approximation Error]
\label{thm:w2_convergence_with_nn}
Under Assumptions \ref{ass:p0_bounded} through \ref{ass:alpha_regularity} (particularly, $\alpha_T = 0$ implies $\sigma_T = 1$), assume that both the sampling algorithm and the reference process have initial distribution $\mathbf{X}(0) \sim \pi = \mathcal{N}(\mathbf{0}, \mathbf{I})$ (standard Gaussian prior). Let $\hat{\mu}_N$ denote the probability distribution of the numerical solution $\hat{\mathbf{x}}_N$ generated by the neural network-driven discrete sampling algorithm at time $\tau = T_{\delta_t}$ (i.e., physical time $t = \delta_t$), and $\mu_*$ denote the distribution of the true probability flow ODE solution $\mathbf{X}(T_{\delta_t})$.

Then the Wasserstein-2 distance between $\hat{\mu}_N$ and $\mu_*$ satisfies the following convergence bound:
\begin{equation}
    W_2^2(\hat{\mu}_N, \mu_*) \le \mathcal{S}(T_{\delta_t}) \left[ \mathcal{C}_{approx'}(d, R, L, \sigma_{\delta_t}) \varepsilon^2 + \frac{\mathcal{C}_{disc'}(d, R, L, \sigma_{\delta_t})}{N} \right].
\end{equation}
where $\mathcal{S}(T_{\delta_t}) = \frac{e^{\tilde{C}_1 T_{\delta_t}} - 1}{\tilde{C}_1}$ is the stability factor defined in Theorem \ref{thm:strong_convergence_precise_v2}, $\mathcal{C}_{approx'}$ is the approximation error constant (depending on dimension $d$, data manifold radius $R$, number of classes $L$, and noise variance $\sigma_{\delta_t}^2$), $\mathcal{C}_{disc'} $ is the discretization error constant (with the same parameter dependencies). Explicit expressions of the constants are given in Remark \ref{rem:w2_constants}.
\end{theorem}

\begin{proof}

The proof is divided into four steps: first establish the connection between $W_2$ distance and trajectory coupling, then integrate over the initial distribution and finely handle the dependence of each term on the initial value, then compute moment estimates under the standard Gaussian distribution ($\sigma_T = 1$), and finally simplify the expression using the small step size assumption and assemble the final bound.

\textbf{Step 1: Coupling Construction and Distance Upper Bound}

According to the definition of Wasserstein-2 distance:
\begin{equation}
    W_2^2(\hat{\mu}_N, \mu_*) = \inf_{\gamma \in \Pi(\hat{\mu}_N, \mu_*)} \int_{\mathbb{R}^d \times \mathbb{R}^d} \|\mathbf{x} - \mathbf{y}\|^2 \, d\gamma(\mathbf{x}, \mathbf{y}),
\end{equation}
where $\Pi(\hat{\mu}_N, \mu_*)$ denotes all joint distributions with marginal distributions $\hat{\mu}_N$ and $\mu_*$.

In this algorithm, both the numerical trajectory $\{\hat{\mathbf{x}}_k\}_{k=0}^N$ and the reference ODE trajectory $\{\mathbf{X}(\tau_k)\}_{k=0}^N$ start from the same random initial value $\mathbf{Z} := \mathbf{X}(0) \sim \pi = \mathcal{N}(\mathbf{0}, \mathbf{I})$ (i.e., $\hat{\mathbf{x}}_0 = \mathbf{Z}$). This common initialization naturally defines a coupling $\hat{\gamma} \in \Pi(\hat{\mu}_N, \mu_*)$, satisfying:
\begin{equation}
    \label{eq:w2_coupling_upper_bound}
    W_2^2(\hat{\mu}_N, \mu_*) \le \mathbb{E}_{(\hat{\mathbf{x}}_N, \mathbf{X}(T_{\delta_t})) \sim \hat{\gamma}} \left[ \|\hat{\mathbf{x}}_N - \mathbf{X}(T_{\delta_t})\|^2 \right] = \mathbb{E}_{\mathbf{Z} \sim \pi} \left[ \mathbb{E}_{\text{alg}} \left[ \|\hat{\mathbf{x}}_N - \mathbf{X}(T_{\delta_t})\|^2 \mid \mathbf{Z} \right] \right],
\end{equation}
where the inner expectation $\mathbb{E}_{\text{alg}}$ averages over the algorithm randomness (random indices $\{\hat{l}_k\}$), and the outer expectation integrates over the initial distribution $\pi$.

\textbf{Step 2: Fine Decomposition of Conditional Mean Squared Error and Initial Value Dependence Analysis}

For a fixed initial value $\mathbf{Z} = \mathbf{z}$, according to Equation \eqref{eq:precise_error_bound_v2} in Theorem \ref{thm:strong_convergence_precise_v2}, the conditional mean squared error satisfies:
\begin{equation}
    \label{eq:conditional_error_bound}
    \mathbb{E}_{\text{alg}} \left[ \|\hat{\mathbf{x}}_N - \mathbf{X}(T_{\delta_t})\|^2 \mid \mathbf{Z} = \mathbf{z} \right] \le \mathcal{S}(T_{\delta_t}) \left( \mathcal{C}_{approx}(\mathbf{z}) \varepsilon^2 + \mathcal{C}_{disc}(\mathbf{z}) \Delta \tau \right),
\end{equation}
where the time step size $\Delta \tau = T_{\delta_t} / N$, and the coefficients are:
\begin{align}
    \mathcal{C}_{approx}(\mathbf{z}) &= 2C_{NN}^2(1 + \Delta \tau_0)^2 (1 + 2R_0^2(\mathbf{z})), \quad R_0^2(\mathbf{z}) = (\|\mathbf{z}\| + R)^2, \\
    \mathcal{C}_{disc}(\mathbf{z}) &= C_{\mathcal{M}}(1 + \|\mathbf{z}\|^2) + C_{trunc}^2(\mathbf{z})(1 + \Delta \tau_0).
\end{align}

The dependence on the initial value norm $\|\mathbf{z}\|$ appears in two places:
\begin{enumerate}
    \item \textbf{Approximation error coefficient}'s $R_0^2(\mathbf{z})$: Expanding gives
    \begin{equation}
        1 + 2R_0^2(\mathbf{z}) = 1 + 2\|\mathbf{z}\|^2 + 4R\|\mathbf{z}\| + 2R^2;
    \end{equation}
    \item \textbf{Discretization error coefficient}'s $C_{disc}(\mathbf{z})$: Utilizing the affine structure of truncation error $C_{trunc}(\mathbf{z}) = C_{trunc,1}\|\mathbf{z}\| + C_{trunc,0}$ (see Remark \ref{rem:w2_constants}), expanding gives
    \begin{align}
        C_{disc}(\mathbf{z}) &= C_{\mathcal{M}}(1 + \|\mathbf{z}\|^2) + (1 + \Delta \tau_0)(C_{trunc,1}^2\|\mathbf{z}\|^2 + 2C_{trunc,1}C_{trunc,0}\|\mathbf{z}\| + C_{trunc,0}^2).
    \end{align}
\end{enumerate}

Integrating Equation \eqref{eq:conditional_error_bound} over the initial distribution $\pi = \mathcal{N}(\mathbf{0}, \mathbf{I})$, and utilizing the linearity of expectation:
\begin{align}
    \mathbb{E}_{\mathbf{Z} \sim \pi} &\left[ \mathbb{E}_{\text{alg}} \left[ \|\hat{\mathbf{x}}_N - \mathbf{X}(T_{\delta_t})\|^2 \mid \mathbf{Z} \right] \right] \nonumber \\
    &\le \mathcal{S}(T_{\delta_t}) \left[ \mathbb{E}_{\mathbf{Z} \sim \pi}[\mathcal{C}_{approx}(\mathbf{Z})] \varepsilon^2 + \mathbb{E}_{\mathbf{Z} \sim \pi}[\mathcal{C}_{disc}(\mathbf{Z})] \Delta \tau \right].
    \label{eq:expectation_after_integration}
\end{align}

\textbf{Step 3: Moment Computation for Standard Gaussian Distribution}

Given $\mathbf{Z} \sim \mathcal{N}(\mathbf{0}, \mathbf{I}_d)$ (from $\sigma_T = 1$), compute the required moments:

\textbf{(1) Second moment}: Directly given by the trace of the covariance matrix:
\begin{equation}
    \mathbb{E}[\|\mathbf{Z}\|^2] = \operatorname{Tr}(\mathbf{I}_d) = d.
\end{equation}

\textbf{(2) First moment}: Utilizing Jensen's inequality ($\sqrt{\cdot}$ is concave):
\begin{equation}
    \mathbb{E}[\|\mathbf{Z}\|] = \mathbb{E}[\sqrt{\|\mathbf{Z}\|^2}] \le \sqrt{\mathbb{E}[\|\mathbf{Z}\|^2]} = \sqrt{d}.
\end{equation}

Substituting the moment estimates into Equation \eqref{eq:expectation_after_integration}:

(i) Approximation error expectation (before simplification):
\begin{equation}
    \mathbb{E}[\mathcal{C}_{approx}(\mathbf{Z})] \le 2C_{NN}^2(1 + \Delta \tau_0)^2 (1 + 2d + 4R\sqrt{d} + 2R^2).
\end{equation}

(ii) Discretization error expectation (before simplification):
\begin{align}
    \mathbb{E}[\mathcal{C}_{disc}(\mathbf{Z})] &\le C_{\mathcal{M}}(1 + d) + (1 + \Delta \tau_0) (C_{trunc,1}\sqrt{d} + C_{trunc,0})^2.
\end{align}

\textbf{Step 4: Simplification Using Small Step Size Assumption and Final Bound Assembly}

Utilizing the small step size assumption $\Delta \tau_0 < 1$, scale the terms involving $(1 + \Delta \tau_0)$:
\begin{itemize}
    \item $(1 + \Delta \tau_0)^2 \le 4$ (since $1 + \Delta \tau_0 < 2$);
    \item $(1 + \Delta \tau_0) \le 2$.
\end{itemize}

Applying these scalings and rearranging:

(i) Approximation error constant:
\begin{equation}
    \mathbb{E}[\mathcal{C}_{approx}(\mathbf{Z})] \le 8 C_{NN}^2 (1 + 2d + 4R\sqrt{d} + 2R^2) =: \mathcal{C}_{approx'}(d, R, L, \sigma_{\delta_t}).
\end{equation}

(ii) Discretization error constant: Utilizing the definition $C_{trunc}(d, R, \sigma_{\delta_t}) = C_{trunc,1}\sqrt{d} + C_{trunc,0}$, we get
\begin{align}
    \mathbb{E}[\mathcal{C}_{disc}(\mathbf{Z})] &\le C_{\mathcal{M}}(1 + d) + 2 C_{trunc}^2(d, R, \sigma_{\delta_t}).
\end{align}

Note that we defined $\mathcal{C}_{disc}(d, R, L, \sigma_{\delta_t})$ in Remark \ref{rem:w2_constants} to include the $T_{\delta_t}$ factor, therefore:
\begin{equation}
    \mathbb{E}[\mathcal{C}_{disc}(\mathbf{Z})] \Delta \tau = \mathbb{E}[\mathcal{C}_{disc}(\mathbf{Z})] \frac{T_{\delta_t}}{N} \le \frac{\mathcal{C}_{disc'}(d, R, L, \sigma_{\delta_t})}{N}.
\end{equation}

Combining Equation \eqref{eq:w2_coupling_upper_bound} with the above estimates:
\begin{equation}
    W_2^2(\hat{\mu}_N, \mu_*) \le \mathcal{S}(T_{\delta_t}) \left[ \mathcal{C}_{approx'}(d, R, L, \sigma_{\delta_t}) \varepsilon^2 + \frac{\mathcal{C}_{disc'}(d, R, L, \sigma_{\delta_t})}{N} \right].
\end{equation}

\end{proof}

\begin{remark}[Explicit Expressions of $W_2$ Convergence Constants]
\label{rem:w2_constants}
The error constants in Theorem \ref{thm:w2_convergence_with_nn} have the following explicit forms:

\paragraph{1. Approximation Error Constant} (depending on $d, R, L, \sigma_{\delta_t}$):
Utilizing the small step size assumption $\Delta \tau_0 < 1$ for scaling, we obtain
\begin{equation}
    \mathcal{C}_{approx'}(d, R, L, \sigma_{\delta_t}) := 8 C_{NN}^2(C_{\alpha,1}, \sigma_{\delta_t}, R, L) \left[ 1 + 2d + 4R\sqrt{d} + 2R^2 \right],
\end{equation}
where the neural network approximation constant $C_{NN}$ (see Theorem \ref{thm:nn_velocity_error_bound}) is defined as:
\begin{equation}
    C_{NN}(C_{\alpha,1}, \sigma_{\delta_t}, R, L) = \frac{C_{\alpha,1}}{\sigma_{\delta_t}^2} \left[ L \left( 1 + R + \frac{R^2}{\sigma_{\delta_t}^2} \right) + 1 \right].
\end{equation}
This constant characterizes the amplification effect of neural network $L^\infty$ approximation error $\varepsilon$ on trajectory $L^2$ error, explicitly depending on:
\begin{itemize}
    \item $L$ (number of classes): Appears through the weighted averaging of mixture velocity fields;
    \item $\sigma_{\delta_t}^2 = 1 - \alpha_{\delta_t}^2$ (noise variance): Appears through the scaling coefficient $\frac{\alpha_t}{\sigma_t^2}$ of the velocity field;
    \item $R$ (data manifold radius): Influences error amplification through the spatial range of trajectories;
    \item $C_{\alpha,1}$ (Assumption \ref{ass:alpha_regularity}): Bound on the first derivative of the noise scheduling parameter $\alpha_t$.
\end{itemize}

\paragraph{2. Discretization Error Constant} (depending on $d, R, L, \sigma_{\delta_t}$):
Utilizing $\Delta \tau_0 < 1$ and $(1 + \Delta \tau_0) \le 2$, we obtain
\begin{equation}
    \mathcal{C}_{disc'}(d, R, L, \sigma_{\delta_t}) := T_{\delta_t} \left\{ C_{\mathcal{M}}(L, \sigma_{\delta_t}, R)(1 + d) + 2 C_{trunc}^2(d, R, \sigma_{\delta_t}) \right\},
\end{equation}
where the coefficients are defined as follows:

\textbf{(a) Martingale Difference Noise Constant} (see Lemma \ref{lem:martingale_bound_nn}):
\begin{equation}
    C_{\mathcal{M}}(L, \sigma_{\delta_t}, R) := 2 [M'_{\delta_t}(L, \sigma_{\delta_t}, R)]^2 (1 + K_{mom}),
\end{equation}

where $M'_{\delta_t}$ is the linear growth coefficient of the neural network velocity field (see Corollary \ref{cor:v_theta_linear_growth}), which is independent of the number of classes $L$ and consists of physical and error components:
\begin{equation}
    M'_{\delta_t}(\sigma_{\delta_t}, R, \varepsilon_y) = C_{\text{phy}} + C_{\text{err}} = \frac{C_{\alpha,1}}{\sigma_{\delta_t}^2} \left( 1 + R + \frac{R^2}{\sigma_{\delta_t}^2} \right) + \frac{C_{\alpha,1}}{\sigma_{\delta_t}^2}\varepsilon_y,
\end{equation}

and $K_{mom} = 2e^{(3M'_{\delta_t} + 2(M'_{\delta_t})^2)T}$ is the uniform moment bound constant for discrete trajectories (see Lemma \ref{lem:uniform_moment_bound_nn}).

\textbf{(b) Truncation Error Constant} (see Lemma \ref{lem:local_truncation_bound}):
Utilizing the affine structure of truncation error $C_{trunc}(\mathbf{z}) = C_{trunc,1}\|\mathbf{z}\| + C_{trunc,0}$, and substituting the moment estimate $\mathbb{E}[\|\mathbf{Z}\|] \le \sqrt{d}$, we obtain
\begin{equation}
    C_{trunc}(d, R, \sigma_{\delta_t}) := C_{trunc,1}(R, \sigma_{\delta_t}) \sqrt{d} + C_{trunc,0}(R, \sigma_{\delta_t}),
\end{equation}
where the coefficients $C_{trunc,1}, C_{trunc,0}$ have the following explicit expressions:
\begin{align}
    C_{trunc,1}(R, \sigma_{\delta_t}) &= \frac{1}{2} \left[ C_{1,\delta_t} + L_{x,\delta_t} M_{\delta_t} \right], \\
    C_{trunc,0}(R, \sigma_{\delta_t}) &= \frac{1}{2} \left[ (C_{1,\delta_t} + L_{x,\delta_t} M_{\delta_t})R + (C_{0,\delta_t} + L_{x,\delta_t} M_{\delta_t}) \right].
\end{align}
For brevity, we omit the parameter dependencies in parentheses after $C_{1,\delta_t}, C_{0,\delta_t}, L_{x,\delta_t}, M_{\delta_t}$, their complete forms are given below.

\paragraph{3. Explicit Expressions of Intermediate Constants}
The above coefficients depend on the following fundamental constants:

\textbf{(i) Spatial Lipschitz Constant} (see Corollary \ref{cor:global_bound_Lv_uniform}):
\begin{equation}
    L_{x,\delta_t}(C_{\alpha,1}, \alpha_{\delta_t}, R) = C_{\alpha,1} \frac{\alpha_{\delta_t}}{\sigma_{\delta_t}^2} \left( 1 + \frac{R^2}{\sigma_{\delta_t}^2} \right), \quad \sigma_{\delta_t}^2 = 1 - \alpha_{\delta_t}^2.
\end{equation}

\textbf{(ii) True Velocity Field Linear Growth Coefficient} (see Proposition \ref{prop:v_bound}):
\begin{equation}
    M_{\delta_t}(C_{\alpha,1}, \sigma_{\delta_t}, R) = C_{\alpha,1} \frac{1}{\sigma_{\delta_t}^2} \left( 1 + R + \frac{R^2}{\sigma_{\delta_t}^2} \right).
\end{equation}

\textbf{(iii) Temporal Derivative Growth Coefficients} (see Corollary \ref{cor:uniform_bound_delta}):
\begin{align}
    C_{1,\delta_t}(C_{\alpha,1}, C_{\alpha,2}, \alpha_{\delta_t}, R) 
    &= \frac{C_{\alpha,2} \alpha_{\delta_t}}{\sigma_{\delta_t}^2} + \frac{2\alpha_{\delta_t}^2 C_{\alpha,1}^2}{\sigma_{\delta_t}^4} + \frac{C_{\alpha,1}^2}{\sigma_{\delta_t}^2} \left( \frac{(1+\alpha_{\delta_t}^2) R^2}{\sigma_{\delta_t}^4} + 1 \right), \\
    C_{0,\delta_t}(C_{\alpha,1}, C_{\alpha,2}, \alpha_{\delta_t}, R) 
    &= \frac{C_{\alpha,2} R}{\sigma_{\delta_t}^2} + \frac{2\alpha_{\delta_t} C_{\alpha,1}^2 R}{\sigma_{\delta_t}^4} + \frac{2\alpha_{\delta_t} C_{\alpha,1}^2 R^3}{\sigma_{\delta_t}^6},
\end{align}
where $C_{\alpha,2}$ is the bound on the second derivative of $\alpha_t$ in Assumption \ref{ass:alpha_regularity}.

\paragraph{4. Summary of Parameter Dependencies}
All constants can ultimately be expressed as functions of the following fundamental parameters:
\begin{itemize}
    \item \textbf{System parameters}: $C_{\alpha,1}, C_{\alpha,2}$ (noise scheduling regularity, Assumption \ref{ass:alpha_regularity}), $R$ (data manifold radius, Assumption \ref{ass:p0_bounded}), $L$ (number of classes, Assumption \ref{assump:mixture_data});
    \item \textbf{Truncation parameter}: $\alpha_{\delta_t}$ (or equivalently $\sigma_{\delta_t}^2 = 1 - \alpha_{\delta_t}^2$), which determines the starting noise level of the sampling algorithm;
    \item \textbf{Dimension}: $d$ (data space dimension).
\end{itemize}
\end{remark}

\section{Training Details and Experimental Results}
\label{sec:training_experiments}

In this section, we first derive the analytical form of the temporal derivative $\frac{da(x_t, t)}{dt}$ for our high-order training objective, then present comprehensive experimental results on two benchmark image datasets: CIFAR-10 and CelebA-HQ.

\subsection{High-Order Training Objectives}
\label{subsec:high_order_training}

Our training framework incorporates a novel regularization term based on the temporal evolution of posterior probabilities. This section provides the theoretical foundation and analytical derivation.

To simplify the notation, we adopt the discrete finite-dataset analogy setting of \cref{eq:CPP}, where we define:

\begin{equation}
\label{eq:discrete_posterior}
\hat{p}^l(x_s|x_t) = \sum_{i=1}^{N_l} (2\pi\sigma_{s|t})^{-\frac{d}{2}} u^l_i(x_t, t) \exp\left\{-\frac{1}{2\sigma_{s|t}^2}\left\|x_s - \frac{\alpha_{t|s}\sigma_s^2}{\sigma_t^2}x_t - \frac{\alpha_s\sigma_{t|s}^2}{\sigma_t^2}\bar{y}^l(x_t, t)\right\|^2\right\},
\end{equation}

\begin{equation}
\label{eq:discrete_weights}
w^l_i(x_t, t) = \frac{\exp\left\{-\frac{\|x_t - \alpha_t y^l_i\|^2}{2\sigma_t^2}\right\}}{\sum_{l,j} \exp\left\{-\frac{\|x_t - \alpha_t y^l_j\|^2}{2\sigma_t^2}\right\}}, \quad a^l(x_t, t) = \sum_{i=1}^{N_l} w^l_i(x_t, t),
\end{equation}

\begin{equation}
\label{eq:discrete_means}
u^l_i(x_t, t) = \frac{w^l_i(x_t, t)}{a^l(x_t, t)}, \quad \bar{y}^l(x_t, t) = \sum_{i=1}^{N_l} u^l_i(x_t, t) y^l_i,
\end{equation}

where $\{y^l_i\}_{i=1}^{N_l}$ are the training samples belonging to cluster $l$, and $N_l$ is the number of samples in that cluster. These discrete formulations replace the continuous integrals in \cref{eq:CPP} with finite sums over the training dataset, making the subsequent derivations computationally tractable.

\subsubsection{Analytical Form of $\frac{da(x_t, t)}{dt}$}

By taking the temporal derivative of the posterior probability $a(x_t, t)$, we obtain the analytical form:
\begin{equation}
\label{eq:da_dt_main}
    \frac{da(x_t, t)}{dt} = \sum_j W_j(x_t,t) G(x_t, t, y_j) [C(y_j) - a_\theta(x_t, t)],
\end{equation}
where the components are defined as:
\begin{equation}
\label{eq:components_def}
\begin{aligned}
    G(x_t, t, y_i) &= \frac{-1}{\alpha_t \sigma_t^2}(\|\epsilon_t\|^2 - \langle\epsilon_t, \bar{\epsilon}_t\rangle)\frac{d\alpha_t}{dt}, \\
    a(x_t,t) &= \sum_j W_j(x_t,t)\, C(y_j), \\
    W_j &= \frac{V_j}{\sum_k V_k}, \quad V_j = \exp\!\left(-\frac{\|x_t-\alpha_t y_j\|^2}{2\sigma_t^2}\right), \\
    C(y_j) &= \text{label}(y_j).
\end{aligned}
\end{equation}

The derivation relies on the following key lemmas regarding the dynamics of the forward diffusion process.

\begin{lemma}[Forward Process Dynamics]
\label{lem:forward_dynamics}
The forward diffusion process satisfies:
\begin{equation}
\label{eq:forward_ode}
\begin{aligned}
\frac{d x_t}{d t}
&= f_t x_t - \tfrac{1}{2} g(t)^2 s_\theta(x_t,t) \\
&= f_t(x_t + s_\theta(x_t,t)) \\
&= \frac{d \log \alpha_t}{d t} \left(x_t - \frac{\bar{\epsilon}_t}{\sigma_t}\right),
\end{aligned}
\end{equation}
where $f_t$ is the drift coefficient and $g(t)$ is the diffusion coefficient.
\end{lemma}

\begin{lemma}[Temporal Evolution of Gaussian Weights]
\label{lem:gaussian_weight_derivative}
Define the Gaussian weight $V_j = \exp\!\left(-\frac{\|x_t-\alpha_t y_j\|^2}{2\sigma_t^2}\right)$. Its temporal derivative is:
\begin{equation}
\label{eq:dVj_dt}
\begin{aligned}
\frac{d V_j}{d t}
&= -V_j \cdot \frac{1}{2}\frac{d}{d t}\!\left(\frac{\|x_t-\alpha_t y_j\|^2}{\sigma_t^2}\right) \\[4pt]
&= -V_j \frac{1}{2\sigma_t^4}\Bigl[2(x_t - \alpha_t y_j)\sigma_t^2
\frac{d(x_t - \alpha_t y_j)}{d t}
- \|x_t-\alpha_t y_j\|^2 \cdot 2\sigma_t \frac{d\sigma_t}{d t}\Bigr] \\[6pt]
&= -\frac{V_j}{\sigma_t^3}\Bigl[(x_t-\alpha_t y_j)\sigma_t\frac{d x_t}{d t}
- \sigma_t(x_t-\alpha_t y_j)\frac{d\alpha_t}{d t}\,y_j 
- \|x_t-\alpha_t y_j\|^2\frac{d\sigma_t}{d t}\Bigr] \\[6pt]
&= -V_j\frac{(x_t - \alpha_t y_j)}{\sigma_t^3}\Bigl( 
\sigma_t\frac{d\log\alpha_t}{d t}\Bigl(x_t - \frac{\bar\epsilon_t}{\sigma_t}\Bigr)
- \sigma_t\frac{d\alpha_t}{d t}\,y_j
- (x_t- \alpha_t y_j)\frac{d\sigma_t}{d t}\Bigr) \\[6pt]
&= -V_j\frac{\epsilon_{t,j}}{\sigma_t^2}\Bigl(\frac{\sigma_t^3}{\alpha_t}
\frac{d\alpha_t}{d t}\frac{\epsilon_{t,j}}{\sigma_t}
- \frac{\sigma_t}{\alpha_t}\frac{d\alpha_t}{d t}\,
\langle \epsilon_{t,j},\bar\epsilon_t\rangle + \alpha_t\sigma_t\frac{d\alpha_t}{d t}\Bigr) \\[6pt]
&= -\frac{V_j}{\sigma_t^3}
\Bigl[\Bigl(\frac{\sigma_t^3}{\alpha_t}+\alpha_t\sigma_t\Bigr)
\|\epsilon_{t,j}\|^2
- \frac{\sigma_t^2}{\alpha_t}
\langle \epsilon_{t,j},\,\bar\epsilon_t \rangle\Bigr]
\frac{d \alpha_t}{d t} \\[6pt]
&= -\frac{V_j}{\alpha_t\sigma_t^2}
\Bigl( \|\epsilon_{t,j}\|^2
- \langle \epsilon_{t,j},\,\bar\epsilon_t \rangle \Bigr)
\frac{d \alpha_t}{d t} \\[6pt]
&= V_j\, G(x_t,t,y_j),
\end{aligned}
\end{equation}
where $\epsilon_{t,j} := \frac{x_t - \alpha_t y_j}{\sigma_t}$ is the noise component.
\end{lemma}

\begin{proof}[Derivation of \cref{eq:da_dt_main}]
Given $a(x_t,t) = \sum_j W_j(x_t,t)\, C(y_j)$ where $C(y_j)$ is the cluster label, we compute:
\begin{equation}
\label{eq:da_dt_derivation}
\begin{aligned}
\frac{d a(x_t, t)}{d t} 
&= \sum_j C(y_j) \frac{d W_j(x_t, t)}{d t} \\
&= \sum_j C(y_j) \frac{d}{d t}\left[\frac{V_j(x_t, t)}{\sum_k V_k(x_t, t)}\right] \\
&= \sum_j C(y_j) \frac{\frac{d V_j(x_t, t)}{d t} \cdot \sum_k V_k(x_t, t) - V_j(x_t, t) \cdot \sum_k \frac{d V_k(x_t, t)}{d t}}{(\sum_k V_k(x_t, t))^2} \\
&= \sum_j C(y_j) \frac{V_j(x_t, t)G(x_t, t, y_j)(\sum_k V_k(x_t, t)) - V_j(x_t, t)\sum_k V_k(x_t, t)G(x_t, t, y_k)}{(\sum_k V_k(x_t, t))^2} \\
&= \sum_j \frac{V_j(x_t, t)}{\sum_k V_k(x_t, t)}C(y_j)G(x_t, t, y_j) \\
&\quad - \sum_j \frac{V_j(x_t, t)}{\sum_k V_k(x_t, t)}C(y_j) \cdot \sum_i \frac{V_i(x_t, t)}{\sum_k V_k(x_t, t)}G(x_t, t, y_i) \\
&= \sum_j W_j(x_t, t)C(y_j)G(x_t, t, y_j) - \sum_j W_j C(y_j) \cdot \sum_i W_i(x_t, t)G(x_t, t, y_i) \\
&= \sum_j W_j(x_t, t)C(y_j)G(x_t, t, y_j) - a_\theta(x_t, t) \cdot \sum_i W_i(x_t, t)G(x_t, t, y_i) \\
&= \sum_j W_j(x_t, t)G(x_t, t, y_j)[C(y_j) - a_\theta(x_t, t)].
\end{aligned}
\end{equation}
\end{proof}

\subsubsection{Regularization Loss}

Based on the analytical form derived above, we define the regularization loss for the temporal derivative network $\frac{d a_\phi}{d t}$ as:
\begin{equation}
\label{eq:reg_loss}
\mathcal{L}_{\text{reg}} = \mathbb{E}_{t\sim \mathcal{U}[0,1], x_0 \sim p_{\text{data}}, x_t \sim p_t(\cdot|x_0)}\left\|\frac{d a_\phi(x_t, t)}{d t} - G(x_t, t, x_0)(C(x_0) - a_\phi(x_t, t))\right\|^2,
\end{equation}
where $G(x_t, t, x_0) = \frac{-1}{\alpha_t \sigma_t^2}(\|\epsilon_t\|^2 - \langle \epsilon_t, \bar{\epsilon}_t\rangle)\frac{d\alpha_t}{dt}$.

This regularization encourages the learned temporal derivative to match the analytical evolution of posterior probabilities, thereby improving the accuracy of category tracking during the reverse sampling process.

\subsection{Experimental Results on Image Datasets}
\label{subsec:image_results}

We evaluate our method on two benchmark datasets: CIFAR-10~\citep{Krizhevsky09learningmultiple} and CelebA-HQ~\citep{DBLP:journals/corr/abs-1710-10196}. CIFAR-10 contains 60,000 natural images across 10 object categories, while CelebA-HQ provides 30,000 high-resolution face images with attribute labels at $256\times256$ resolution. 

\subsubsection{Training Configuration}

Due to numerical stability considerations during early training, we adopt a two-stage training strategy:
\begin{itemize}
    \item \textbf{Stage 1 (0--50k steps):} Combined cross-entropy and MSE loss to stabilize category prediction
    \item \textbf{Stage 2 (50k--500k steps):} Pure MSE loss with the regularization term from \cref{eq:reg_loss}. To address the heterogeneous distribution of temporal derivatives across the reverse process, we train three separate $\frac{da_\phi}{dt}$ networks initialized with the same weights but trained independently on different time intervals:
\begin{itemize}
    \item \textbf{Network 1 ($t \in [0, 300]$):} Handles the near-data regime where $\frac{da}{dt} \approx 0$ (as analytically shown in \cref{eq:da_dt_main}, the gradient term $G(x_t, t, y_j)$ vanishes when $x_t \approx \alpha_t y_j$)
    \item \textbf{Network 2 ($t \in [300, 700]$):} Captures the transition regime with rapidly varying temporal derivatives
    \item \textbf{Network 3 ($t \in [700, 1000]$):} Models the early sampling phase where categories are relatively uniformly distributed
\end{itemize}
During inference, the appropriate network is selected based on the current timestep. This time-conditional specialization significantly improves training stability and prediction accuracy compared to using a single network across all timesteps.
\end{itemize}

All experiments use the following hyperparameters: learning rate $\eta = 10^{-4}$, total batch size $B = 128$, linear scheduler, and Adam optimizer~\citep{kingma2014adam} with default momentum parameters $\beta_1 = 0.9, \beta_2 = 0.999$. The neural network architecture follows the standard U-Net~\citep{ronneberger2015u} design commonly used in diffusion models. Training is performed on 8 NVIDIA A800 GPUs (80GB memory each) for approximately 2 days per dataset, totaling around 200k optimization steps.

\subsubsection{Quantitative Comparison}

\cref{tab:fid_results} presents FID scores comparing our method against three baselines:
\begin{itemize}
    \item \textbf{DDIM (Uncond baseline):} Standard unconditional DDIM~\citep{bao2022estimating}
    \item \textbf{+ Naive:} Direct category prediction without temporal derivative
    \item \textbf{+ MoG-a net:} Training with only the posterior probability network $a_\theta$
    \item \textbf{+ MoG-da net (ours):} Our full method with both $a_\theta$ and $\frac{da_\phi}{dt}$ networks
\end{itemize}

\begin{table}[t]
\centering
\caption{FID scores ($\downarrow$) on CIFAR-10 and CelebA-HQ with varying number of sampling steps compared to the baseline(unconditional DDIM) method}
\label{tab:fid_results}
\begin{tabular}{l|cccc|cccc}
\toprule
\multirow{2}{*}{Method} & \multicolumn{4}{c|}{CIFAR-10} & \multicolumn{4}{c}{CelebA-HQ} \\
\cmidrule(lr){2-5} \cmidrule(lr){6-9}
& 10 & 25 & 50 & 100 & 10 & 25 & 50 & 100 \\
\midrule
DDIM (Uncond) & 21.31 & 10.70 & 7.74 & 6.08 & 20.54 & 13.45 & 9.33 & 6.60 \\
+ Naive & \textbf{15.97} & \textbf{8.83} & \textbf{6.54} & \textbf{5.27} & 10.76 & 7.00 & 5.97 & 5.23  \\
+ MoG-a net & 17.93 & 9.63 & 7.35 & 6.21 & 11.45 & 7.41 & 6.23 & 5.45 \\
+ MoG-da net (ours) & 16.50 & 9.16 & 6.96 & 5.75 & \textbf{10.61} & \textbf{6.68} & \textbf{5.73} & \textbf{5.10} \\
\bottomrule
\end{tabular}
\end{table}

Our method achieves competitive performance for the better performance than unconditional DDIM and MoG-a, supporting our theoritical analysis. Besides, our method even outperform Naive method on CelebA-HQ dataset, demonstrating the effectiveness of higher order training of temporal derivation regularization for accelerated sampling.

\subsubsection{Qualitative Results}

\cref{fig:cifar10_samples,fig:celeba_samples} show generated samples from both datasets with varying numbers of sampling steps. Our method maintains high visual quality and categorical consistency even with very few steps, demonstrating the effectiveness of the temporal derivative regularization.

\begin{figure}[t]
    \centering
    \begin{subfigure}[b]{0.23\textwidth}
        \includegraphics[width=\textwidth]{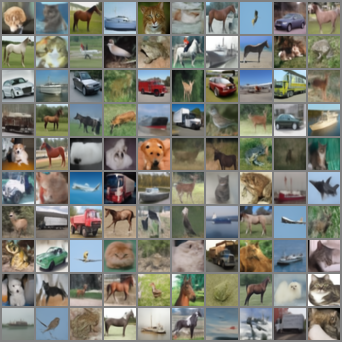}
        \caption{10 steps}
    \end{subfigure}
    \hfill
    \begin{subfigure}[b]{0.23\textwidth}
        \includegraphics[width=\textwidth]{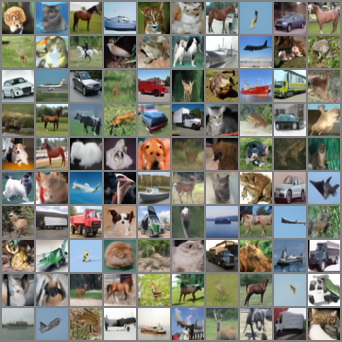}
        \caption{25 steps}
    \end{subfigure}
    \hfill
    \begin{subfigure}[b]{0.23\textwidth}
        \includegraphics[width=\textwidth]{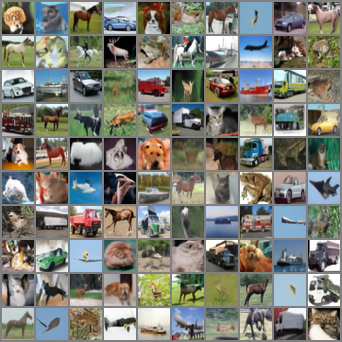}
        \caption{50 steps}
    \end{subfigure}
    \hfill
    \begin{subfigure}[b]{0.23\textwidth}
        \includegraphics[width=\textwidth]{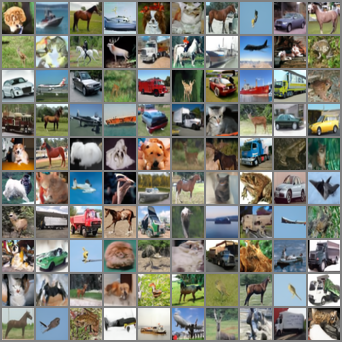}
        \caption{100 steps}
    \end{subfigure}
    \hfill
    \caption{Generated samples on CIFAR-10 with varying numbers of sampling steps. Visual quality improves steadily as the number of steps increases, with diminishing returns after 100 steps.}
    \label{fig:cifar10_samples}
\end{figure}

\begin{figure}[t]
    \centering
    \begin{subfigure}[b]{0.23\textwidth}
        \includegraphics[width=\textwidth]{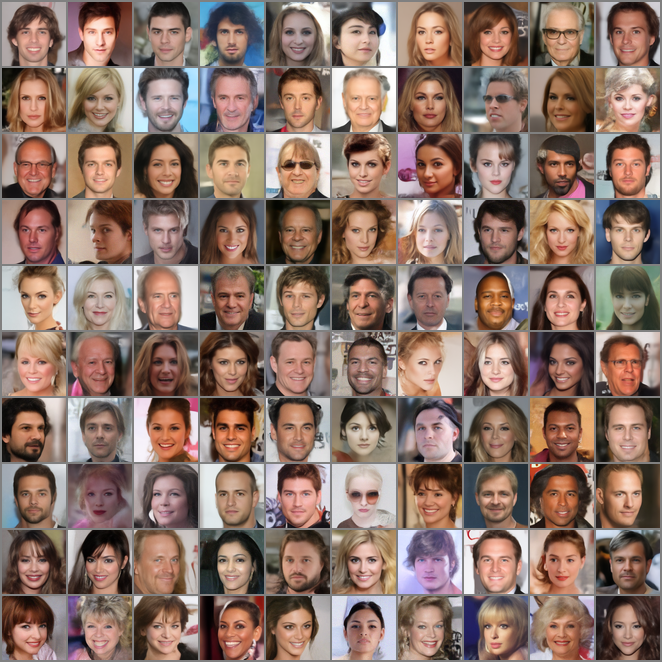}
        \caption{10 steps}
    \end{subfigure}
    \hfill
    \begin{subfigure}[b]{0.23\textwidth}
        \includegraphics[width=\textwidth]{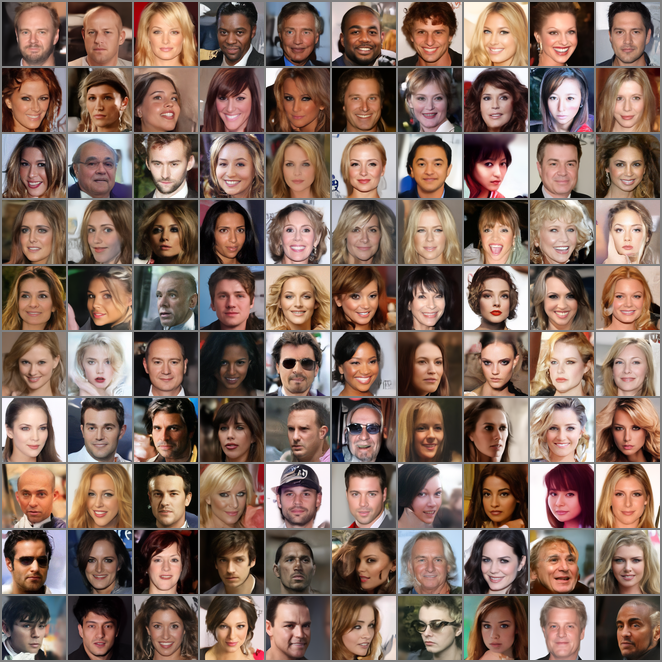}
        \caption{25 steps}
    \end{subfigure}
    \hfill
    \begin{subfigure}[b]{0.23\textwidth}
        \includegraphics[width=\textwidth]{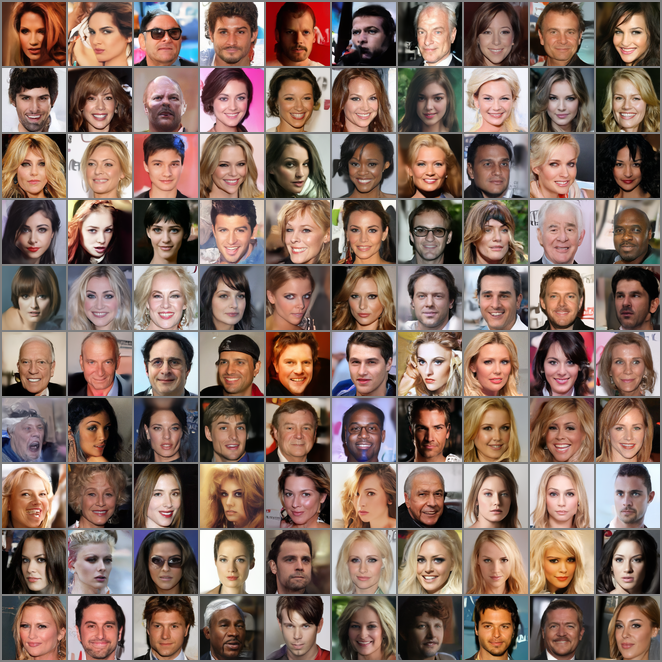}
        \caption{50 steps}
    \end{subfigure}
    \hfill
    \begin{subfigure}[b]{0.23\textwidth}
        \includegraphics[width=\textwidth]{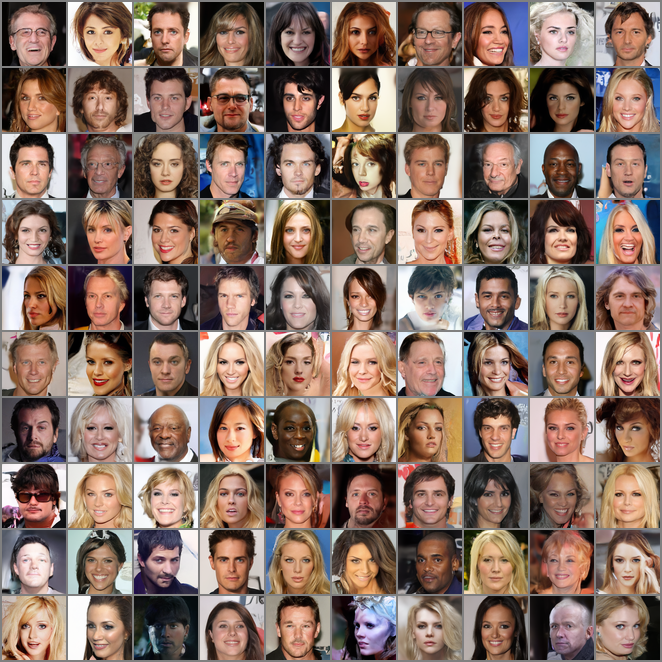}
        \caption{100 steps}
    \end{subfigure}
    \caption{Generated samples on CelebA-HQ with varying numbers of sampling steps. Our method produces high-quality faces even with as few as 10 steps, demonstrating effective category tracking on structured data.}
    \label{fig:celeba_samples}
\end{figure}

\subsection{Trajectory Error Analysis}
\label{subsec:trajectory_error}

To validate the theoretical predictions from \cref{thm:strong_convergence_precise_v2}, we conduct a trajectory error analysis measuring the deviation between our accelerated sampling and a high-precision deterministic baseline on CIFAR-10.

\subsubsection{Experimental Setup}

We establish a ground truth by generating 10,000 images using 1000-step deterministic sampling (DDIM with 1000 steps), which we denote as $\{\mathbf{x}_{\text{ref}}^{(i)}\}_{i=1}^{10000}$. This high-precision baseline serves as the reference trajectory endpoint.

For each $N \in \{10, 25, 50, 100, 200, 500, 1000\}$, we generate 10,000 images using our method with $N$ sampling steps, starting from the same initial noise realizations. We denote these samples as $\{\mathbf{x}_N^{(i)}\}_{i=1}^{10000}$. The trajectory error at step size $N$ is measured as the average squared $L^2$ distance:
\begin{equation}
\label{eq:trajectory_error}
    E(N) = \frac{1}{10000} \sum_{i=1}^{10000} \|\mathbf{x}_N^{(i)} - \mathbf{x}_{\text{ref}}^{(i)}\|^2.
\end{equation}

This metric directly quantifies how much the few-step sampling trajectory deviates from the high-precision reference, which is the key quantity bounded in \cref{thm:strong_convergence_precise_v2}.

\subsubsection{Error Scaling Analysis}

\cref{fig:error_scaling} plots the trajectory error $E(N)$ versus the number of steps $N$. Fitting a power-law model $E(N) = c \cdot N^{-\beta} + b$ to the empirical data yields:
\begin{equation}
\label{eq:fitted_power_law}
    E(N) = 0.10 \cdot N^{-1.2916} + 0.0042, \quad R^2 = 0.999963.
\end{equation}

The fitted exponent $\beta \approx 1.29$ indicates super-linear convergence with respect to the number of steps, significantly faster than the first-order rate ($\beta = 1$) expected from standard Euler-based discretization schemes. This improved scaling is consistent with our theoretical analysis in \cref{thm:strong_convergence_precise_v2}: when the neural network approximation error $\varepsilon$ is sufficiently small (as achieved through our high-order training objective in \cref{eq:reg_loss}), the dominant error contribution shifts from the $O(\Delta\tau)$ discretization term to higher-order components, enabling faster convergence as $N$ increases.

The near-perfect coefficient of determination $R^2 = 0.999963$ validates the power-law decay assumption underlying our error bound. Moreover, the convergence rate $\beta = 1.29 > 1$ empirically confirms that incorporating the temporal derivative network $\frac{da_\phi}{dt}$ provides a more accurate approximation to the underlying probability flow ODE than methods relying solely on the score function and images' labels.

\begin{figure}[t]
    \centering
    \includegraphics[width=0.7\textwidth]{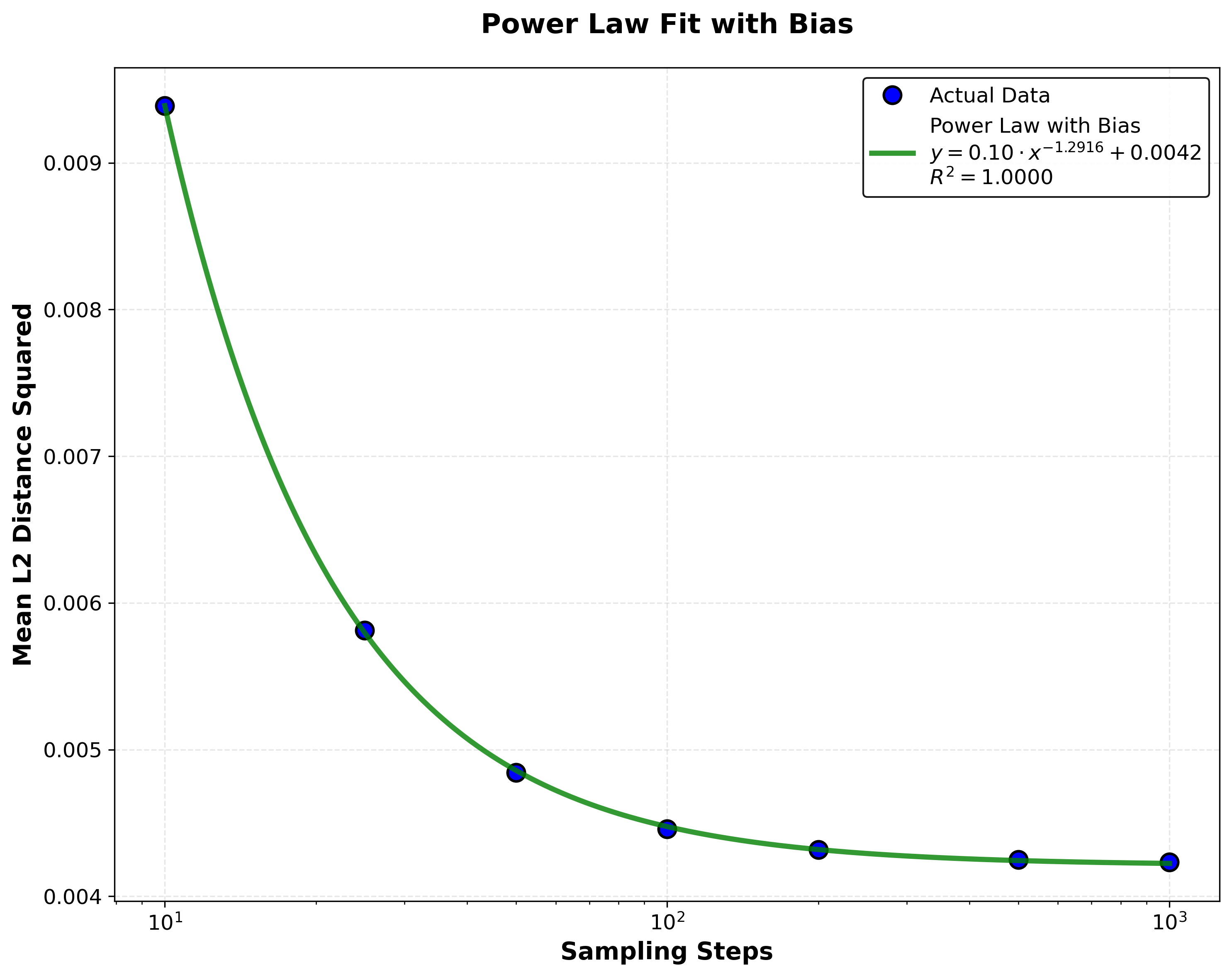}
    \caption{Trajectory error $E(N)$ versus number of sampling steps $N$ on CIFAR-10 (log-log scale). Blue circles represent empirical measurements from \cref{eq:trajectory_error}, and the red line shows the fitted power law $E(N) = 0.10 \cdot N^{-1.29}$ with $R^2 = 0.999963$. The super-linear exponent $\beta = 1.29 > 1$ demonstrates that our method achieves faster convergence than standard first-order numerical schemes.}
    \label{fig:error_scaling}
\end{figure}

\subsubsection{Practical Implications}

The super-linear convergence has important practical implications. To achieve a target error level $\epsilon$, our method requires approximately $N \propto \epsilon^{-1/1.29} \approx \epsilon^{-0.77}$ steps, compared to $N \propto \epsilon^{-1}$ for standard first-order methods. For example, to reduce error by a factor of 4, a first-order method requires 4$\times$ more steps, while our method only needs approximately $4^{0.77} \approx 3.03\times$ more steps—a 24\% reduction in computational cost.

These empirical results strongly support our theoretical framework developed in \cref{sec:convergence_analysis}: incorporating the temporal derivative $\frac{da_\phi}{dt}$ not only accelerates sampling but also provides provably more accurate trajectory tracking, especially in the few-step regime where traditional ODE solvers struggle.

\end{document}